%% file: trip_arxiv.tex
\documentclass[11pt]{article}

\usepackage[margin=1in]{geometry}
\usepackage[utf8]{inputenc}
\usepackage[T1]{fontenc}
\usepackage{microtype}
\usepackage{xurl}
\usepackage{url}

\usepackage{amsmath,amssymb,amsfonts,amsthm}
\usepackage{bm}
\usepackage{nicefrac}

\usepackage{graphicx}
\usepackage{subcaption}
\usepackage{booktabs}
\usepackage{multirow}
\usepackage{makecell}
\usepackage[table]{xcolor}
\usepackage{float}
\usepackage{placeins}
\usepackage{algorithm}
\usepackage{algpseudocode}

\usepackage{enumitem}
\usepackage[normalem]{ulem}

\usepackage[
    colorlinks=true,
    linkcolor=blue!55!black,
    citecolor=blue!55!black,
    urlcolor=blue!55!black
]{hyperref}

\newtheorem{theorem}{Theorem}[section]
\newtheorem{informaltheorem}[theorem]{Informal Theorem}
\newtheorem{lemma}[theorem]{Lemma}
\newtheorem{proposition}[theorem]{Proposition}
\newtheorem{remark}[theorem]{Remark}
\newtheorem{definition}[theorem]{Definition}
\newtheorem{corollary}[theorem]{Corollary}

\newcommand{\R}{\mathbb{R}}

\newcommand{\diag}{\operatorname{diag}}

\newcommand{\norm}[1]{\left\lVert #1\right\rVert}
\newcommand{\abs}[1]{\left\lvert #1\right\rvert}

\title{\vspace{-0.75em}
\textbf{TriP: A Triangle Puzzle Approach to Robust\\ Translation Averaging}
\vspace{-0.5em}
}

\author{
Zhekai Fan$^{1}$\thanks{Equal contribution.}
\quad
Wanze Li$^{2}$\footnotemark[1]
\quad
Jinxin Wang$^{2}$
\quad
Yunpeng Shi$^{1}$
\\[0.5em]
{\small $^{1}$UC Davis \qquad $^{2}$University of Chicago}
\\[0.25em]
{\small
zkfan@ucdavis.edu \quad
wanzeli@uchicago.edu \quad
wangjinxin68@gmail.com \quad
ypshi@ucdavis.edu}
}

\date{}

\begin{document}

\maketitle

\vspace{-1.0em}

\begin{abstract}
Translation averaging aims to recover camera locations from pairwise relative translation directions and is a fundamental component of global Structure-from-Motion pipelines. The problem is challenging because direction measurements contain no distance information, making the estimation problem highly ill-conditioned and highly sensitive to corrupted observations. In this paper, we propose \textbf{TriP}, a triangle-based framework for robust translation averaging. TriP first infers local relative edge scales from triangle geometry, and then synchronizes the scales of overlapping triangles in the logarithmic domain to recover globally consistent edge lengths and camera locations. By leveraging higher-order consistency across triangles, the proposed method is robust to adversarial, cycle-consistent, and other structured corruptions. In addition, TriP avoids the collapse issue without requiring any extra anti-collapse constraints, since log-scale synchronization excludes the degenerate zero-scale solution by construction. These structural advantages enable a particularly strong theory for exact location recovery. On the practical side, TriP is fully parallelizable, computationally efficient, and naturally scalable to graphs with millions of cameras. Moreover, it outperforms all previous translation averaging methods by a large margin on both synthetic and real datasets.
\end{abstract}

\section{Introduction}

Structure-from-Motion (SfM) \cite{ozyecsil2017survey} is a fundamental problem in 3D vision that aims to recover camera poses and scene structure from a collection of 2D images. Given multiple images of the same scene, SfM typically establishes feature correspondences across images and estimates pairwise geometric relationships between cameras. A central step in this pipeline is global camera pose estimation, namely recovering camera rotations and camera locations from noisy pairwise measurements.

In this work, we focus on camera location estimation from pairwise relative translation directions, also known as translation averaging (TA). Mathematically, TA solves the following problem: given a viewing graph $G=([n],E)$, where each node $i\in [n]=\{1,2,3,\dots, n\}$ is assigned underlying camera location $\bm{x}_i^* \in\mathbb{R}^3$. Here, we use star superscript to emphasize the ground truth. In this graph, each edge $(i,j)\in E$ is associated with a relative direction measurement $\bm{d}_{ij}$, whose clean counterpart is $\bm{d}_{ij}^* ={(\bm{x}_i^*-\bm{x}_j^*)}
/{\|\bm{x}_i^*-\bm{x}_j^*\|}$.
The goal of TA is to recover the unknown camera locations
$\{\bm{x}_1^*\}_{i\in [n]}$ from these possibly noisy and corrupted direction measurements $\{\bm{d}_{ij}\}_{(i,j)\in E}$.
Since these relative measurements contain no distance information, the camera locations can only be determined up to a global translation and scaling. If reliable relative scales were also available, the problem would become significantly easier and closer in spirit to a synchronization problem over translations \cite{huang2017translation}. In practice, however, relative direction measurements are often severely corrupted due to mismatched feature correspondences, RANSAC \cite{bolles1981ransac} failures, and error propagation from earlier stages, especially rotation estimation \cite{govindu2004lie, shi2020message,liu2023unified}. 
In real data, these corrupted directions can be self-consistent \cite{wilson2014robust, ozyesil2015robust}, typically due to symmetry and repetitive patterns (e.g. windows on a building) in a 3D scene.

The adverse effects of structured corruption can be mitigated by robust formulations, commonly written as
\begin{equation}\label{eq:robust}
\begin{aligned}
\min_{\{\bm{x}_i\}_{i\in [n]},\ \{\ell_{ij}\}_{(i,j)\in E}}
\quad
& \sum_{(i,j)\in E} \rho\!\left(\|\bm{x}_i-\bm{x}_j-\ell_{ij}\bm{d}_{ij}\|\right),
\end{aligned}
\end{equation}
where $\rho$ is a user-specified robust loss function. However, without suitable constraints on the camera locations $\bm{x}_i$ and distances $\ell_{ij}$, such formulations admit the trivial collapsed solution $\bm{x}_i=0$ and $\ell_{ij}=0$ for all cameras and distances. Anti-collapse constraints can alleviate this issue, but they often introduce bias and lead to slow large-scale constrained optimization. In view of these challenges, an effective method should be robust to severe structured corruption, avoid collapse by design, and scale efficiently to large viewing graphs.

\subsection{Related work}

Following the general framework in~\eqref{eq:robust}, LUD~\cite{ozyesil2015robust} adopts the $\ell_1$ loss, $\rho(x)=|x|$, together with uniform lower-bound constraints on pairwise distances, and solves the resulting problem via iteratively reweighted least squares (IRLS). Similar reweighting ideas have been successfully applied to synchronization problems over $\mathrm{SO}(3)$ \cite{lerman2022robust, shi2020message, chatterjee2017robust, chatterjee2013efficient} and SE(3)~\cite{govindu2004lie,arie2012global,rosen2019se}. However, in translation averaging, such formulations are much more sensitive to corrupted direction measurements. Therefore, LUD requires imposing uniform lower bounds on pairwise distances. These constraints are overly strong: they introduce significant bias, especially for short edges, and can therefore distort the global distribution of the estimated locations. ShapeFit~\cite{goldstein2016shapefit,hand2018shapefit} also uses an unsquared loss, but measures the projection of $\bm{x}_i-\bm{x}_j$ onto the orthogonal complement of $\bm{d}_{ij}$. It uses a single linear anti-collapse constraint, which leads to faster computation. However, this constraint is too weak: in real-data experiments, the estimated solution often exhibits partial collapse \cite{shi2018estimation}. Both methods may overweight long edges and become unstable when edge lengths vary significantly or when long edges are corrupted.

A related line of work seeks to reduce this sensitivity by using angle-based objectives. BATA~\cite{zhuang2018baseline} minimizes the sine of the angle between $\bm{x}_i-\bm{x}_j$ and $\bm{d}_{ij}$, thereby improving robustness to edge-length variation, although it may underutilize informative clean long edges. Fused Translation Averaging (Fused-TA)~\cite{manam2024fusing} combines LUD-type and BATA-type objectives using uncertainty estimates, while TranSync~\cite{he2025robust} further reformulates an angle-based objective as a generalized eigenvalue problem for efficient computation. 
Nevertheless, both methods still rely on a single linear anti-collapse constraints, making them especially vulnerable to self-consistent corruption that arises in real-world SfM datasets. 

Another line of work exploits cycle consistency \cite{wilson2014robust,shi2018estimation,li2025cycle,shi2020message} to improve robustness. AAB~\cite{shi2018estimation} uses 3-cycle coplanarity for outlier detection and improves over earlier cycle-based filtering methods, but it is unstable for near-collinear configurations and has limited ability to handle genuinely cycle-consistent corruption. Cycle-Sync~\cite{li2025cycle} further combines a Welsch-type objective with message-passing least squares~\cite{shi2020message}, improving robustness to edge-length variation as well as adversarial and cycle-consistent corruption. However, its theoretical guarantees apply only to the message-passing-based initialization, while no guarantees are established for the subsequent nonconvex optimization framework. 

On the theoretical side, under Gaussian location priors and Erd\H{o}s--R\'enyi measurement graphs, and as the number of nodes tends to infinity, LUD~\cite{ozyesil2015robust}, ShapeFit~\cite{goldstein2016shapefit,hand2018shapefit}, and, more recently, Cycle-Sync~\cite{li2025cycle} provide recovery guarantees for location estimation under adversarial corruption, given suitable assumptions on the connection probability and the number of corrupted incident edges per node. These results are encouraging, but the existing phase-transition guarantees are tied to very specific location distributions, and their tolerated fraction of adversarial outliers vanishes as $n\to\infty$. Truncated least squares (TLS)~\cite{huang2017translation} can tolerate a nonvanishing fraction of adversarial corruption, but it requires distance information in the algorithm and therefore addresses a substantially easier problem. This motivates us to seek a different formulation that exploits distance information more directly.

\subsection{Our Method and Contributions}

To address the above challenges, we propose TriP, a triangle-puzzle framework for robust translation averaging. Unlike existing cycle-based filtering methods, TriP does not use cycles only for outlier detection; instead, it uses triangles as the core units for estimating and synchronizing latent edge scales. More specifically, each triangle provides local relative edge-length information, which is then synchronized across overlapping triangles in the logarithmic domain to produce globally consistent edge-length estimates. These synchronized local estimates are subsequently aggregated for final camera location recovery. The main contributions of this work are as follows:
\begin{itemize}
\item We propose \emph{TriP}, a translation averaging framework designed to handle arbitrary generic camera configurations and adversarially corrupted direction measurements, in contrast to many existing solvers whose guarantees and designs are tailored primarily to stylized synthetic models, such as Gaussian locations, Erd\H{o}s--R\'enyi graphs, or uniformly corrupted directions.

\item TriP introduces a log-scale synchronization formulation for estimating pairwise distances. By parameterizing distances in the logarithmic domain, the formulation rules out the trivial zero-distance collapse by construction, without requiring additional anti-collapse constraints. This directly addresses a long-standing open issue in translation averaging.

\item We establish deterministic exact-recovery guarantees for TriP under arbitrary generic locations and arbitrary adversarial corruptions of direction vectors, provided that the maximum degree of the corrupted subgraph is bounded by a positive absolute constant. To our knowledge, this is the first translation averaging theory that tolerates a nonvanishing level of adversarial corruption as $n\to\infty$.

\item TriP is decentralized, parallelizable, and scalable to graphs with over one million cameras. On sparse large-scale graphs, it is up to three orders of magnitude faster than existing fastest methods, while achieving strong accuracy. We demonstrate its effectiveness on realistic synthetic benchmarks and real-world SfM datasets.
\end{itemize}

\section{The TriP Framework}
The key idea of TriP is to estimate local camera configurations from many triangles, each determined only up to a scale and a shift, and then stitch and aggregate these local estimates into a globally consistent location estimate. Ideally, clean triangles are mutually consistent after appropriate rescaling and form a large connected component, whereas corrupted triangles, although they may be locally consistent, are typically inconsistent with the majority of clean triangles. Thus, TriP can be viewed as solving a \textbf{triangle jigsaw puzzle}: each triangle piece may be freely rescaled and shifted, while some malicious ``junk'' pieces attempt to mislead the reconstruction.
\subsection{Step 1: Triangle Prefiltering}
\label{sec:triangle_shape}
This step first removes triangles whose direction measurements are clearly inconsistent; ideally, the three directions in each retained triangle should be coplanar.
For a graph triangle $t=(i,j,k)$ with $i<j<k$, let
$\bm a=\bm d_{ij}$, $\bm b=\bm d_{jk}$, and $\bm c=\bm d_{ki}$.
In the noiseless case, an oriented triangle satisfies $$\ell_1\bm a+\ell_2\bm b+\ell_3\bm c=\bm{0}$$ for some positive side lengths $(\ell_1,\ell_2,\ell_3)$. We estimate their relative
ratios by
$\bm h_t=
\bigl(\|\bm b\times\bm c\|,\ \|\bm c\times\bm a\|,\ \|\bm a\times\bm b\|\bigr)$,
which is proportional to the positive null vector of
$[\bm a\ \bm b\ \bm c]$ in the noiseless, consistently oriented case.
We then score the triangle by the normalized closure residual
\[
r_t=
\|h_{t,1}\bm a+h_{t,2}\bm b+h_{t,3}\bm c\|/
\|\bm h_t\|.
\]
Small $r_t$ indicates that the three directions are compatible with a
positive oriented triangle closure. This is stronger than the simple
volume test
$|\det[\bm a\ \bm b\ \bm c]|=|\bm a^\top(\bm b\times\bm c)|$,
which only checks coplanarity. We also discard nearly collinear triangles
by thresholding
$\min\{\|\bm b\times\bm c\|,\|\bm c\times\bm a\|,\|\bm a\times\bm b\|\}$.
For efficiency, TriP keeps a small pool of low-residual triangles incident to each camera edge.

\subsection{Step 2: ``Stitching" Triangles by Log-Scale Triangle Synchronization}
\label{sec:triangle_sync}

The side-length ratio tuple $\boldsymbol h_t$ determines only the shape of triangle $t$ up to an absolute scale. To stitch all triangles together, we rescale each triangle so that clean triangles are mutually compatible, i.e., their common edges have the same length. A natural approach is to synchronize triangle scales up to a global scale. Assign each triangle a scale $s_t>0$, so that triangle $t$ predicts the edge length
$\ell_{t,e}=s_t h_{t,e}$. If two triangles $t$ and $u$ share an edge $e$, their length predictions should agree:
\begin{equation}\label{eq:scale}
  s_t h_{t,e}\approx s_u h_{u,e}.  
\end{equation}

We synchronize scales in the logarithmic domain to avoid the trivial solution $s_t=0$. Let $z_t=\log s_t$. Taking logarithms in~\eqref{eq:scale} gives the additive constraint on each shared edge
\[
z_u-z_t\approx g_{tu},
\qquad
g_{tu}=\log\left({h_{t,e}}/{h_{u,e}}\right).
\]
Therefore, estimating $\{z_t\}$ exactly reduces to an one-dimensional translation synchronization problem on a triangle graph: each node is a triangle in the original viewing graph, and two nodes are connected if and only if the corresponding triangles share an edge. Let $\mathcal C$ be the edge set of this graph. Each constraint $(t,u)\in\mathcal C$ connects two triangles sharing a camera edge $e$ and uses the log-ratio $g_{tu}$. The initial weight is chosen as
$w_{tu}=\sqrt{\pi_t\pi_u}$,
where $\pi_t$ is a monotone reliability score computed from the closure residual $r_t$; smaller closure residuals give larger $\pi_t$. We solve
\[
\min_z\sum_{(t,u)\in\mathcal C}
w_{tu}\rho(|z_u-z_t-g_{tu}|).
\]
By default, we use the Cauchy loss $\rho(x)=\frac{c^2}{2}\log(1+x^2/c^2)$ with fixed $c=0.1$. The problem is solved by IRLS: at each iteration, we solve a weighted least-squares problem and update the edge weights from the previous residuals. Our implementation supports two least-squares solvers:
\textbf{(i) exact}: preconditioned conjugate gradient (PCG), avoiding explicit inversion of large matrices;
\textbf{(ii) fast}: a fully parallelizable local averaging scheme adapted from the Weiszfeld algorithm for rotation averaging~\cite{hartley2013rotation}. Since the fast updates can be sensitive to initialization, we initialize log-scales by propagating relative log-scales along a spanning tree with minimal average closure score $\sqrt{r_t r_u}$ from Step~1.

After synchronization, we compute the pair residual
$r_{tu}^{\mathrm{sync}}=|z_u-z_t-g_{tu}|$. Let $\mathcal N_t$ denote the neighboring triangles of $t$ in the sparse triangle graph. The average incident residual is
\[
\bar r_t=\frac{1}{|\mathcal N_t|}\sum_{u\in\mathcal N_t}r_{tu}^{\mathrm{sync}}.
\]
This quantity ranks the quality of triangles, with lower $\bar r_t$ indicating higher reliability.
\subsection{Step 3: Triangle Selection, Distance Estimation, and Location Recovery}
\label{sec:edge_recovery}

Triangles are sorted by increasing $\bar r_t$ and added sequentially. Let $S_k$ be the first $k$ triangles. An edge becomes active once it is supported by a prescribed minimum number of triangles in $S_k$. Let $C(S_k)$ be the largest connected component of the active-edge graph. TriP stops when $|C(S_k)|/n\ge\gamma$, where $\gamma$ controls the final node coverage; we use $\gamma=1$ by default.

For each active edge \(e\), the selected incident triangles provide length
proposals
\[
\mathcal P_e=\{\exp(z_t)h_t(e): t\in S_k,\ e\in t\}.
\]
Ideally, if \(e\) is clean and is supported by many clean triangles, then all
elements of \(\mathcal P_e\) should be nearly identical, since every clean
triangle proposes the same underlying edge length. Thus the dispersion of
\(\mathcal P_e\) provides a natural indicator of edge quality, giving another
advantage of the triangle-based framework. TriP estimates the edge length by $\hat\ell_e=\operatorname{median}(\mathcal P_e)$, defines the relative proposal dispersion $\delta_e=\operatorname{median}_{\lambda\in\mathcal P_e}|\lambda/\hat\ell_e-1|$, and sets (with $c=0.1$ in Cauchy reweighting function):
\[
w_e=\frac{1}{1+(\delta_e/c)^2}\cdot \frac{|\mathcal P_e|}{|\mathcal P_e|+1}.
\]
The two factors favor consistent length proposals and sufficient triangle support, respectively. The relative displacement on edge $e=(i,j)$ is then $\hat\ell_{ij}\bm d_{ij}$. Finally, locations are recovered by the same Cauchy objective for
\[
\min_{\bm{x}_1,\ldots,\bm{x}_n}
\sum_{(i,j)} w_{ij}\,
\rho\!\left(\|\bm{x}_i-\bm{x}_j-\hat\ell_{ij}\bm d_{ij}\|\right).
\]
Each least-squares subproblem is solved either by PCG, or by a faster Weiszfeld-type local averaging scheme. The exact-recovery theorem in Section~\ref{app:trip-scale-theory} analyzes the corresponding
annealed redescending version of parameter $c$ in reweighting. The fixed-Cauchy
implementation $c=0.1$ used in experiments is a practical variant in the presence of both outlier and additive noise. All the steps of TriP is summarized in Algorithm~\ref{alg:trip}.
\begin{algorithm}[t]
\caption{TriP}
\label{alg:trip}
\small
\begin{algorithmic}[1]
\Require Camera graph $G=([n],E)$, directions $\{\bm d_e\}_{e\in E}$, target coverage $\gamma$
\Ensure Estimated camera locations $\{\bm x_i\}_{i=1}^n$

\State Enumerate graph triangles $t=(i,j,k)$.
\State Compute local side ratios $\bm h_t$ by cross-product cofactors, compute closure residuals $r_t$, and keep a capped low-residual triangle pool.
\State Build the sparse shared-edge triangle graph. For each shared edge $e$ between triangles $t$ and $u$, set $g_{tu}=\log(h_{t,e}/h_{u,e})$.
\State Synchronize triangle log-scales: $z=\arg\min_z
\sum_{(t,u)\in \mathcal{C}}
w_{tu}
\rho\!\left(|z_u-z_t-g_{tu}|\right)$.
\State Rank triangles by average incident residual $\bar r_t$ and select the smallest prefix $S_k$ whose support-filtered graph reaches coverage $\gamma$.
\State Aggregate edge lengths by
$\hat \ell_e=
\operatorname{median}\{\exp(z_t)h_{t,e}: t\in S_k,\ e\in t\}$,
and compute edge weights $w_e$ from proposal support and dispersion.
\State Set $\bm v_e=\hat \ell_e\bm d_e$ and solve fixed-Cauchy displacement averaging:
\[
\{\bm x_i\}_{i=1}^n
=
\mathop{\arg\min}_{\{\bm x_i\}_{i=1}^n}
\sum_{e=(i,j)}
w_e \rho\!\left(\|\bm x_i-\bm x_j-\bm v_e\|\right).
\]
\end{algorithmic}
\end{algorithm}

\section{Theoretical Guarantees for TriP}
\label{sec:theory}

We give two complementary deterministic exact-recovery guarantees for the
ideal noiseless, full-triangle log-scale synchronization problem on the complete
observation graph.  The first is an algorithmic result for warm-started proximal
annealing with Cauchy and other redescending losses.  The second is a global
optimality result for a convex \(\ell_1\) synchronization objective.  These two
results concern different estimators and therefore have different admissible
corruption constants.

Let \(G=K_n\), and decompose the camera edges as
\[
    E(K_n)=E_g\sqcup E_b,
\]
where \(E_g\) and \(E_b\) are the clean and corrupted edges.  For
\(ij\in E_g\), assume
\[
    \bm d_{ij}=\bm d_{ij}^{\star}
    =
    \frac{\bm x_i^\star-\bm x_j^\star}
    {\|\bm x_i^\star-\bm x_j^\star\|},
\]
whereas for \(ij\in E_b\), the direction \(\bm d_{ij}\) is arbitrary and may
be adversarial.  Write
\[
    \ell_{ij}^\star=\|\bm x_i^\star-\bm x_j^\star\|.
\]
Let
\[
    \Delta_E=\max_{i\in[n]}\deg_{E_b}(i)
\]
be the maximum corrupted degree in the original camera graph.  In the ideal
full-triangle model, every nondegenerate graph triangle is retained and every
shared-edge triangle constraint is used.  Let \(\mathcal T_0\) denote the clean
triangles, namely the triangles whose three camera edges belong to \(E_g\).
We assume that every clean graph triangle is nondegenerate.  For the
reliability-weighted formulation, the normalized triangle scores satisfy
\[
    0\le \pi_t\le1,
    \qquad
    \pi_t=1\quad(t\in\mathcal T_0),
    \qquad
    w_{tu}=\sqrt{\pi_t\pi_u}.
\]
The unweighted formulation is obtained by taking \(\pi_t\equiv1\).

\begin{informaltheorem}[Proximal annealed-Cauchy recovery on the complete graph]
\label{thm:complete-main}
Assume in addition that the clean camera locations are generic.  Consider the
ideal full-triangle Cauchy objective
\[
    F_{\sigma}(z)
    =
    \sum_{(t,u)\in\mathcal C}
    w_{tu}\,
    \frac{\sigma^2}{2}
    \log\!\left(
        1+
        \frac{(z_u-z_t-g_{tu})^2}{\sigma^2}
    \right).
\]
Initialize by weighted least squares and choose the initial scale
\(\sigma_0\) sufficiently large as specified in the appendix.  At scale
\(\sigma_k\), run the exact proximal Cauchy-IRLS iteration with
\[
    \mu_n=\frac{n}{28},
    \qquad
    \sigma_{k+1}=\frac12\sigma_k,
\]
warm-started from the preceding stage.  Each quadratic proximal WLS subproblem
is solved exactly, and the stage output is any cluster-point fixed point of the
corresponding warm-started exact inner orbit.  There is a universal constant
\(C_{\rm Cau}<\infty\) such that, for all sufficiently large \(n\),
\[
    \Delta_E\le10^{-5}n
\]
implies
\[
    \min_{\alpha\in\mathbb R}
    \max_{t\in\mathcal T_0}
    |z_t^{(k)}-z_t^\star-\alpha|
    \le
    C_{\rm Cau}\frac{\Delta_E}{n}\sigma_k
    =
    C_{\rm Cau}\frac{\Delta_E}{n}\sigma_0\,2^{-k}.
\]
Consequently, the clean triangle log-scales converge uniformly to their true
values up to the unavoidable global additive gauge.
\end{informaltheorem}

The formal proof is given in Section~\ref{app:trip-scale-theory}.  The proximal
term is used to control the non-clean triangle variables during each frozen
weighted least-squares step; it vanishes from the first-order condition at a
fixed point.  The result is therefore a guarantee for the stated ideal
warm-started proximal annealing procedure, not for an arbitrary fixed point of
the nonconvex Cauchy objective and not for the fixed-Cauchy implementation with
loss scale \(0.1\) used in the experiments.

Section~\ref{app:trip-general-redescending-losses} extends the same argument to
a class of redescending losses of the form
\(\rho_\sigma(r)=q_\sigma(r^2)\).  The conditions include concavity of
\(q_\sigma\), a positive MM-weight floor on the clean residual window,
a selected score
\(\psi_\sigma(r)\) satisfying
\[
    |\psi_\sigma(r)|\le K\sigma,
\]
and a redescending score.  At nonsmooth points, \(\psi_\sigma(r)\) is an
admissible selected subgradient.  If \(a\) is the clean-basin radius and
\(\mathsf G_a\) is the corresponding clean Green-response constant, then,
for any degree fraction \(c\) satisfying the one-step smallness conditions in
the appendix, an annealing schedule is admissible whenever
\[
    \frac{6\mathsf G_a Kc}{a}
    <
    \frac{\sigma_{k+1}}{\sigma_k}
    <1,
    \qquad
    \Delta_E\le cn.
\]
For the common choice \(a=1/20\), the half-scale schedule
\(\sigma_{k+1}=\sigma_k/2\) is valid with the following explicit,
non-optimized constants:
\[
\begin{array}{c|cccc}
\rho
& \text{Cauchy}
& \text{Welsch}
& \text{Tukey}
& \text{truncated least squares}
\\
\hline
c
& 1.0\times10^{-5}
& 1.2\times10^{-5}
& 8.0\times10^{-6}
& 1.0\times10^{-5}
\end{array}
\]
The appendix also gives larger thresholds of order \(10^{-3}\) for a
conditional stationary-continuation statement in which a clean-basin
stationary point is supplied at every scale.  Those conditional thresholds are
not the unconditional proximal-MM constants above.

The convex \(\ell_1\) formulation permits a substantially larger deterministic
max-degree bound because its analysis concerns every global minimizer and does
not require a basin-invariance argument.  Consider
\[
    F_1(z)
    =
    \sum_{(t,u)\in\mathcal C}
    w_{tu}|z_u-z_t-g_{tu}|.
\]
For a clean-clean shared-edge row, suppose \(w_{tu}\ge w_{\rm g}>0\), and for
a row involving at least one non-clean triangle, suppose
\(0\le w_{tu}\le w_{\rm b}\).

\begin{informaltheorem}[Global \(\ell_1\) recovery on the complete graph]
\label{thm:complete-main-l1}
Assume \(n\ge4\).  If
\[
    \frac{\Delta_E}{n}
    <
    \frac{w_{\rm g}}
    {12(w_{\rm g}+w_{\rm b})},
\]
then every global minimizer of \(F_1\) recovers all clean triangle log-scales
up to one common additive gauge.  In particular, for the unweighted objective,
and for the TriP reliability weights above, one may take
\(w_{\rm g}=w_{\rm b}=1\), and the sufficient condition becomes
\[
    \Delta_E<\frac{n}{24}.
\]
Thus every global minimizer \(z^\sharp\) satisfies
\[
    z_t^\sharp=z_t^\star+\alpha,
    \qquad
    t\in\mathcal T_0,
\]
for some \(\alpha\in\mathbb R\).
\end{informaltheorem}

The proof in Section~\ref{app:trip-l1-scale-theory} is a direct robust-nullspace
argument.  The clean triangle-overlap graph has edge expansion at least
\(n/2-6\Delta_E\), whereas at most \(6\Delta_E\) rows involving non-clean
triangles can touch any clean triangle.  Consequently, every nonconstant
perturbation of the clean log-scales increases the clean \(\ell_1\) term more
than the profiled non-clean term can decrease it.  The weighted inequality
above is the resulting strict domination condition.  The \(\ell_1\) objective
is convex and can be written as a linear program, but the theorem requires an
exact global minimizer, or an independently certified global optimum; it is not
a guarantee for an arbitrarily terminated numerical iteration.

Both the annealed-redescending and global-\(\ell_1\) results are deterministic.
They impose no probability distribution on the camera locations, no
ShapeFit-style well-distribution condition, and no quantitative lower bound on
clean triangle angles.  The locations are deterministic; the Cauchy statement
uses genericity, and both theories require qualitative nondegeneracy so that the
clean side ratios and exact log-scale equations are well-defined.  None of these
assumptions introduces a quantitative location or angle parameter.  The
constants above are therefore independent of any location distribution,
well-distribution, or triangle-angle constant.

Exact clean-scale recovery gives globally consistent clean edge-length
proposals.  For a global \(\ell_1\) minimizer, there is a common positive scale
\(\lambda\) such that
\[
    \exp(z_t^\sharp)h_{t,e}=\lambda\ell_e^\star
\]
for every clean triangle \(t\) and clean edge \(e\subset t\).  For the
annealed redescending solutions, there is a sequence of positive global scale
factors \(\lambda_k\) such that
\[
    \max_{\substack{t\in\mathcal T_0\\ e\subset t}}
    \left|
       \frac{\exp(z_t^{(k)})h_{t,e}}
            {\lambda_k\ell_e^\star}-1
    \right|
    \longrightarrow0.
\]
If the downstream displacement stage is supplied with these clean edges and
lengths, the connected clean camera graph determines the camera locations up
to global translation and global scale.  The theory proves recovery of the
clean scale block; it does not require or assert recovery of the absolute
scales assigned to non-clean triangles.

Table~\ref{tab:complete-phase} compares the complete-graph consequences of the
available direction-only exact-recovery results.  Following the convention in
Cycle-Sync, \(\epsilon_b=pq\), where \(p\) is the Erd\H{o}s--R\'enyi edge
probability and \(q\) is the independent edge-corruption probability.  Setting
\(p=1\) gives the complete-graph bounds in the first three rows.  For the TriP
rows, if camera edges are independently corrupted with a fixed probability
strictly below the displayed deterministic max-degree constant, then the
corresponding condition on \(\Delta_E\) holds with high probability as
\(n\to\infty\).  The two TriP constants quantify different guarantees: the
redescending result is an algorithmic theorem for a nonconvex warm-started
proximal-MM path, whereas the \(1/24\) result concerns every global minimizer of
a convex \(\ell_1\) objective.  In the table, \(c_\rho\) denotes the
loss-specific constant displayed above.

\begin{table}[t]
\centering
\small
\caption{
Complete-graph corruption regimes implied by existing direction-only theory.
The first three rows are obtained by setting \(p=1\) in the
\(\epsilon_b=pq\) convention used in Cycle-Sync; their admissible independent
corruption probabilities vanish as \(n\to\infty\).  The TriP rows follow from
deterministic bounds on the maximum corrupted degree in the original camera
graph.}
\begin{tabular}{p{0.29\linewidth}p{0.34\linewidth}p{0.29\linewidth}}
\toprule
Theory
& Location model
& \makecell{Complete graph \(p=1\):\\ corruption probability \(q\)}
\\
\midrule
ShapeFit \cite{hand2018shapefit}
&
i.i.d. Gaussian camera locations in \(\mathbb R^3\)
&
\(q=O(\log^{-3}n)\to0\)
\\[0.65em]

LUD theory \cite{lerman2018exact}
&
i.i.d. Gaussian camera locations in \(\mathbb R^3\)
&
\(q=O(\log^{-9/2}n)\to0\)
\\[0.65em]

Cycle-Sync \cite{li2025cycle}
&
i.i.d. Gaussian camera locations in \(\mathbb R^3\)
&
\(q=O(\log^{-1/2}n)\to0\)
\\[0.65em]

TriP log-scale, proximal annealed redescending
&
deterministic generic, nondegenerate camera locations
&
\(q<c_\rho\), e.g., \(q<10^{-5}\) for Cauchy
\\[0.65em]

TriP log-scale, global \(\ell_1\)
&
deterministic nondegenerate camera locations
&
\(q<1/24\)
\\
\bottomrule
\end{tabular}
\label{tab:complete-phase}
\end{table}
\section{Experiments on Synthetic Data}
\label{sec:synthetic}
We first evaluate TriP on synthetic instances designed to stress-test
translation averaging under structured, coherent corruption. Unlike classical
synthetic benchmarks that often use i.i.d. Gaussian camera locations,
Erd\H{o}s--R\'enyi viewing graphs, and independently sampled random outliers
\cite{ozyesil2015robust,shi2018estimation}, our benchmark keeps the ground truth
fully controlled while reintroducing three structures common in real SfM:
low-dimensional camera trajectories, geometry-dependent viewing graphs, and
correlated corrupted measurements. This is important because independent random
outliers are usually inconsistent with local cycles, whereas real outliers may
be self-consistent due to repeated structures, wrong loop closures, or locally
bad pose estimates.

\paragraph{Synthetic scene model.}
We generate camera locations from two structured geometries. The first is a
planar grid. This models common ground-level or indoor captures, where camera
centers vary much more horizontally than vertically. For example, a photographer
walking around a room, along a street, or across a building facade often produces
a trajectory that is close to a two-dimensional surface. The second geometry is a torus. This abstracts a closed-loop capture, such as a
photographer walking around a building or an object, with small but nonzero
variation in height and radial distance. The torus is not simply another planar
case: it introduces a nontrivial loop structure.

\paragraph{Viewing graph.}
For each scene, we construct a clean viewing graph using a symmetric
$k_{\mathrm{good}}$-nearest-neighbor graph. This is closer to the visibility
pattern in SfM than an Erdős--Rényi graph, since nearby cameras are more likely
to share visual overlap and produce reliable relative measurements. We require
the clean graph to be connected, and every clean edge to participate in at least
one all-good triangle. This condition ensures that clean edges have local cycle
witnesses; without such witnesses, the status of an edge is statistically
ambiguous for any triangle- or cycle-based consistency test. This requirement
does not reveal the solution to TriP: all competing methods receive the same
observed graph and measurements, and corrupted edges may also participate in
coherent false triangles.

\paragraph{Coherent corruption models.}
We consider two ways of placing corrupted edges. In the
\emph{spatially uniform coherent corruption} setting, corrupted edges are spread
broadly across the geometry-dependent viewing graph, playing a role similar to
adversarial or cycle-consistent corruption models in prior synthetic studies.
They are sampled from local non-clean nearest-neighbor edges and far-range
candidate edges, capturing both local ambiguities and spurious long-range links.
In the \emph{clustered coherent corruption} setting, corrupted edges concentrate
around a subset of bad nodes, modeling localized failures such as unreliable
poses, repeated facades, motion blur, or systematically wrong matches. This is
more challenging because a few bad nodes can create many mutually reinforcing
incorrect measurements.

For clean edges, the observed direction is generated by tangent-plane Gaussian
perturbation, $\bm{d}_{ij}
=
\operatorname{Normalize}
\left(
{(\bm{x}_i-\bm{x}_j)}/{\|\bm{x}_i-\bm{x}_j\|}
+
\sigma \bm{\epsilon}_{ij}^{\perp}
\right)$,
where $\bm{\epsilon}_{ij}^{\perp}$ is orthogonal to the ground-truth direction.
For corrupted edges, all measurements are generated from a shared incorrect
latent layout $\bm{y}_i=(a_i,b_i,0)$, $a_i,b_i\sim\mathcal{N}(0,1)$, with
$\bm{d}_{ij}^{\mathrm{bad}}=
(\bm{y}_i-\bm{y}_j)/\|\bm{y}_i-\bm{y}_j\|$.
Thus, outliers form a coherent geometric distractor rather than arbitrary noise.
Additional uniform/clustered synthetic curves are in Appendix~\ref{sec:app-synthetic-curves}.

\paragraph{Evaluation and results.}
We compare TriP with representative translation averaging baselines, including
Cycle-Sync~\cite{li2025cycle},
LUD~\cite{ozyesil2015robust},
BATA~\cite{zhuang2018baseline},
and FusedTA~\cite{manam2024fusing}.
All estimated locations are aligned to the ground truth by robust similarity
alignment. We report median translation error in the main paper, with additional
mean and percentile errors in the supplement. Figure~\ref{fig:synthetic_runtime_lines_overlay_five} shows that TriP is both
robust and scalable. On the quantitative curves, TriP maintains near-zero median
translation error over a wider range of corruption levels on both grid and torus
geometries, whereas the baselines undergo sharp degradation as the corrupted
measurements become denser. The visualizations at $q=0.4$ explain this behavior:
the corrupted directions form coherent alternative structures that can dominate
full-graph solvers, while TriP still recovers a geometry closely aligned with
the ground truth. The runtime plot further shows that TriP-Fast preserves this
pipeline while scaling to substantially larger camera graphs than the baseline
methods.

\begin{figure*}[t]
    \centering
    \includegraphics[width=1\textwidth]{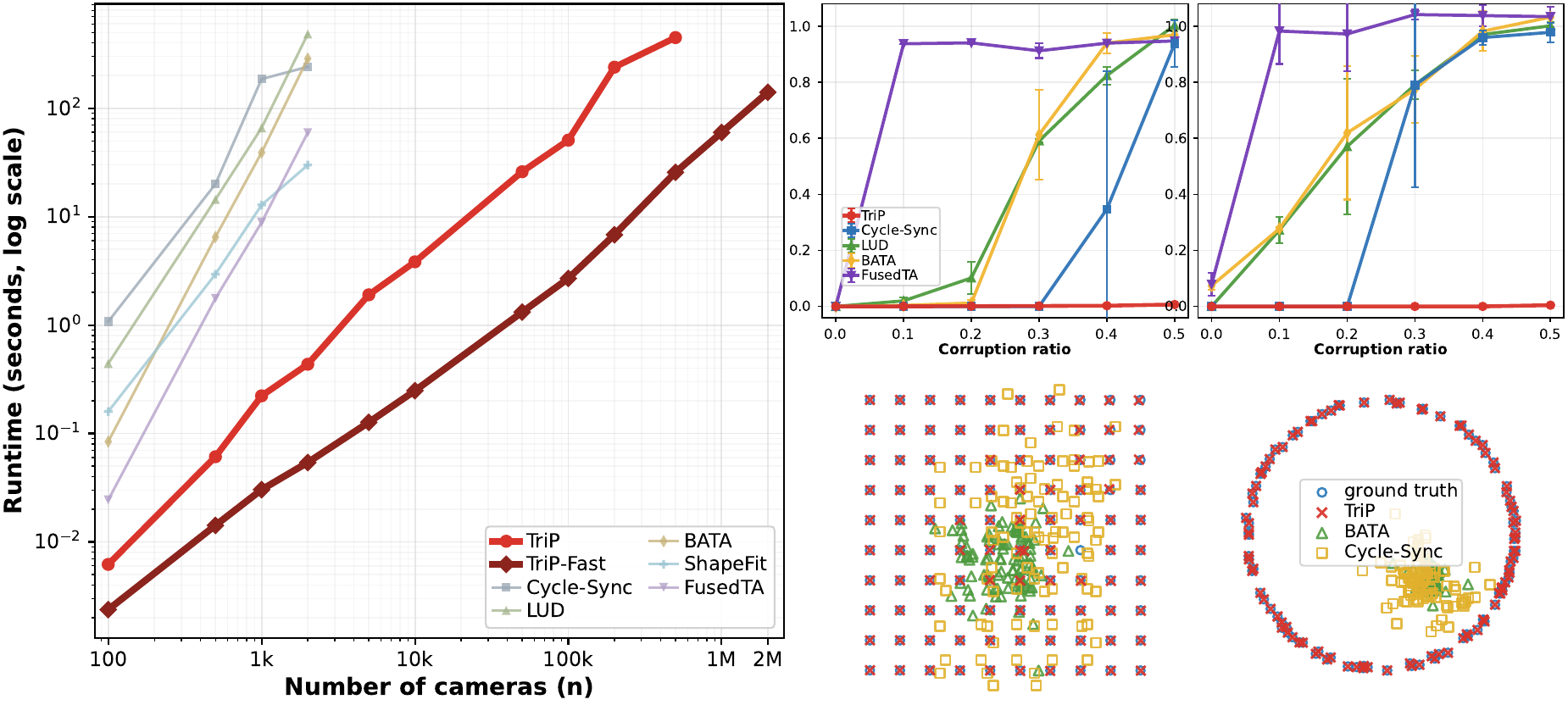}

    \caption{\textbf{Synthetic-data runtime scaling, quantitative robustness,
    and qualitative behavior.}
    Left: runtime scaling on torus instances with uniform corruption $q=0.1$ and
    $\sigma=0$, shown on logarithmic axes as the number of cameras $n$ increases.
    Right-top: median translation error v.s. corruption probability $q$ (horizontal axis) on grid and torus
    geometries as a function of the corruption ratio under uniform corruption with
    $\sigma=0$. Right-bottom:
    camera-location visualizations for grid and torus layouts at uniform
    corruption $q=0.4$, $\sigma=0$, and seed $0$, comparing ground truth with
    TriP, BATA, and Cycle-Sync. This panel demonstrates TriP's advantage in avoiding collapsed solutions.}
    \label{fig:synthetic_runtime_lines_overlay_five}
\end{figure*}

\section{Experiments on Real Data}
\label{sec:realdata}

We further evaluate TriP on real-world Structure-from-Motion
datasets, where measurement noise, graph sparsity, and outlier
patterns are substantially more complex than in the synthetic
setting. Our benchmark contains twelve image collections:
\texttt{courtyard}, \texttt{delivery\_area}, \texttt{electro},
\texttt{facade}, \texttt{kicker}, \texttt{meadow},
\texttt{office}, \texttt{pipes}, \texttt{relief},
\texttt{relief\_2}, \texttt{terrace}, and \texttt{terrains}.
These datasets cover diverse scene geometries and graph
structures, providing a challenging test bed for evaluating
robustness and localization accuracy.

\paragraph{Evaluation Protocol}

Our real-data evaluation is coverage-conditioned. For each target
coverage level
$\gamma \in \{0.7,0.8,0.9,1.0\}$, TriP returns a selected node
set $S_\gamma$. All methods are then evaluated on the same
node set $S_\gamma$. For the competing methods, we keep the input graph unchanged:
each baseline is solved on the full measurement graph, and only
the final alignment and error computation are restricted to
$S_\gamma$. This gives the baselines access to all available
measurements while ensuring that all methods are compared on
the same cameras. When $\gamma=1.0$, $S_\gamma$ is the full
node set, so the protocol reduces to the standard full-graph
evaluation. We compare against Cycle-Sync, LUD, BATA, ShapeFit, and
FusedTA. For each method, we align the estimated camera
locations to the ground truth by similarity transformation and
report the median, mean, and $90$-th percentile translation
errors. 
\begin{figure*}[!t]
    \centering
    \includegraphics[width=0.92\textwidth]{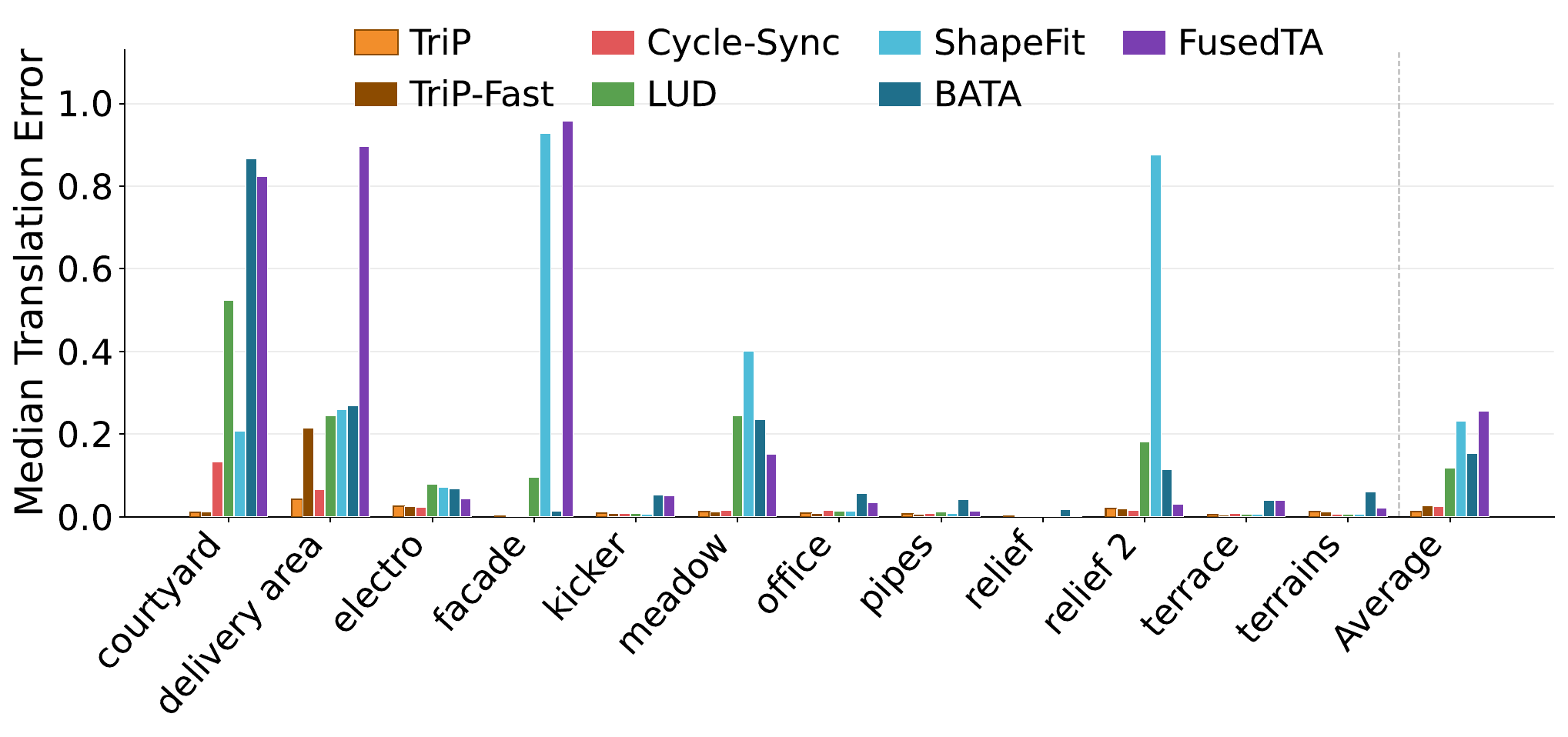}

    \vspace{0.25em}

    \includegraphics[width=0.92\textwidth]{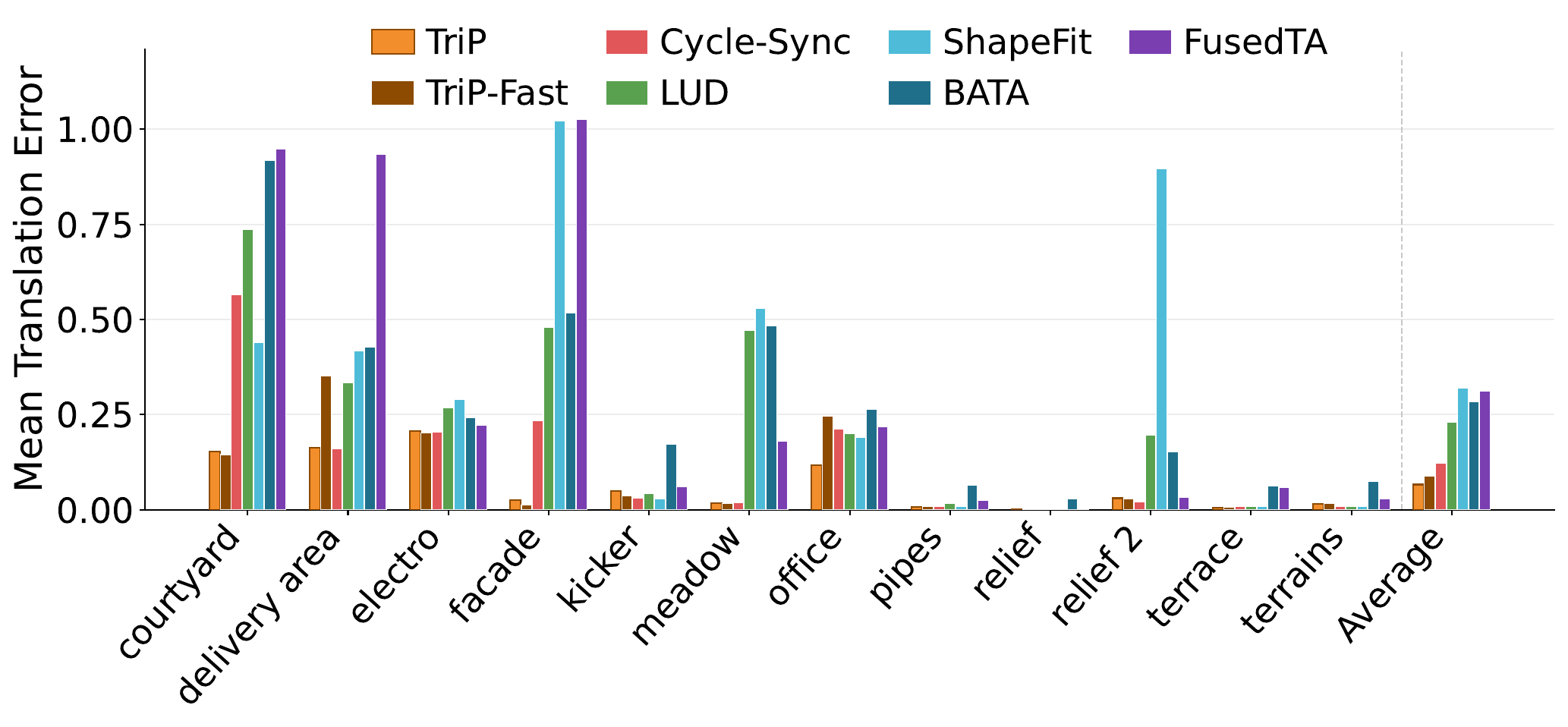}

    \caption{
    \textbf{Real-data performance at full coverage ($\gamma=1.0$).}
    Top: median translation error per dataset.
    Bottom: mean translation error per dataset.
    Lower is better. TriP achieves the lowest average error under both
    median and mean evaluation.}
    \label{fig:real_cov100_stacked}
\end{figure*}

\paragraph{Results Across Coverage Levels}

Across all coverage levels, TriP achieves the lowest
dataset-averaged median and mean translation errors among the
compared methods. As coverage increases, the evaluation becomes
harder because more challenging cameras are included; TriP
remains stable in this high-coverage regime. For median translation error, TriP obtains
$0.0093$, $0.0101$, $0.0102$, and $0.0124$ at
coverage levels $0.7$, $0.8$, $0.9$, and $1.0$,
respectively. The strongest baseline on average, Cycle-Sync,
obtains $0.0109$, $0.0120$, $0.0193$, and $0.0248$.
The trend is even clearer for mean translation error:
TriP obtains $0.0284$, $0.0425$, $0.0441$, and $0.0655$,
whereas Cycle-Sync obtains $0.0338$, $0.0521$, $0.0745$,
and $0.1227$. These results show that TriP is not only effective at selecting
recoverable subsets at low coverage, but also remains accurate
as increasingly difficult cameras are included. The improvement
in mean error further indicates that TriP reduces large
localization errors, rather than only improving the typical camera.

\paragraph{Full-Coverage Comparison}

The full-coverage case is especially important because it removes
any ambiguity associated with partial evaluation: all cameras are
included. At $\gamma=1.0$, TriP achieves the lowest average
median error ($0.0124$) and mean error ($0.0655$). By
comparison, Cycle-Sync obtains $0.0248$ median error and
$0.1227$ mean error, while the remaining baselines have larger
average errors. TriP also improves the tail of the error distribution. Its average
$p90$ translation error is $0.1351$, compared with $0.4286$ for
Cycle-Sync. Thus, the gain is not limited to the median camera;
TriP also reduces large reconstruction errors. At the dataset level, TriP is best or near-best on many scenes.
Its advantage is especially clear on \texttt{courtyard},
\texttt{facade}, \texttt{meadow}, \texttt{office}, and
\texttt{terrace}, where several baselines exhibit large error
growth. Cycle-Sync remains competitive on some scenes, including
\texttt{electro}, \texttt{kicker}, \texttt{relief},
\texttt{relief\_2}, and \texttt{terrains}, but the aggregate trend
consistently favors TriP. Overall, the real-data results support the main claim of the paper:
TriP identifies geometrically reliable parts of the measurement
graph and produces stable camera locations even in the presence
of difficult nodes and structured outliers.

\paragraph{Ablation Studies}

We conduct three ablation studies to understand why the components of TriP work
together. First, we compare TriP's triangle ranking with TAAB-style and random
selection baselines, testing whether the synchronization residual is a meaningful
reliability score rather than only a graph-size reduction heuristic. The results
show that TriP selects cleaner triangles earlier.

Second, we evaluate the estimated pairwise distances, since the downstream
displacement solver depends heavily on reliable initial edge lengths. Inaccurate
or collapsed distances can prevent the final Cauchy iteration from recovering
camera locations. This ablation shows that TriP produces more accurate and stable
distances with fewer collapses.

Third, we test solver compatibility on the selected graph, checking whether the
gain comes only from an easier subgraph or from the full pipeline: reliable
triangle selection, median distance aggregation on the induced graph, and Cauchy
displacement refinement. The results show that this combination is most effective.

Together, these ablations suggest that TriP's real-data gains come from the
interaction of cleaner triangle selection, reliable distance initialization, and
a downstream solver that exploits the selected graph structure. Additional plots
and per-dataset results are provided in the supplementary material.

\section{Conclusion and Limitations}

In this paper we proposed TriP, a triangle-based framework for robust
translation averaging. The key idea is to exploit triangle structures
in the viewing graph to infer relative edge-length ratios from
direction-only measurements. By synchronizing triangle scales across
the graph and aggregating the resulting edge-length proposals, the
method introduces an intermediate geometric layer between pairwise
direction observations and global camera location estimation.
Experiments on both synthetic and real-world datasets demonstrate
that the proposed approach improves robustness and accuracy compared
with several existing translation averaging methods.

Despite its effectiveness, the current framework has several
limitations. First, the method relies on the availability of sufficient
triangle structures in the viewing graph; performance may degrade when
the graph is extremely sparse. Second, the triangle synchronization
step assumes reasonably reliable direction measurements, and highly
adversarial noise may still affect the quality of the recovered scales.
Exploring more robust triangle selection and synchronization strategies
remains an interesting direction for future work.
\section*{Acknowledgment}
This work is supported by NSF award DMS-2514152.
\bibliographystyle{plain}
\bibliography{trip}

%%%%%%%%%%%%%%%%%%%%%%%%%%%%%%%%%%%%%%%%%%%%%%%%%%%%%%%%%%%%
\clearpage
\onecolumn
\appendix

\begin{center}
{\Large \textbf{Appendix}}
\end{center}

\section{More Results for ETH3D}
\label{sec:app-eth3d}

This section provides detailed per-dataset ETH3D \cite{schoeps2017cvpr} results complementary to the
main paper. All methods are aligned to the ground-truth camera locations and
evaluated on the same TriP-selected node set at full coverage, making the
comparison consistent across methods. Across the ETH3D scenes, TriP achieves the
strongest overall accuracy, while TriP-Fast preserves comparable accuracy with
substantially lower runtime. These results show that the proposed triangle-based
selection and estimation pipeline is effective not only in average performance,
but also consistently across individual real scenes.

\input{real_data/tables/table_real_cov1_accuracy}

Figure~\ref{fig:app-eth3d-camera-positions-3d} provides qualitative 3D
camera-location comparisons on representative ETH3D scenes. The estimates are
aligned to the ground truth by robust similarity alignment. Compared with
Cycle-Sync, TriP more consistently preserves the global camera geometry,
especially on scenes where Cycle-Sync exhibits larger spatial deviations.

\begin{figure}[H]
    \centering
    \includegraphics[width=0.98\textwidth]{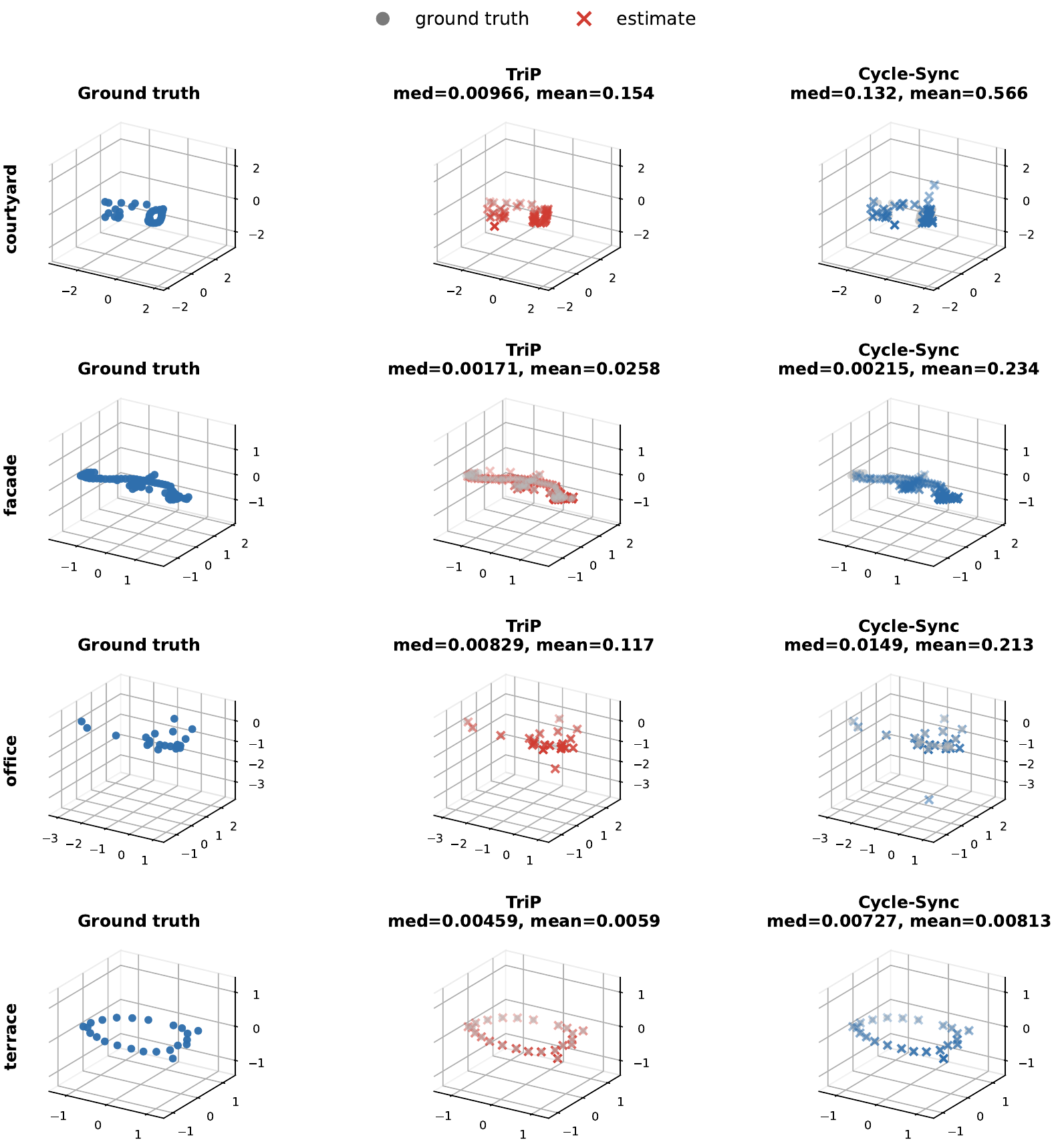}
    \caption{\textbf{3D camera-location visualizations on representative ETH3D
    scenes.} We compare ground-truth camera locations with TriP and Cycle-Sync
    estimates after robust similarity alignment. TriP better preserves the
    scene geometry on representative datasets where Cycle-Sync shows larger
    deviations.}
    \label{fig:app-eth3d-camera-positions-3d}
\end{figure}

\FloatBarrier

\section{More Results for Synthetic Data}
\label{sec:app-synthetic}
\subsection{Synthetic Error Curves}
\label{sec:app-synthetic-curves}

We first report the full synthetic error curves on the two structured camera
geometries used in the main paper. Figure~\ref{fig:app-synthetic-grid-curves}
shows the grid geometry, and Figure~\ref{fig:app-synthetic-torus-curves} shows
the torus geometry. In each figure, the four panels compare spatially uniform
and clustered coherent corruption under both noiseless directions and noisy
directions. These plots show the same general trend as the main synthetic
experiments: TriP remains stable over a wider range of corruption levels, while
the baselines often degrade rapidly once the coherent corrupted subgraph becomes
sufficiently dense.

\begin{figure*}[t]
    \centering
    \includegraphics[width=\textwidth]{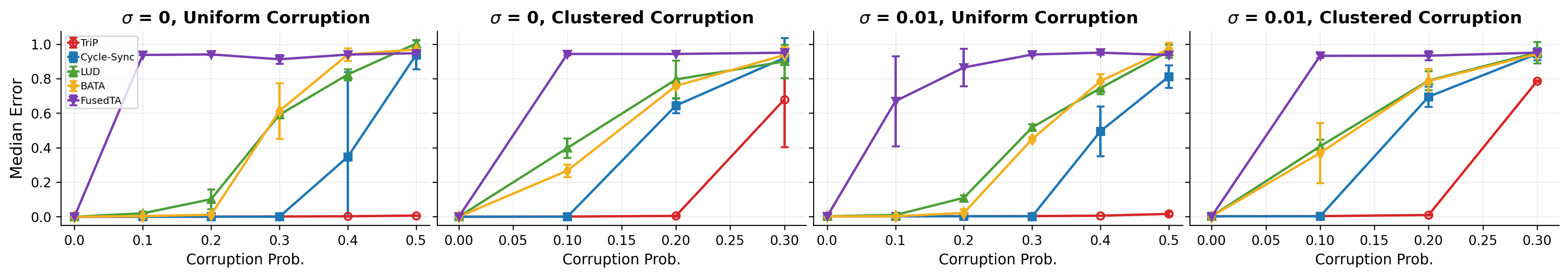}
    \vspace{-0.5em}
    \caption{\textbf{Synthetic error curves on the grid geometry.}
    We report median translation error versus corruption probability at full
    coverage. From left to right, the four panels show:
    (1) noiseless spatially uniform coherent corruption ($\sigma=0$),
    (2) noiseless clustered coherent corruption ($\sigma=0$),
    (3) noisy spatially uniform coherent corruption ($\sigma=0.01$), and
    (4) noisy clustered coherent corruption ($\sigma=0.01$).
    Uniform corruption spreads corrupted edges broadly across the graph, while
    clustered corruption concentrates them around bad nodes. Lower is better.
    TriP remains accurate over a wider range of corruption levels than the
    baselines.}
    \label{fig:app-synthetic-grid-curves}
    \vspace{-0.8em}
\end{figure*}

\begin{figure*}[t]
    \centering
    \includegraphics[width=\textwidth]{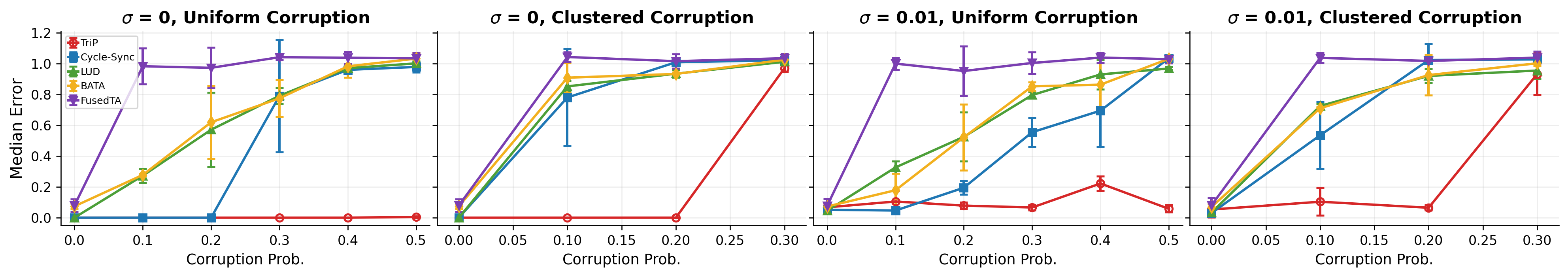}
    \vspace{-0.5em}
    \caption{\textbf{Synthetic error curves on the torus geometry.}
    We report median translation error versus corruption probability at full
    coverage. From left to right, the four panels show:
    (1) noiseless spatially uniform coherent corruption ($\sigma=0$),
    (2) noiseless clustered coherent corruption ($\sigma=0$),
    (3) noisy spatially uniform coherent corruption ($\sigma=0.01$), and
    (4) noisy clustered coherent corruption ($\sigma=0.01$).
    The torus geometry models a loop-like non-planar camera trajectory.
    Clustered corruption is especially challenging because corrupted
    measurements are localized and can mutually reinforce an incorrect geometry.
    Lower is better. TriP remains more stable than the baselines across both
    corruption patterns.}
    \label{fig:app-synthetic-torus-curves}
    \vspace{-0.8em}
\end{figure*}

\FloatBarrier

\subsection{Synthetic Camera-Location Visualizations}
\label{sec:app-synthetic-visualizations}

We visualize representative synthetic examples under uniform corruption with
corruption level $q=0.4$ and directional noise $\sigma=0.01$. Blue circles
indicate ground-truth camera locations, while red crosses indicate aligned
estimates. The grid example tests a structured planar layout, while the torus
example tests a trajectory-like non-planar layout. Under stronger simultaneous
outliers and directional noise, TriP remains well aligned with the ground truth
on both geometries, whereas several baselines drift away from the true
configuration or collapse to highly concentrated estimates. This supports the
main synthetic trend: TriP is robust when corrupted measurements and noisy
directions are both present.

\begin{figure*}[t]
    \centering
    \includegraphics[width=1\textwidth]{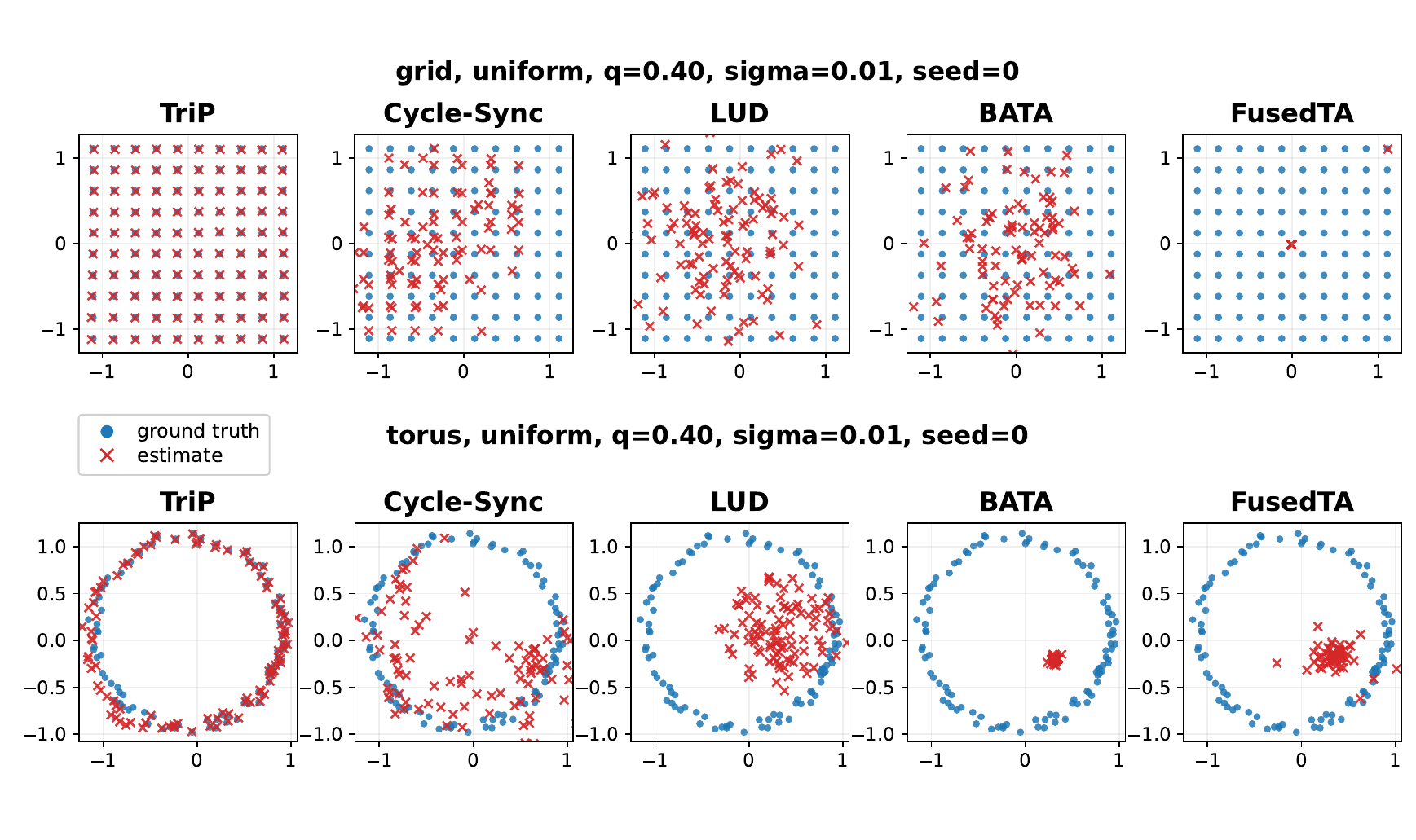}
    \vspace{-0.6em}
    \caption{\textbf{Synthetic camera-location visualizations under noisy
    uniform corruption.} Top: grid layout. Bottom: torus layout. We compare
    estimated locations against ground truth for $n=100$, uniform corruption
    level $q=0.4$, and $\sigma=0.01$.}
    \label{fig:app-synthetic-visualizations-noisy}
    \vspace{-1.0em}
\end{figure*}

\FloatBarrier

We visualize representative synthetic examples under uniform corruption with
corruption level $q=0.5$ and no directional noise ($\sigma=0$). Blue circles
indicate ground-truth camera locations, while red crosses indicate aligned
estimates. The grid example tests a structured planar layout, while the torus
example tests a trajectory-like non-planar layout. Even at this high corruption
level, TriP remains well aligned with the ground truth on both geometries,
whereas several baselines collapse to incorrect or highly concentrated
configurations. This supports the main synthetic trend: TriP is especially
robust when corrupted measurements become dense.

\begin{figure}[H]
    \centering
    \includegraphics[width=0.98\linewidth]{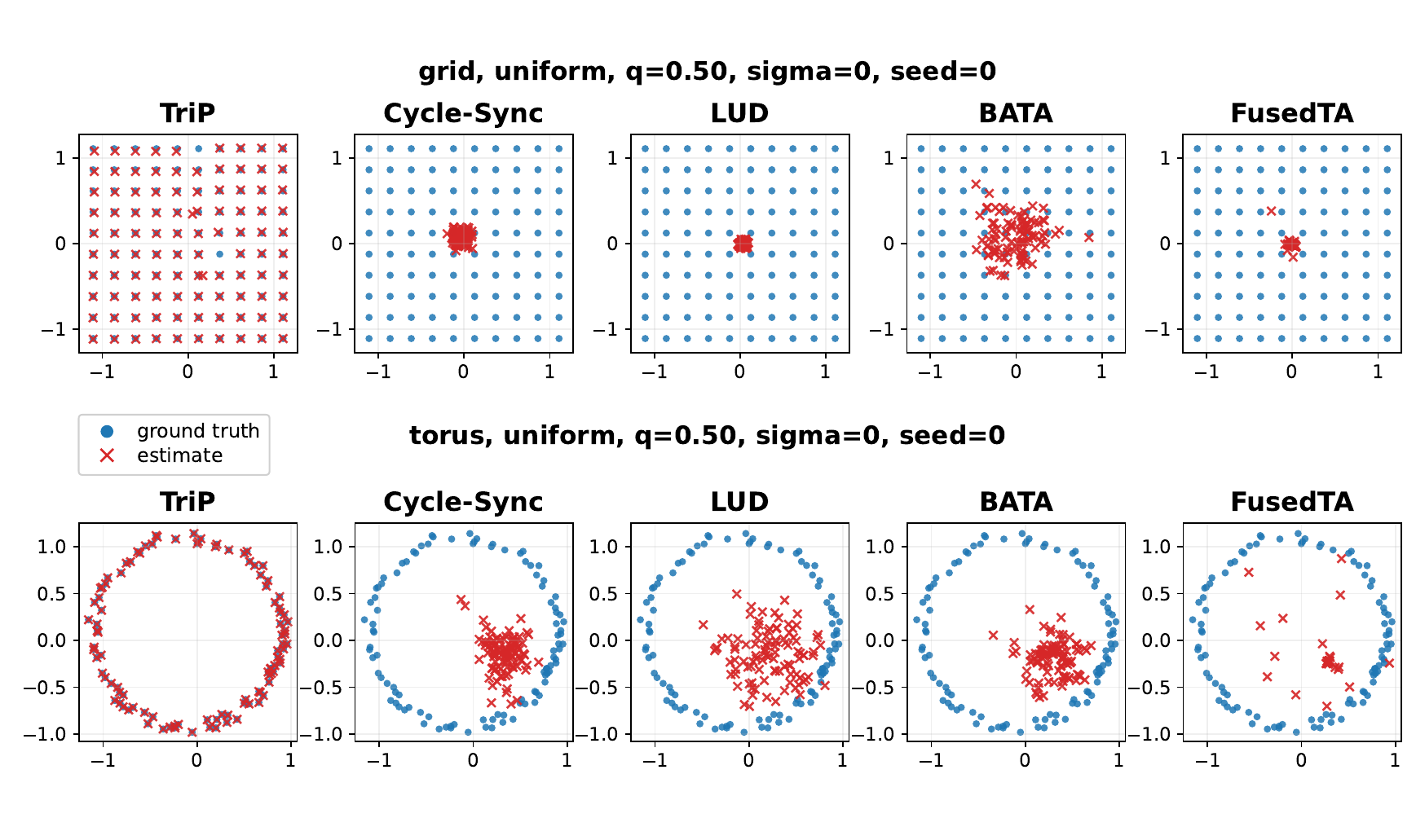}
    \vspace{-0.6em}
    \caption{\textbf{Synthetic camera-location visualizations under high
    corruption.} Top: grid layout. Bottom: torus layout. We compare estimated
    locations against ground truth for $n=100$, uniform corruption level
    $q=0.5$, and $\sigma=0$.}
    \label{fig:app-synthetic-visualizations-high-corruption}
    \vspace{-1.0em}
\end{figure}

\FloatBarrier
\vspace{-0.5em}

\section{Runtime}
\label{sec:app-runtime}
\vspace{-0.6em}

This section reports runtime results on ETH3D real-data scenes. All timings were
measured on a MacBook Air with an Apple M4 chip, 10 CPU cores (4 performance
cores and 6 efficiency cores), and 16 GB memory, without GPU acceleration. The
runtime comparison follows the same evaluation protocol as the accuracy
comparison. The results show a clear speed advantage for TriP-Fast: it is
consistently faster than the exact TriP implementation while preserving
comparable localization accuracy.

\subsection{Runtime on Real Data}
\label{sec:app-runtime-real}
\vspace{-0.4em}

On ETH3D, TriP-Fast gives the best overall runtime. The speedup is most visible
on larger scenes, where the C++ implementation reduces the cost of the core
triangle-selection and location-estimation steps.

\input{real_data/tables/table_real_cov1_runtime.tex}

\FloatBarrier

\section{Additional Experiments on 1DSfM Datasets}
\label{sec:app-1dsfm}

We additionally evaluate the proposed method on large-scale 1DSfM-style datasets.\cite{wilson2014robust, snavely2006photo}
These scenes contain substantially larger and more challenging measurement graphs
than ETH3D. We report translation accuracy and runtime separately. On these
larger scenes, TriP remains highly competitive in accuracy, and TriP-Fast
provides a clear runtime advantage. This indicates that the method is not tuned
only for small benchmark scenes, but also transfers to larger and more difficult
SfM graphs.

\subsection{Accuracy on 1DSfM Datasets}
\label{sec:app-1dsfm-accuracy}
The accuracy table reports mean and median translation errors for each scene.
TriP is among the best-performing methods on the aggregate results, and it is
especially strong in median error, indicating stable performance across cameras
within each scene.
\input{big_data/tables/table_big_cov1_accuracy.tex}

\subsection{Runtime on 1DSfM Datasets}
\label{sec:app-1dsfm-runtime}

The runtime table shows that TriP-Fast substantially reduces computation time
compared with both exact TriP and the classical baselines. This is particularly
important for the largest scenes, where several competing solvers become much
slower.

\input{big_data/tables/table_big_cov1_runtime.tex}

\FloatBarrier

\section{Ablation Studies}
\label{sec:app-ablation}

The ablation studies isolate the contribution of the main components of TriP:
distance estimation, triangle selection, and the downstream location solver.
The goal is not only to show that each component improves the numerical results,
but also to clarify why these components are necessary in the full pipeline.
Together, the experiments show that TriP works because reliable triangles yield
good initial edge-length estimates, and these estimates are well matched to the
final robust displacement solver.

The top panel of Figure~\ref{fig:app-ablation-studies} studies distance
estimation. The horizontal axis is the scale-aligned relative distance error,
and the vertical axis is its empirical cumulative distribution function (ECDF),
that is, the fraction of shared edges whose relative error is below a given
threshold. Curves that rise faster and stay closer to the top-left corner
therefore indicate more accurate and more concentrated distance estimates. This
ablation is important because the downstream displacement solver depends strongly
on the quality of the initial edge lengths: if many distances are collapsed or
severely biased, the final Cauchy iteration cannot reliably recover the correct
camera locations. The results show that TriP produces more accurate and more
stable edge-length estimates, indicating that triangle scale synchronization and
median aggregation provide a strong distance initialization.

The middle panel studies solver transfer on the selected graph. The horizontal
axis is the target coverage level, and the vertical axis is the resulting median
translation error after running each downstream solver on the graph induced by
TriP-selected triangles. Lower bars indicate better final location recovery.
This ablation tests whether the gain comes only from selecting an easier
subgraph, or from the full design of the pipeline: selecting reliable
triangles, aggregating median edge lengths on the induced graph, and then
applying a robust Cauchy displacement solver. The results show that our solver
performs best on this selected graph, suggesting that the selected structure and
the downstream optimization are particularly well matched.

\begin{figure}[t]
    \centering
    \includegraphics[width=1\textwidth]{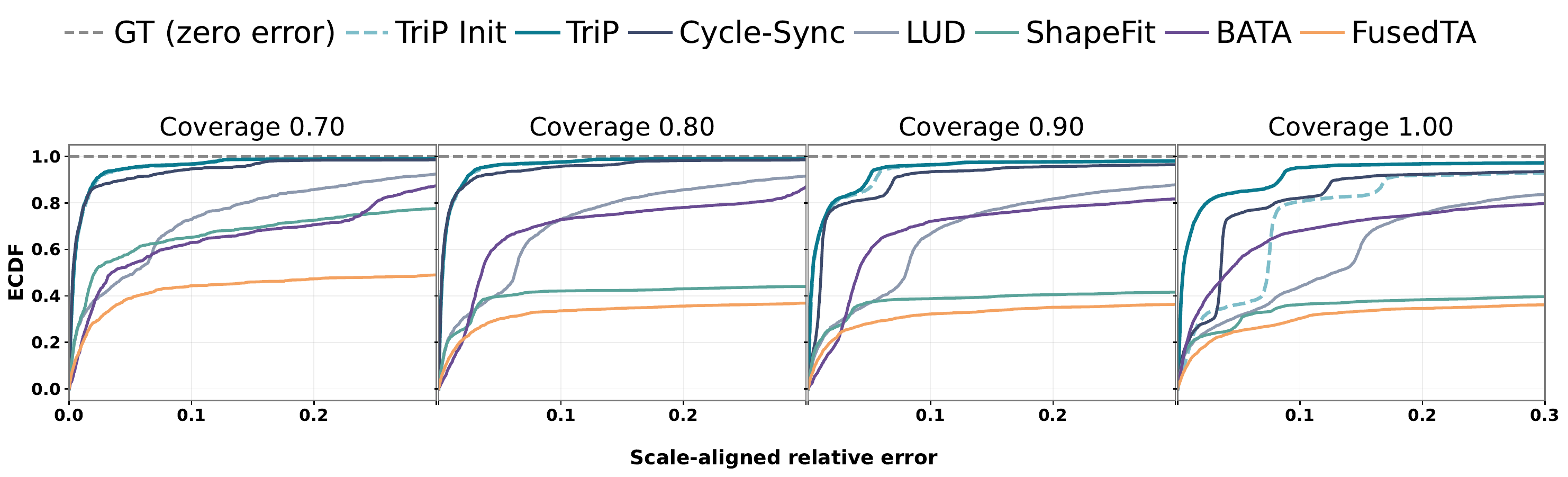}

    \vspace{0.45em}
    \includegraphics[width=1\textwidth]{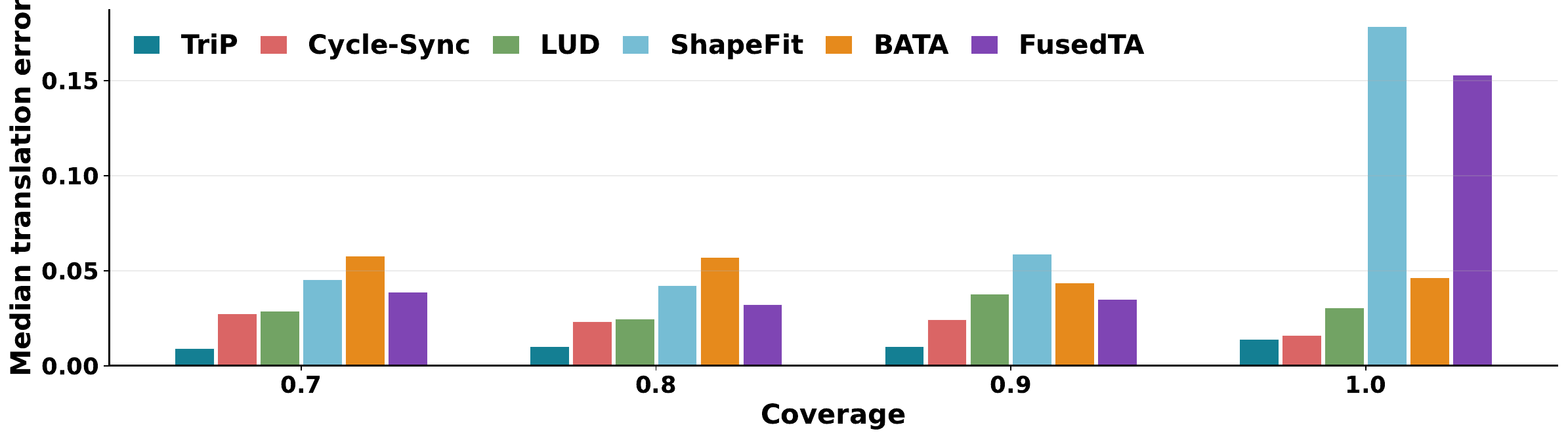}
    \vspace{0.45em}
    \includegraphics[width=1\textwidth]{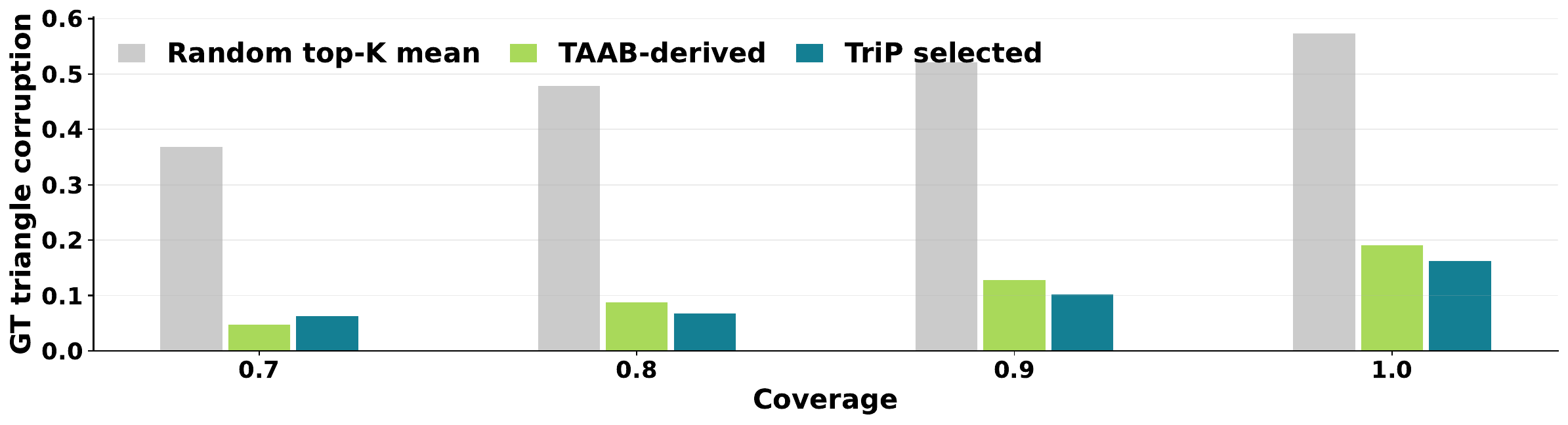}
        
    \caption{\textbf{Ablation studies.}
    Top: distance-estimation ablation using the ECDF of scale-aligned relative
    distance errors. Middle: solver-transfer ablation on graphs induced by
    TriP-selected triangles. Bottom: triangle-selection ablation using top-$k$
    ranked triangle quality.}
    \label{fig:app-ablation-studies}
\end{figure}

\FloatBarrier
The bottom panel studies triangle selection. The horizontal axis again denotes
the target coverage level, while the vertical axis measures the ground-truth
triangle corruption rate among the selected top-$k$ triangles. Lower values mean
that the selected triangle pool is cleaner. We compare TriP-selected triangles
with random top-$k$ selection and a TAAB-derived ranking. This ablation is
meaningful because the later distance and location stages can only be reliable
if the early triangle pool is sufficiently clean. The results show that the
synchronization residual used by TriP identifies cleaner triangles earlier in
the ranking, providing a better input set for scale synchronization and
edge-length aggregation.

\input{supply_theory}

\end{document}

%% file: real_data/tables/table_real_cov1_accuracy.tex
\begin{table}[H]
\centering
\caption{Translation accuracy on ETH3D at full coverage. $\bar{t}$ and $\hat{t}$ denote mean and median translation errors, respectively. Lower values are better.}
\label{tab:real_cov1_accuracy}
\small
\setlength{\tabcolsep}{3.0pt}
\renewcommand{\arraystretch}{1.08}
\begin{tabular}{lcccccccccccccc}
\toprule
\multirow{2}{*}{Scene} & \multicolumn{2}{c}{\textbf{TriP}} & \multicolumn{2}{c}{\textbf{TriP-Fast}} & \multicolumn{2}{c}{Cycle-Sync} & \multicolumn{2}{c}{LUD} & \multicolumn{2}{c}{BATA} & \multicolumn{2}{c}{ShapeFit} & \multicolumn{2}{c}{FusedTA} \\
\cmidrule(lr){2-3} \cmidrule(lr){4-5} \cmidrule(lr){6-7} \cmidrule(lr){8-9} \cmidrule(lr){10-11} \cmidrule(lr){12-13} \cmidrule(lr){14-15}
 & $\bar{t}$ & $\hat{t}$ & $\bar{t}$ & $\hat{t}$ & $\bar{t}$ & $\hat{t}$ & $\bar{t}$ & $\hat{t}$ & $\bar{t}$ & $\hat{t}$ & $\bar{t}$ & $\hat{t}$ & $\bar{t}$ & $\hat{t}$ \\
\midrule
courtyard & \underline{0.15} & \textbf{0.01} & \textbf{0.14} & \textbf{0.01} & 0.57 & \underline{0.13} & 0.74 & 0.52 & 0.92 & 0.87 & 0.44 & 0.21 & 0.95 & 0.82 \\
delivery area & \textbf{0.16} & \textbf{0.04} & 0.35 & 0.21 & \textbf{0.16} & \underline{0.07} & \underline{0.33} & 0.25 & 0.43 & 0.27 & 0.42 & 0.26 & 0.93 & 0.90 \\
electro & \underline{0.21} & \textbf{0.02} & \textbf{0.20} & \textbf{0.02} & \textbf{0.20} & \textbf{0.02} & 0.27 & 0.08 & 0.24 & 0.07 & 0.29 & 0.07 & 0.22 & \underline{0.04} \\
facade & \underline{0.03} & \textbf{0.00} & \textbf{0.01} & \textbf{0.00} & 0.23 & \textbf{0.00} & 0.48 & 0.10 & 0.52 & \underline{0.01} & 1.02 & 0.93 & 1.03 & 0.96 \\
kicker & 0.05 & \textbf{0.01} & \underline{0.04} & \textbf{0.01} & \textbf{0.03} & \textbf{0.01} & \underline{0.04} & \textbf{0.01} & 0.17 & \underline{0.05} & \textbf{0.03} & \textbf{0.01} & 0.06 & \underline{0.05} \\
meadow & \textbf{0.02} & \textbf{0.01} & \textbf{0.02} & \textbf{0.01} & \textbf{0.02} & \underline{0.02} & 0.47 & 0.24 & 0.48 & 0.24 & 0.53 & 0.40 & \underline{0.18} & 0.15 \\
office & \textbf{0.12} & \textbf{0.01} & 0.25 & \textbf{0.01} & 0.21 & \textbf{0.01} & 0.20 & \textbf{0.01} & 0.26 & 0.06 & \underline{0.19} & \textbf{0.01} & 0.22 & \underline{0.03} \\
pipes & \textbf{0.01} & \textbf{0.01} & \textbf{0.01} & \textbf{0.01} & \textbf{0.01} & \textbf{0.01} & \underline{0.02} & \textbf{0.01} & 0.06 & \underline{0.04} & \textbf{0.01} & \textbf{0.01} & \underline{0.02} & \textbf{0.01} \\
relief & \textbf{0.00} & \textbf{0.00} & \textbf{0.00} & \textbf{0.00} & \textbf{0.00} & \textbf{0.00} & \textbf{0.00} & \textbf{0.00} & \underline{0.03} & \underline{0.02} & \textbf{0.00} & \textbf{0.00} & \textbf{0.00} & \textbf{0.00} \\
relief 2 & \underline{0.03} & \textbf{0.02} & \underline{0.03} & \textbf{0.02} & \textbf{0.02} & \textbf{0.02} & 0.20 & 0.18 & 0.15 & 0.11 & 0.90 & 0.88 & \underline{0.03} & \underline{0.03} \\
terrace & \textbf{0.01} & \textbf{0.00} & \textbf{0.01} & \textbf{0.00} & \textbf{0.01} & \underline{0.01} & \textbf{0.01} & \underline{0.01} & \underline{0.06} & 0.04 & \textbf{0.01} & \underline{0.01} & \underline{0.06} & 0.04 \\
terrains & \underline{0.02} & \textbf{0.01} & \underline{0.02} & \textbf{0.01} & \textbf{0.01} & \textbf{0.01} & \textbf{0.01} & \textbf{0.01} & 0.07 & 0.06 & \textbf{0.01} & \textbf{0.01} & 0.03 & \underline{0.02} \\
\midrule
\textbf{Average} & \textbf{0.07} & \textbf{0.01} & \underline{0.09} & 0.03 & 0.12 & \underline{0.02} & 0.23 & 0.12 & 0.28 & 0.15 & 0.32 & 0.23 & 0.31 & 0.26 \\
\bottomrule
\end{tabular}
\end{table}

%% file: real_data/tables/table_real_cov1_runtime.tex
\begin{table}[H]
\centering
\caption{Runtime comparison on ETH3D at full coverage. Runtime is reported in seconds. Lower values are better.}
\label{tab:real_cov1_runtime}
\small
\setlength{\tabcolsep}{4.5pt}
\renewcommand{\arraystretch}{1.08}
\begin{tabular}{lccccccc}
\toprule
\multirow{2}{*}{Scene} & \multicolumn{1}{c}{\textbf{TriP}} & \multicolumn{1}{c}{\textbf{TriP-Fast}} & \multicolumn{1}{c}{Cycle-Sync} & \multicolumn{1}{c}{LUD} & \multicolumn{1}{c}{BATA} & \multicolumn{1}{c}{ShapeFit} & \multicolumn{1}{c}{FusedTA} \\
\cmidrule(lr){2-2} \cmidrule(lr){3-3} \cmidrule(lr){4-4} \cmidrule(lr){5-5} \cmidrule(lr){6-6} \cmidrule(lr){7-7} \cmidrule(lr){8-8}
 & Time(s) & Time(s) & Time(s) & Time(s) & Time(s) & Time(s) & Time(s) \\
\midrule
courtyard & \underline{0.019} & \textbf{0.0062} & 2.22 & 0.310 & 0.043 & 0.110 & 0.460 \\
delivery area & \underline{0.0075} & \textbf{0.0025} & 0.370 & 0.210 & 0.021 & 0.096 & 0.210 \\
electro & \underline{0.0097} & \textbf{0.0036} & 0.390 & 0.220 & 0.016 & 0.066 & 0.210 \\
facade & \underline{0.060} & \textbf{0.027} & 2.52 & 1.32 & 0.064 & 0.250 & 0.710 \\
kicker & \underline{0.0073} & \textbf{0.0032} & 0.320 & 0.160 & 0.019 & 0.065 & 0.130 \\
meadow & \underline{0.0013} & \textbf{2.7e-04} & 0.096 & 0.039 & 0.0070 & 0.0047 & 0.083 \\
office & \underline{0.0033} & \textbf{0.0012} & 0.140 & 0.079 & 0.0076 & 0.053 & 0.110 \\
pipes & \underline{0.0013} & \textbf{3.8e-04} & 0.073 & 0.041 & 0.0054 & 0.0058 & 0.084 \\
relief & \underline{0.0020} & \textbf{9.0e-04} & 0.081 & 0.054 & 0.0070 & 0.0031 & 0.061 \\
relief 2 & \underline{0.0055} & \textbf{0.0017} & 0.190 & 0.095 & 0.0082 & 0.019 & 0.120 \\
terrace & \underline{0.0047} & \textbf{0.0020} & 0.150 & 0.080 & 0.0081 & 0.052 & 0.120 \\
terrains & \underline{0.012} & \textbf{0.0041} & 0.460 & 0.270 & 0.018 & 0.093 & 0.220 \\
\midrule
\textbf{Average} & \underline{0.011} & \textbf{0.0044} & 0.580 & 0.240 & 0.019 & 0.068 & 0.210 \\
\bottomrule
\end{tabular}
\end{table}

%% file: big_data/tables/table_big_cov1_accuracy.tex
\begin{center}
\captionof{table}{\textbf{Accuracy on large-scale 1DSfM datasets.} For each method, $\bar{t}$ and $\hat{t}$ denote the mean and median translation errors, respectively. Lower values are better. Bold and underlined entries indicate the best and second-best results for each scene and metric.}
\label{tab:app-1dsfm-accuracy}
\footnotesize
\setlength{\tabcolsep}{2.1pt}
\renewcommand{\arraystretch}{1.12}
\begin{tabular}{lcccccccccccccc}
\toprule
\multirow{2}{*}{Scene} & \multicolumn{2}{c}{\textbf{TriP}} & \multicolumn{2}{c}{\textbf{TriP-Fast}} & \multicolumn{2}{c}{LUD} & \multicolumn{2}{c}{Cycle-Sync} & \multicolumn{2}{c}{BATA} & \multicolumn{2}{c}{FusedTA} & \multicolumn{2}{c}{ShapeFit} \\
\cmidrule(lr){2-3} \cmidrule(lr){4-5} \cmidrule(lr){6-7} \cmidrule(lr){8-9} \cmidrule(lr){10-11} \cmidrule(lr){12-13} \cmidrule(lr){14-15}
 & $\bar{t}$ & $\hat{t}$ & $\bar{t}$ & $\hat{t}$ & $\bar{t}$ & $\hat{t}$ & $\bar{t}$ & $\hat{t}$ & $\bar{t}$ & $\hat{t}$ & $\bar{t}$ & $\hat{t}$ & $\bar{t}$ & $\hat{t}$ \\
\midrule
Alamo & 2.36 & \textbf{0.41} & \underline{1.72} & \textbf{0.41} & \textbf{1.69} & \underline{0.44} & 1.83 & \textbf{0.41} & 2.34 & 0.59 & 1.98 & 0.63 & 1.78 & \textbf{0.41} \\
Ellis Island & 25.92 & \textbf{4.94} & 26.06 & \underline{9.80} & \underline{21.98} & 18.64 & \textbf{21.90} & 18.47 & 25.38 & 23.07 & 28.10 & 23.12 & 28.26 & 23.19 \\
Gendarmenmarkt & 43.44 & 24.62 & 45.26 & 26.96 & \textbf{38.17} & \underline{16.54} & \underline{38.57} & \textbf{16.04} & 44.69 & 19.09 & 50.08 & 28.91 & 57.85 & 33.34 \\
Madrid Metropolis & 6.78 & \textbf{1.42} & \underline{4.98} & \underline{1.45} & 5.94 & 1.84 & \textbf{4.61} & 1.50 & 13.54 & 2.61 & 12.53 & 2.37 & 23.28 & 1.46 \\
Montreal Notre Dame & 2.21 & \textbf{0.48} & 3.40 & 0.52 & \underline{1.21} & 0.53 & \textbf{1.01} & \underline{0.51} & 1.96 & 0.84 & 1.70 & 0.82 & 3.37 & 0.54 \\
NYC Library & 3.17 & \textbf{0.70} & \underline{2.96} & 0.78 & 6.53 & 2.52 & \textbf{2.76} & \underline{0.75} & 7.34 & 1.70 & 14.41 & 13.61 & 14.27 & 13.41 \\
Notre Dame & 0.98 & \textbf{0.22} & 1.07 & \underline{0.23} & \underline{0.84} & 0.28 & \textbf{0.75} & \textbf{0.22} & 1.71 & 0.36 & 4.03 & 0.41 & 0.96 & \underline{0.23} \\
Piazza del Popolo & 5.83 & \textbf{1.16} & 8.46 & 7.64 & \textbf{5.26} & 1.51 & \underline{5.54} & \underline{1.23} & 6.60 & 1.59 & 6.30 & 1.65 & 5.89 & 1.26 \\
Piccadilly & \underline{3.71} & \textbf{1.18} & 4.96 & \underline{1.27} & 9.88 & 1.65 & \textbf{3.37} & 1.34 & 4.12 & 1.56 & 7.95 & 1.77 & 15.03 & 14.21 \\
Roman Forum & \textbf{8.31} & \textbf{1.55} & 10.30 & 4.46 & \underline{8.32} & \underline{2.12} & 55.63 & 39.48 & 17.68 & 6.36 & 31.86 & 14.44 & 39.52 & 26.56 \\
Tower of London & \textbf{11.63} & \underline{2.42} & \underline{12.28} & 2.53 & 17.47 & 3.92 & 15.76 & 2.59 & 14.48 & 3.37 & 19.06 & 3.94 & 16.80 & \textbf{2.34} \\
Union Square & \textbf{10.98} & \textbf{5.41} & 13.95 & \underline{5.86} & 11.69 & 7.43 & \underline{11.58} & 7.19 & 12.14 & 7.58 & 17.86 & 11.84 & 19.19 & 13.22 \\
Vienna Cathedral & 15.39 & \textbf{1.78} & 17.04 & 2.45 & \underline{12.80} & 5.83 & \textbf{8.65} & \underline{2.03} & 15.63 & 4.84 & 36.95 & 28.69 & 36.72 & 28.24 \\
Yorkminster & 6.74 & \textbf{1.35} & 7.99 & 1.53 & \textbf{5.39} & 2.48 & \underline{5.95} & \underline{1.37} & 17.55 & 8.49 & 18.27 & 7.14 & 29.29 & 19.38 \\
\midrule
\textbf{Average} & \underline{10.53} & \textbf{3.40} & 11.46 & 4.71 & \textbf{10.51} & \underline{4.70} & 12.71 & 6.65 & 13.23 & 5.86 & 17.93 & 9.95 & 20.87 & 12.70 \\
\bottomrule
\end{tabular}
\end{center}

%% file: big_data/tables/table_big_cov1_runtime.tex
\begin{table}[H]
\centering
\caption{\textbf{Runtime on large-scale 1DSfM datasets.}
Runtime is reported in seconds. Lower values are better. Bold and underlined
entries indicate the fastest and second-fastest methods for each scene.}
\label{tab:app-1dsfm-runtime}

\small
\setlength{\tabcolsep}{4.8pt}
\renewcommand{\arraystretch}{1.12}
\begin{tabular}{lccccccc}
\toprule
Scene & \textbf{TriP} & \textbf{TriP-Fast} & BATA & FusedTA & ShapeFit & LUD & Cycle-Sync \\
\midrule
Alamo & 5.83 & \textbf{1.81} & \underline{2.60} & 27.79 & 31.39 & 138.34 & 163.76 \\
Ellis Island & 0.870 & \textbf{0.370} & \underline{0.650} & 6.73 & 4.14 & 10.19 & 22.60 \\
Gendarmenmarkt & 2.48 & \textbf{0.800} & \underline{1.48} & 15.56 & 21.15 & 58.47 & 80.49 \\
Madrid Metropolis & 0.820 & \textbf{0.340} & \underline{0.690} & 3.43 & 4.45 & 20.41 & 25.64 \\
Montreal Notre Dame & \underline{1.99} & \textbf{1.05} & \underline{1.99} & 21.55 & 18.81 & 70.90 & 76.85 \\
NYC Library & 0.860 & \textbf{0.310} & \underline{0.660} & 6.81 & 4.16 & 15.34 & 23.92 \\
Notre Dame & 5.34 & \textbf{2.32} & \underline{2.93} & 47.81 & 41.48 & 173.09 & 199.04 \\
Piazza del Popolo & 1.38 & \textbf{0.400} & \underline{0.850} & 5.20 & 5.88 & 20.43 & 28.21 \\
Piccadilly & 24.85 & \textbf{5.21} & \underline{10.97} & 142.16 & 347.36 & 571.12 & 1177.01 \\
Roman Forum & 3.88 & \textbf{0.930} & \underline{2.41} & 23.18 & 41.88 & 105.73 & 93.17 \\
Tower of London & \underline{0.980} & \textbf{0.330} & 1.04 & 9.21 & 9.00 & 19.12 & 27.03 \\
Union Square & 1.95 & \textbf{0.350} & \underline{1.15} & 9.28 & 15.29 & 22.23 & 29.54 \\
Vienna Cathedral & 8.06 & \textbf{2.27} & \underline{3.25} & 47.97 & 55.41 & 150.39 & 265.17 \\
Yorkminster & 1.29 & \textbf{0.440} & \underline{1.23} & 11.37 & 10.29 & 28.34 & 36.54 \\
\midrule
\textbf{Average} & 4.33 & \textbf{1.21} & \underline{2.28} & 27.00 & 43.62 & 100.29 & 160.64 \\
\bottomrule
\end{tabular}
\end{table}

%% file: supply_theory.tex
\makeatletter
\@ifundefined{theorem}{\newtheorem{theorem}{Theorem}[section]}{}
\@ifundefined{corollary}{\newtheorem{corollary}[theorem]{Corollary}}{}
\@ifundefined{proposition}{\newtheorem{proposition}[theorem]{Proposition}}{}
\@ifundefined{lemma}{\newtheorem{lemma}[theorem]{Lemma}}{}
\@ifundefined{definition}{\newtheorem{definition}[theorem]{Definition}}{}
\@ifundefined{remark}{\newtheorem{remark}[theorem]{Remark}}{}
\@ifundefined{claim}{\newtheorem{claim}[theorem]{Claim}}{}
\makeatother

\makeatletter
\let\TriPproof@saved@Call\Call
\providecommand{\Call}{}
\renewcommand{\Call}{C}
\makeatother

\providecommand{\vcentcolon}{\mathrel{\mathop:}}

\providecommand{\argmin}{\operatorname*{arg\,min}}
\providecommand{\spanop}{\operatorname{span}}
\providecommand{\osc}{\operatorname{osc}}
\providecommand{\range}{\operatorname{range}}
\providecommand{\Fix}{\operatorname{Fix}}
\providecommand{\one}{\mathbf{1}}
\providecommand{\defeq}{\vcentcolon=}
\providecommand{\Tzero}{T_{0}}
\providecommand{\Tb}{T_{\mathrm b}}
\providecommand{\Tall}{T}
\providecommand{\Czero}{C_{0}}
\providecommand{\Cb}{C_{\mathrm b}}
\providecommand{\Ezero}{E_{0}}
\providecommand{\Eb}{E_{\mathrm b}}
\providecommand{\dtri}{d^{\triangle}_{\mathrm b}}
\providecommand{\qmin}{q_{\min}}
\providecommand{\sig}{\sigma}
\providecommand{\psisig}{\psi_{\sigma}}
\providecommand{\err}{\mathcal E}
\providecommand{\basin}{\mathcal B}
\providecommand{\LS}{\mathrm{LS}}
\providecommand{\IRLS}{\mathrm{IRLS}}
\providecommand{\bad}{\mathrm b}
\providecommand{\good}{\mathrm g}
\providecommand{\wbase}{\bar w}
\providecommand{\Cauchy}{\rho_{\sigma}}
\providecommand{\Proj}{\mathsf P}
\providecommand{\Gconst}{\mathsf G}
\providecommand{\astar}{a_{\star}}
\providecommand{\mstar}{m_{\star}}
\providecommand{\mustar}{\mu_n}
\providecommand{\cstar}{c_{\star}}
\providecommand{\jstar}{J_{\star}}

\section{Theory for Annealed Cauchy IRLS}
\label{app:trip-scale-theory}
\label{app:complete-graph-max-degree-recovery-proximal-cauchy}

\paragraph{Overview.}
This section proves a deterministic exact-recovery guarantee for the annealed Cauchy log-scale synchronization stage of TriP on the complete camera graph.  The theorem is stated only in terms of the maximum corrupted camera-edge degree.  Quantities such as the clean witness number \(q_{\min}\), the clean-boundary bad degree, the clean Green response, and the clean basin are introduced only inside the proof and are derived from the max-degree condition.

The proof uses a proximal version of Cauchy IRLS.  At each scale and each frozen IRLS step we solve the quadratic subproblem with the additional term
\[
    \frac{\mustar}{2}\norm{z-z^m}_2^2,
    \qquad
    \mustar=\kappa n .
\]
The proximal term is used only to stabilize the frozen WLS map analyzed below.  At a fixed point it disappears from the first-order condition, so every fixed point of the proximal IRLS map is a stationary point of the original weighted Cauchy objective at that scale.  The damping also gives a uniform inverse bound for the nuisance block generated by non-clean triangles:
\[
    (C^\top\Theta C+\mustar I)^{-1}.
\]
This bound is what allows the influence of all non-clean triangle variables, including internal bad-bad shared-edge rows, to be controlled by local degree quantities.

Everything below is noiseless exact-recovery theory for annealed Cauchy scales \(\sig_k\downarrow0\).  No fixed-Cauchy exact-recovery theorem is asserted.

At each annealing scale, the output is defined through the exact proximal orbit warm-started from the previous-scale output.  The one-step estimate keeps the entire orbit in the clean basin, while the descent and compactness argument yields cluster points and shows that every cluster point is a fixed point.  The annealing induction uses one such basin fixed point.  No claim is made that every fixed point of the nonconvex Cauchy IRLS map lies in the clean basin.

\subsection{Clean theorem statement}

Let \(x_1^\star,\ldots,x_n^\star\in\R^3\) be ground-truth camera locations.  The observation graph is \(K_n\).  Its camera edges are partitioned as
\[
    E(K_n)=\Ezero\sqcup\Eb,
    \qquad
    \Delta_E\defeq\max_{i\in[n]}\deg_{\Eb}(i).
\]
For \(ij\in\Ezero\), the observed direction is exact,
\[
    d_{ij}=d_{ij}^\star
    \defeq
    \frac{x_i^\star-x_j^\star}{\norm{x_i^\star-x_j^\star}},
\]
and for \(ij\in\Eb\), the direction \(d_{ij}\) is arbitrary.  If an adversarial direction equals the true direction, that edge may be moved from \(\Eb\) to \(\Ezero\); this can only decrease \(\Delta_E\).  A graph triangle is called clean if all three of its camera edges are clean.  Let \(\Tzero\) be the retained clean triangles and let \(\Tb\) be the retained non-clean triangles.  The algorithm is not given this partition.

\begin{theorem}[Complete-graph max-degree theorem for proximal annealed Cauchy IRLS]\label{thm:main}
There are absolute constants \(\cstar>0\), \(n_0<\infty\), and \(\kappa>0\) such that the following holds for every \(n\ge n_0\).  One admissible explicit choice in the proof is \(\cstar=10^{-5}\) and \(\kappa=1/28\).  Assume that the clean locations are generic and that every clean graph triangle is nondegenerate.  Assume only the complete-graph max bad-degree condition
\begin{equation}\label{eq:main-degree-condition}
    \Delta_E\le \cstar n .
\end{equation}
Run the ideal noiseless TriP log-scale synchronization stage on \(K_n\) with all nondegenerate graph triangles as variables and all shared-edge constraints between them, with normalized triangle reliability scores satisfying
\[
    0\le \pi_t\le1,
    \qquad
    \pi_t=1\quad\text{for every exact clean retained triangle},
\]
and shared-edge base weights \(\wbase_{tu}=\sqrt{\pi_t\pi_u}\).  This includes the unweighted case \(\pi_t\equiv1\).  Initialize by weighted least squares, choose the first scale \(\sig_0\) as in \eqref{eq:sigma0-choice}, and at every scale run the proximal Cauchy-IRLS iteration with
\[
    \mustar=\kappa n,
    \qquad
    \sig_{k+1}=\frac12\sig_k.
\]
Each proximal WLS subproblem is solved exactly.  At each annealing scale, warm-start the exact proximal IRLS map from the previous-scale output and define the stage output to be any cluster point of that exact inner orbit.  Then the clean triangle log-scales converge uniformly to the true clean log-scales up to the unavoidable global additive gauge:
\begin{equation}\label{eq:main-conv}
    \min_{\alpha\in\R}
    \max_{t\in\Tzero}
    \abs{z_t^{(k)}-z_t^\star-\alpha}
    \longrightarrow0 .
\end{equation}
More quantitatively,
\begin{equation}\label{eq:main-rate}
    \min_{\alpha\in\R}
    \max_{t\in\Tzero}
    \abs{z_t^{(k)}-z_t^\star-\alpha}
    \le C_\star \Delta_E\,\frac{\sig_k}{n}
    \le C_\star \cstar \sig_k,
\end{equation}
where \(C_\star\) is a universal constant.  In particular, \(\sig_k\to0\) implies exact clean-scale recovery.
\end{theorem}

\begin{remark}[What the theorem does and does not assume]
The only corruption hypothesis is \(\Delta_E\le\cstar n\).  The theorem does not assume a clean-basin initialization, does not assume a lower bound on \(q_{\min}\), does not assume a clean Green-response condition, and does not assume that bad triangles receive tiny weights.  Those facts are either proved from \(\Delta_E\le\cstar n\) or are unnecessary because of the proximal damping.  There is also no edge filtering: every retained shared-edge row is used with its given positive or zero base weight.
\end{remark}

\begin{remark}[Why the fixed-point wording is used]
At a fixed point of the proximal IRLS map, the proximal term vanishes from the first-order condition, so the fixed point is a stationary point of the original weighted Cauchy objective at that scale.  The exact-solver assumption in the theorem refers to the quadratic WLS subproblem \eqref{eq:prox-wls}: no iterative linear-solver error is analyzed.  The theorem is an exact-recovery statement for the ideal annealed fixed-point solver.
\end{remark}

\begin{remark}[Fixed-point output convention used in the proof]\label{rem:fixed-point-output-convention}
The phrase ``take a fixed point'' is used in the algorithmic warm-started sense.  At scale \(\sig_k\), start the exact proximal map from the previous-scale output \(z^{k,0}\), form the exact inner orbit \(z^{k,m+1}=T_{\sig_k}(z^{k,m})\), and take a cluster-point fixed point of this orbit.  Proposition~\ref{prop:warm-started-basin-fixed-point} proves that such a cluster point exists, lies in the clean basin, and is a fixed point of the exact proximal IRLS map.  Thus the theorem does not require a global statement that every fixed point of the nonconvex Cauchy objective lies in the clean basin.
\end{remark}

\subsection{Log-scale model and Cauchy identities}

For a retained triangle \(t=\{i,j,k\}\), TriP computes positive side-ratio entries \(h_{t,ij},h_{t,jk},h_{t,ki}\).  If \(t\in\Tzero\), then there is a positive scale \(s_t^\star\) such that
\begin{equation}\label{eq:clean-local-scale}
    s_t^\star h_{t,e}=\ell_e^\star
    \qquad
    \text{for every camera edge }e\subset t,
\end{equation}
where \(\ell_{ij}^\star=\norm{x_i^\star-x_j^\star}\).  Put
\[
    z_t^\star\defeq\log s_t^\star,
    \qquad t\in\Tzero .
\]
If retained triangles \(t,u\) share a camera edge \(e\), the log-scale equation is
\begin{equation}\label{eq:shared-log-equation}
    z_u-z_t=g_{tu},
    \qquad
    g_{tu}=\log\frac{h_{t,e}}{h_{u,e}},
\end{equation}
with the sign determined by the row orientation.  If \(t,u\in\Tzero\), then \eqref{eq:clean-local-scale} gives the exact clean equation
\begin{equation}\label{eq:clean-log-equation}
    z_u^\star-z_t^\star=g_{tu}.
\end{equation}

Let \(\Call\) be the retained shared-edge row set.  Let \(\Czero\subset\Call\) be the rows with two clean endpoints and \(\Cb=\Call\setminus\Czero\).  Write \(z=(z_0,z_b)\in\R^{\Tzero}\times\R^{\Tb}\).  For \(c\in\Czero\),
\[
    r_c(z_0)=b_c^\top z_0-g_c,
    \qquad
    g_c=b_c^\top z_0^\star,
\]
where \(b_c\) is an oriented incidence row on clean triangle variables.  For \(f\in\Cb\), split the full incidence row as
\begin{equation}\label{eq:bad-row-split}
    r_f(z)=a_f^\top z_0+c_f^\top z_b-g_f.
\end{equation}
A bad row has either one clean endpoint and one non-clean endpoint, or two non-clean endpoints.  Hence
\begin{equation}\label{eq:row-norms}
    \norm{a_f}_1\le1,
    \qquad
    \norm{c_f}_1\le2.
\end{equation}
The bound \(\norm{a_f}_1\le1\) is used below; it holds because every row with two clean endpoints belongs to \(\Czero\).

The clean scales are identifiable only up to a global additive constant.  Define
\begin{equation}\label{eq:error}
    \err(z_0)
    \defeq
    \min_{\alpha\in\R}
    \norm{z_0-z_0^\star-\alpha\one}_{\infty}.
\end{equation}
Throughout the proof we use the fixed numerical basin radius
\begin{equation}\label{eq:a-star}
    \astar\defeq\frac1{20}.
\end{equation}
The scale-\(\sig\) clean basin is
\begin{equation}\label{eq:basin}
    \basin_\sig
    \defeq
    \{z=(z_0,z_b):\err(z_0)\le \astar\sig\}.
\end{equation}
The theorem does not assume that an iterate is in \(\basin_\sig\); the proof establishes it by induction.

The Cauchy loss, score, and IRLS weight are
\begin{equation}\label{eq:cauchy}
    \rho_\sig(r)=\frac{\sig^2}{2}\log\left(1+\frac{r^2}{\sig^2}\right),
    \quad
    \psi_\sig(r)=\rho_\sig'(r)=\frac{r}{1+(r/\sig)^2},
    \quad
    \omega_\sig(r)=\frac{1}{1+(r/\sig)^2}.
\end{equation}
We repeatedly use
\begin{equation}\label{eq:cauchy-basic-bounds}
    \abs{\psi_\sig(r)}\le\frac{\sig}{2},
    \qquad
    \omega_\sig(r)\abs r\le\frac{\sig}{2}.
\end{equation}
Indeed, writing \(r=\sig x\), both inequalities reduce to
\(\abs{x}/(1+x^2)\le1/2\).  If \(\abs r\le2\astar\sig=\frac{\sig}{10}\), then
\begin{equation}\label{eq:mstar}
    \omega_\sig(r)
    \ge
    \frac{1}{1+(1/10)^2}
    =\frac{100}{101}
    \defeq \mstar .
\end{equation}

\paragraph{Numerical constants used in the proof.}
The remaining numerical constants are chosen as
\begin{equation}\label{eq:numerical-constants}
    \Gconst=7,
    \qquad
    \kappa=\frac1{28},
    \qquad
    \jstar=\frac12+\astar=\frac{11}{20},
    \qquad
    \cstar=10^{-5}.
\end{equation}
The value \(\astar=1/20\) is intentionally round.  It keeps the clean Cauchy conductances in the interval \([100/101,1]\), gives a short weighted-Green perturbation argument, and still leaves enough margin for the proximal bad-network estimate.  The constants are not optimized; they are explicit conservative values for the deterministic max-degree proof.

\subsection{Combinatorial consequences of the max-degree bound}

The following estimates are deterministic and use only the complete graph and the max degree of corrupted camera edges.

\begin{lemma}[Clean triangles pass exact triangle prefiltering]\label{lem:clean-pass}
Every clean nondegenerate graph triangle satisfies the exact positive closure relation and hence is retained by the ideal noiseless triangle prefilter.
\end{lemma}

\begin{proof}
For a clean triangle \(t=\{i,j,k\}\), oriented as \(ij,jk,ki\),
\[
    \ell_{ij}^\star d_{ij}^\star
    +\ell_{jk}^\star d_{jk}^\star
    +\ell_{ki}^\star d_{ki}^\star
    =(x_i^\star-x_j^\star)+(x_j^\star-x_k^\star)+(x_k^\star-x_i^\star)=0.
\]
All coefficients are positive.  Nondegeneracy excludes collinearity.  Thus the exact positive-closure test retains the triangle.
\end{proof}

\begin{lemma}[Clean witnesses from max degree]\label{lem:qmin}
For every clean camera edge \(ij\in\Ezero\), the number of clean witnesses
\[
    q_{ij}
    \defeq
    \#\{k\in[n]\setminus\{i,j\}:ik\in\Ezero,\ jk\in\Ezero\}
\]
satisfies
\begin{equation}\label{eq:qmin-bound}
    q_{ij}\ge n-2-\deg_{\Eb}(i)-\deg_{\Eb}(j)
    \ge n-2-2\Delta_E.
\end{equation}
Consequently, if \(\Delta_E\le\cstar n\) and \(\cstar\le10^{-3}\), then \(q_{ij}\ge n/2\) for all sufficiently large \(n\).
\end{lemma}

\begin{proof}
A vertex \(k\notin\{i,j\}\) fails to be a clean witness for the clean edge \(ij\) only if \(ik\in\Eb\) or \(jk\in\Eb\).  There are at most \(\deg_{\Eb}(i)+\deg_{\Eb}(j)\) such vertices.  This proves \eqref{eq:qmin-bound}; the final claim follows immediately for large \(n\).
\end{proof}

\begin{lemma}[Clean-boundary bad degree]\label{lem:boundary-degree}
Let
\[
    \dtri
    \defeq
    \max_{t\in\Tzero}\#\{f\in\Cb:\text{the row }f\text{ is incident to }t\}.
\]
Then
\begin{equation}\label{eq:dtri-bound}
    \dtri\le 6\Delta_E.
\end{equation}
\end{lemma}

\begin{proof}
Fix a clean triangle \(t=\{i,j,k\}\).  A bad shared-edge row incident to \(t\) must share one of the clean camera edges \(ij,ik,jk\) with \(t\).  Consider the edge \(ij\).  A neighboring triangle along this edge has the form \(\{i,j,\ell\}\).  Since \(ij\) is clean, the neighboring triangle is non-clean only if \(i\ell\in\Eb\) or \(j\ell\in\Eb\).  Hence there are at most
\[
    \deg_{\Eb}(i)+\deg_{\Eb}(j)\le2\Delta_E
\]
non-clean neighbors of \(t\) along \(ij\).  The same argument applies to \(ik\) and \(jk\).  Summing over the three clean edges of \(t\) gives \(6\Delta_E\).
\end{proof}

\begin{lemma}[Clean triangle graph has bounded diameter]\label{lem:diam}
If \(\Delta_E\le10^{-3}n\) and \(n\) is sufficiently large, then the clean triangle-overlap graph \((\Tzero,\Czero)\) is connected and has graph diameter at most six.
\end{lemma}

\begin{proof}
Let \(t=\{a,b,c\}\) and \(u=\{p,q,r\}\) be two clean triangles.  Since each camera vertex is incident to at most \(\Delta_E\) corrupted edges, fewer than \(6\Delta_E+6\) vertices are either in \(t\cup u\) or fail to be cleanly connected to all vertices in \(t\cup u\).  For large \(n\), there is a vertex \(s_1\notin t\cup u\) cleanly connected to all vertices in \(t\cup u\).  Similarly, one can choose \(s_2\) cleanly connected to every vertex in \(t\cup u\cup\{s_1\}\), and then \(s_3\) cleanly connected to every vertex in \(t\cup u\cup\{s_1,s_2\}\).  This is possible because at each step we exclude only \(O(\Delta_E)+O(1)<n\) vertices.

Then the sequence
\[
\{a,b,c\},
\{a,b,s_1\},
\{a,s_1,s_2\},
\{s_1,s_2,s_3\},
\{p,s_2,s_3\},
\{p,q,s_3\},
\{p,q,r\}
\]
consists entirely of clean triangles.  Consecutive triangles share two camera vertices, hence share one clean camera edge, so they are adjacent in \((\Tzero,\Czero)\).  This gives a clean path of length six between arbitrary clean triangles.
\end{proof}

\begin{lemma}[Maximum triangle-graph degree]\label{lem:triangle-degree}
Every graph triangle in \(K_n\) is adjacent, in the full shared-edge triangle graph, to at most
\begin{equation}\label{eq:Dmax}
    D_n\defeq 3(n-3)
\end{equation}
other graph triangles.  Therefore every retained triangle has retained shared-edge degree at most \(D_n\).
\end{lemma}

\begin{proof}
A triangle has three camera edges.  Along a fixed camera edge, there are \(n-3\) other graph triangles containing that edge.  Summing over the three edges gives \(3(n-3)\).  Retaining fewer triangles can only decrease this degree.
\end{proof}

\subsection{Clean Green response from the max-degree bound}

This section proves the clean inverse-response estimate needed later.  The theorem statement does not assume it; it follows from the complete graph and the max-degree bound.  We keep the proof explicit because this is the point where \(q_{\min}\) disappears from the theorem statement.  We also record the passive-network/Kron-reduction estimate used twice below: first for the clean Johnson graph after deleting non-clean triangles, and later for the bad Schur-complement operator created by a frozen WLS step.

\begin{lemma}[Passive boundary response under Kron reduction]\label{lem:passive-boundary-response}
Let \(B\) be a set of boundary vertices and \(I\) a set of internal vertices.  Consider a finite weighted electrical network, with no direct edges between two boundary vertices, with nonnegative edge conductances at most one, boundary variables \(x\in\R^B\), internal variables \(u\in\R^I\), and optional nonnegative internal grounding conductances \(\gamma_i\ge0\).  Its energy has the form
\begin{equation}\label{eq:passive-energy}
    \mathcal Q(x,u)
    =\frac12\sum_{ab} \theta_{ab}(v_a-v_b)^2
    +\frac12\sum_{i\in I}\gamma_i u_i^2,
\end{equation}
where \(v_b=x_b\) for \(b\in B\), \(v_i=u_i\) for \(i\in I\), and \(0\le\theta_{ab}\le1\).  Assume the internal minimizer is unique for each boundary potential \(x\), and define the boundary Schur response \(S\) by
\begin{equation}\label{eq:schur-response-definition}
    \frac12 x^\top Sx
    =\min_{u\in\R^I}\mathcal Q(x,u).
\end{equation}
For a boundary vertex \(b\), let
\[
    d_B(b)\defeq\sum_{i\in I}\theta_{bi},
    \qquad
    \text{the total original conductance from }b\text{ into the eliminated network}.
\]
Then the following statements hold.
\begin{enumerate}[label=(\roman*),leftmargin=2em]
    \item \(S\) is symmetric positive semidefinite.
    \item \(S\) is an \(M\)-matrix on the boundary: its off-diagonal entries are nonpositive and its row sums are nonnegative.
    \item For every \(x\in\R^B\),
    \begin{equation}\label{eq:passive-infty}
        \abs{(Sx)_b}\le 2d_B(b)\norm{x}_\infty,
        \qquad
        \norm{Sx}_\infty\le 2\Bigl(\max_{b\in B}d_B(b)\Bigr)\norm{x}_\infty.
    \end{equation}
    \item If there is no internal grounding, then \(S\one=0\).  If there is grounding, then \(S\one\ge0\) entrywise.
\end{enumerate}
\end{lemma}

\begin{proof}
Write the quadratic form in block matrix form as
\[
    \mathcal Q(x,u)
    =\frac12
    \begin{pmatrix}x\\u\end{pmatrix}^{\!\top}
    \begin{pmatrix}
        L_{BB} & L_{BI}\\
        L_{IB} & K_{II}
    \end{pmatrix}
    \begin{pmatrix}x\\u\end{pmatrix},
\]
where \(K_{II}\) is the internal Laplacian block plus the internal grounding diagonal.  The assumed uniqueness of the internal minimizer means that \(K_{II}\) is positive definite on the relevant internal variables.  The minimizing internal potential is
\[
    u(x)=-K_{II}^{-1}L_{IB}x,
\]
and the Schur response is
\begin{equation}\label{eq:S-schur-general}
    S=L_{BB}-L_{BI}K_{II}^{-1}L_{IB}.
\end{equation}
This proves symmetry.  Positive semidefiniteness follows from the variational identity
\[
    x^\top Sx=2\min_u \mathcal Q(x,u)\ge0.
\]

Next, \(K_{II}\) is a symmetric positive definite \(M\)-matrix: its off-diagonal entries are nonpositive and its inverse is entrywise nonnegative.  Since there are no direct boundary-boundary edges, the off-diagonal boundary entries of \(S\) come from the term \(-L_{BI}K_{II}^{-1}L_{IB}\).  Hence \(S_{bb'}\le0\) for \(b\ne b'\).  The row-sum statement is also a maximum-principle fact.  If all boundary potentials are equal to one, then the internal minimizer satisfies \(0\le u_i\le1\) when grounding is present, and \(u_i=1\) without grounding.  Therefore the boundary current \((S\one)_b=\sum_{i}\theta_{bi}(1-u_i)\) is nonnegative, and it is zero in the ungrounded case.

It remains to prove the quantitative passivity estimate.  Fix \(x\) and let \(u=u(x)\) be the minimizing internal potential.  By the discrete maximum principle applied to the internally grounded network, every internal value lies in the interval between \(-\norm{x}_\infty\), \(\norm{x}_\infty\), and the ground value \(0\).  Hence
\[
    \norm{u}_\infty\le \norm{x}_\infty.
\]
The boundary current at \(b\) is contributed only by original edges from \(b\) to internal vertices:
\[
    (Sx)_b=\sum_{i\in I}\theta_{bi}(x_b-u_i).
\]
Therefore
\[
    \abs{(Sx)_b}
    \le \sum_{i\in I}\theta_{bi}\abs{x_b-u_i}
    \le 2\norm{x}_\infty\sum_{i\in I}\theta_{bi}
    =2d_B(b)\norm{x}_\infty.
\]
Taking the maximum over \(b\) proves \eqref{eq:passive-infty}.
\end{proof}

\begin{lemma}[Johnson triangle Green function in quotient norm]\label{lem:johnson-green}
Let \(J(n,3)\) be the graph whose vertices are triples in \(\binom{[n]}3\), with two triples adjacent when they share two camera vertices.  Let \(L_J\) be its unweighted Laplacian.  Then for every vector \(y\) on \(\binom{[n]}3\), not necessarily zero-sum,
\begin{equation}\label{eq:johnson-green}
    \norm{L_J^\dagger y}_{\infty/\one}
    \le \frac{4}{n-3}\norm{y}_{\infty/\one},
    \qquad n\ge6.
\end{equation}
Here \(\norm{x}_{\infty/\one}=\min_{\alpha\in\R}\norm{x-\alpha\one}_\infty\).  In particular, the same bound holds with \(\norm{y}_\infty\) on the right-hand side.
\end{lemma}

\begin{proof}
The graph \(J(n,3)\) is distance-regular.  Its Laplacian eigenvalues are
\[
    0,
    n,
    2(n-1),
    3(n-2).
\]
Fix a vertex \(x\).  By vertex transitivity, the Green kernel \(G(x,y)=(L_J^\dagger)_{xy}\) depends only on \(r=\abs{x\cap y}\in\{0,1,2,3\}\).  Write the four values as \(G_0,G_1,G_2,G_3\).  The level sizes are
\[
    N_3=1,
    \quad
    N_2=3(n-3),
    \quad
    N_1=\frac{3(n-3)(n-4)}2,
    \quad
    N_0=\binom{n-3}{3}.
\]
Solving the four radial Green equations together with the zero-row-sum condition
\(\sum_y G(x,y)=0\) gives
\begin{align*}
G_0&=-\frac{11n^2-26n+12}{n^2(n-1)^2(n-2)^2},\\
G_1&=\frac{2n^3-39n^2+82n-36}{3n^2(n-1)^2(n-2)^2},\\
G_2&=\frac{n^4+5n^3-88n^2+172n-72}{6n^2(n-1)^2(n-2)^2},\\
G_3&=\frac{(n-3)(2n^4+n^3+16n^2-52n+24)}{6n^2(n-1)^2(n-2)^2}.
\end{align*}
A direct substitution shows that, for \(n\ge6\),
\[
    \sum_y\abs{G(x,y)}
    =N_0\abs{G_0}+N_1\abs{G_1}+N_2\abs{G_2}+N_3\abs{G_3}
    \le \frac4{n-3}.
\]

A second direct counting calculation gives a useful row-difference bound.  If two Johnson vertices \(x,x'\) have \(s=|x\cap x'|\), then
\begin{equation}\label{eq:johnson-row-diff}
\sum_y\abs{G(x,y)-G(x',y)}
=
\begin{cases}
\displaystyle \frac{11n^2-62n+156}{3n(n-1)(n-2)}, & s=0,\\[2mm]
\displaystyle \frac{9n^2-40n+56}{3n(n-1)(n-2)}, & s=1,\\[2mm]
\displaystyle \frac{2}{n}, & s=2,\\[1mm]
0, & s=3.
\end{cases}
\end{equation}
Indeed, if \(|x\cap x'|=s\), partition \([n]\) into \(x\cap x'\), \(x\setminus x'\), \(x'\setminus x\), and the complement, count a test triple by how many vertices it selects from these four sets, and substitute the four displayed Green values.  The three nonzero cases in \eqref{eq:johnson-row-diff} are all at most \(4/n\) for \(n\ge6\).
Since every row of \(L_J^\dagger\) sums to zero, for any scalar \(\beta\),
\[
    (L_J^\dagger y)(x)=\sum_y G(x,y)(y_y-\beta).
\]
Choose \(\beta\) so that \(\norm{y-\beta\one}_\infty=\norm{y}_{\infty/\one}\).  Then
\[
    \abs{(L_J^\dagger y)(x)}
    \le \sum_y\abs{G(x,y)}\norm{y}_{\infty/\one}
    \le\frac4{n-3}\norm{y}_{\infty/\one}.
\]
Taking the maximum over \(x\) and allowing an arbitrary additive gauge on the output proves the lemma.
\end{proof}

\begin{lemma}[Unweighted clean Green response under bounded corrupted degree]\label{lem:unweighted-clean-green}
There are universal constants \(\delta_J>0\) and \(G_J<\infty\) such that the following holds.  If \(\Delta_E\le\delta_J n\), then the unweighted clean triangle Laplacian \(L_0=B_0^\top B_0\) satisfies, for every vector \(y\) on \(\Tzero\),
\begin{equation}\label{eq:unweighted-clean-green}
    \norm{L_0^\dagger y}_{\infty/\one}
    \le \frac{G_J}{n}\norm{y}_{\infty/\one}.
\end{equation}
One may take \(\delta_J=1/384\) and \(G_J=32/7\), after increasing \(n_0\) if necessary.
\end{lemma}

\begin{proof}
Let \(\Omega=\binom{[n]}3\) be the full triangle set and let \(U=\Tzero\).  Let \(R=\Omega\setminus U\).  Components of \(R\) that have no Johnson-graph edge to \(U\) carry no boundary current and can be discarded.  On the remaining part, the Dirichlet block is invertible.  Partition the full Johnson Laplacian according to \(\Omega=U\sqcup R\).  The Dirichlet Schur complement
\[
    L_{\mathrm{harm}}
    =L_{UU}-L_{UR}L_{RR}^{-1}L_{RU}
\]
is the Laplacian on \(U\) obtained by harmonically extending potentials to the removed vertices.  Equivalently, for a force on \(U\), the solution of the harmonic Schur-complement equation is the restriction to \(U\) of the full Johnson solution with zero force on \(R\), with the irrelevant global mean removed by the pseudoinverse.

We first record a quotient-norm bound for the harmonic inverse.  Let \(y\) be any vector on \(U\).  Since \(L_{\mathrm{harm}}^\dagger\one_U=0\), replace \(y\) by its zero-sum representative on \(U\),
\[
    y_\perp=y-\frac{1}{|U|}\one_U^\top y\,\one_U .
\]
If \(q=\norm{y}_{\infty/\one}\), then the range of \(y\) is \(2q\), and hence \(\norm{y_\perp}_\infty\le2q\).  Extend \(y_\perp\) by zero to \(\Omega\).  This extended force has total sum zero, so the harmonic solution is the restriction of the full Johnson solution with this force and zero force on \(R\).  Therefore, for any two clean triangles \(u,v\in U\),
\[
    (L_{\mathrm{harm}}^\dagger y)_u-(L_{\mathrm{harm}}^\dagger y)_v
    =\sum_{w\in U}\bigl(\mathcal G_J(u,w)-\mathcal G_J(v,w)\bigr)(y_\perp)_w,
\]
where \(\mathcal G_J\) is the full Johnson Green kernel.  By the row-difference estimate \eqref{eq:johnson-row-diff},
\[
    \abs{(L_{\mathrm{harm}}^\dagger y)_u-(L_{\mathrm{harm}}^\dagger y)_v}
    \le \frac4n\norm{y_\perp}_\infty
    \le \frac8n\norm{y}_{\infty/\one}.
\]
The quotient norm of a vector is half its range, so
\begin{equation}\label{eq:harm-green-new}
    \norm{L_{\mathrm{harm}}^\dagger y}_{\infty/\one}
    \le \frac4n\norm{y}_{\infty/\one}.
\end{equation}
This establishes the harmonic Green response directly in the same quotient norm that will be used later.

The harmonic Laplacian decomposes as
\[
    L_{\mathrm{harm}}=L_0+M_\partial,
\]
where \(M_\partial\succeq0\) is the boundary Dirichlet-to-Neumann Laplacian induced by the removed triangles.  For a clean triangle \(t=\{i,j,k\}\), a removed Johnson neighbor sharing the clean edge \(ij\) has the form \(\{i,j,\ell\}\), and is removed only if \(i\ell\in\Eb\) or \(j\ell\in\Eb\).  Thus the number of removed neighbors of \(t\) is at most
\[
    2(\deg_{\Eb}(i)+\deg_{\Eb}(j)+\deg_{\Eb}(k))\le6\Delta_E.
\]
Applying Lemma~\ref{lem:passive-boundary-response} to the deleted Johnson network gives
\begin{equation}\label{eq:Mboundary-op-new}
    \norm{M_\partial x}_\infty
    \le 12\Delta_E\norm{x}_\infty.
\end{equation}
Since \(M_\partial\one=0\), the same estimate holds on quotient representatives: choose \(x\) with \(\norm{x}_\infty\) arbitrarily close to \(\norm{x}_{\infty/\one}\).
Combining \eqref{eq:harm-green-new} and \eqref{eq:Mboundary-op-new} gives
\[
    \norm{L_{\mathrm{harm}}^\dagger M_\partial x}_{\infty/\one}
    \le \frac4n\,12\Delta_E\norm{x}_{\infty/\one}.
\]
If \(\Delta_E\le\delta_J n\) and \(\delta_J\le1/384\), this operator norm is at most \(1/8\).  On the quotient space,
\[
    L_0=L_{\mathrm{harm}}-M_\partial
    =L_{\mathrm{harm}}(I-L_{\mathrm{harm}}^\dagger M_\partial).
\]
The Neumann series is therefore valid on the quotient, and for every force vector \(y\),
\[
    \norm{L_0^\dagger y}_{\infty/\one}
    \le \frac{1}{1-1/8}\norm{L_{\mathrm{harm}}^\dagger y}_{\infty/\one}
    \le \frac{32}{7n}\norm{y}_{\infty/\one}.
\]
Thus \(G_J=32/7\) is admissible for all sufficiently large \(n\).
\end{proof}

\begin{lemma}[Weighted clean Green response in the small clean basin]\label{lem:weighted-clean-green}
Let \(L_\Theta=B_0^\top\Theta B_0\) be a clean triangle Laplacian whose clean conductances obey
\[
    \theta_c\in[\mstar,1],
    \qquad
    \mstar=100/101.
\]
There are universal constants \(\delta_G>0\) and \(\Gconst<\infty\) such that if \(\Delta_E\le\delta_G n\), then for every vector \(y\) on \(\Tzero\),
\begin{equation}\label{eq:weighted-clean-green}
    \norm{L_\Theta^\dagger y}_{\infty/\one}
    \le \frac{\Gconst}{n}\norm{y}_{\infty/\one}.
\end{equation}
With the constants above, one may take \(\Gconst=7\).
\end{lemma}

\begin{proof}
Write
\[
    L_\Theta=L_0-H,
    \qquad
    H=B_0^\top(I-\Theta)B_0.
\]
Here \(H\) is a nonnegative weighted Laplacian whose edge conductances are at most
\[
    1-\mstar=\frac1{101}.
\]
The clean triangle degree is at most \(3(n-3)\).  Therefore, for every vector \(x\),
\begin{equation}\label{eq:H-op-new}
    \norm{Hx}_\infty
    \le 2\cdot\frac1{101}\cdot3(n-3)\norm{x}_\infty
    \le \frac{6n}{101}\norm{x}_\infty.
\end{equation}
Since \(H\one=0\), choose the representative of the quotient class of \(x\) with nearly minimal \(\ell_\infty\) norm.  Using Lemma~\ref{lem:unweighted-clean-green} in quotient-force form gives
\[
    \norm{L_0^\dagger Hx}_{\infty/\one}
    \le \frac{32}{7n}\norm{Hx}_{\infty/\one}
    \le \frac{32}{7n}\frac{6n}{101}\norm{x}_{\infty/\one}
    =\frac{192}{707}\norm{x}_{\infty/\one}.
\]
The coefficient is strictly smaller than one.  Hence, on the quotient modulo constants,
\[
    L_\Theta=L_0(I-L_0^\dagger H)
\]
is invertible and
\[
    \norm{L_\Theta^\dagger y}_{\infty/\one}
    \le \frac{1}{1-192/707}\norm{L_0^\dagger y}_{\infty/\one}
    \le \frac{707}{515}\frac{32}{7n}\norm{y}_{\infty/\one}
    =\frac{22624}{3605n}\norm{y}_{\infty/\one}
    <\frac{7}{n}\norm{y}_{\infty/\one}.
\]
We state the lemma with \(\Gconst=7\).  The strict inequality \(22624/3605<7\) is uniform for all \(n\) covered by the preceding estimates.
\end{proof}
\subsection{Weighted least squares enters the first clean basin}

Let the weighted LS initializer be any minimizer, in a fixed gauge, of
\begin{equation}\label{eq:ls-init}
    z^{\LS}\in\argmin_z
    \frac12\sum_{f\in\Call}\wbase_f r_f(z)^2.
\end{equation}
Define its computable residual radius
\begin{equation}\label{eq:RLS}
    R_{\LS}
    \defeq
    \max_{f\in\Call:\ \wbase_f>0}\abs{r_f(z^{\LS})}.
\end{equation}
If the maximum is over an empty set, the synchronization problem is vacuous; otherwise choose a positive initial Cauchy scale satisfying
\begin{equation}\label{eq:sigma0-choice}
    \sig_0>0,
    \qquad
    \sig_0\ge \frac{4}{\astar}R_{\LS}=80 R_{\LS}.
\end{equation}

\begin{lemma}[LS initialization is in the first clean basin]\label{lem:ls-basin}
With \(\sig_0\) chosen by \eqref{eq:sigma0-choice},
\begin{equation}\label{eq:ls-in-basin}
    \err(z_0^{\LS})\le \astar\sig_0.
\end{equation}
\end{lemma}

\begin{proof}
Every clean retained triangle has \(\pi_t=1\), so every clean-clean row has base weight \(1\).  Therefore every clean-clean residual at \(z^{\LS}\) is bounded by \(R_{\LS}\).  Put
\[
    e_t=z_t^{\LS}-z_t^\star,
    \qquad t\in\Tzero.
\]
For a clean-clean row joining \(t\) and \(u\), exact clean consistency gives
\[
    e_u-e_t=r_{tu}(z^{\LS}),
    \qquad
    \abs{e_u-e_t}\le R_{\LS}.
\]
The clean triangle graph has diameter at most six by Lemma~\ref{lem:diam}.  Hence, for arbitrary clean triangles \(t,u\),
\[
    \abs{e_t-e_u}\le6R_{\LS}.
\]
Subtracting the midpoint of the range of the numbers \(e_t\) gives
\[
    \err(z_0^{\LS})\le3R_{\LS}\le\astar\sig_0,
\]
because \(\sig_0\ge4R_{\LS}/\astar\).  This proves the first-basin initialization.
\end{proof}

\subsection{Proximal-damped Cauchy IRLS at a fixed scale}

At scale \(\sig\), the weighted Cauchy objective is
\begin{equation}\label{eq:full-cauchy-obj}
    F_\sig(z)
    =
    \sum_{f\in\Call}\wbase_f\rho_\sig(r_f(z)).
\end{equation}
Given an iterate \(z^m\), define the frozen IRLS conductance of row \(f\) by
\begin{equation}\label{eq:frozen-conductance}
    \theta_f^m
    \defeq
    \wbase_f\omega_\sig(r_f(z^m)).
\end{equation}
The proximal-damped WLS target is the unique minimizer
\begin{equation}\label{eq:prox-wls}
    \widehat z^{m+1}
    =
    \argmin_z
    \left\{
    \frac12\sum_{f\in\Call}\theta_f^m r_f(z)^2
    +\frac{\mustar}{2}\norm{z-z^m}_2^2
    \right\},
    \qquad
    \mustar=\kappa n.
\end{equation}
The proximal term makes the quadratic subproblem strictly convex in all variables, including gauges and nuisance-only components.  It is useful to name the exact target map
\begin{equation}\label{eq:T-map-def}
    T_\sig(x)
    \defeq
    \argmin_z
    \left\{
    \frac12\sum_{f\in\Call}
        \wbase_f\omega_\sig(r_f(x))r_f(z)^2
    +\frac{\mustar}{2}\norm{z-x}_2^2
    \right\}.
\end{equation}
Thus the undamped exact inner IRLS step is
\[
    z^{m+1}=T_\sig(z^m),
\]
and \(\widehat z^{m+1}=T_\sig(z^m)\) in the notation of \eqref{eq:prox-wls}.  The optional relaxation below is only an additional damping of this target; the fixed-point extraction used in the theorem is for the exact map \(T_\sig\).

One may optionally apply an additional convex relaxation
\begin{equation}\label{eq:optional-relaxation}
    z^{m+1}=(1-\lambda_m)z^m+\lambda_m\widehat z^{m+1},
    \qquad 0<\lambda_m\le1.
\end{equation}
The proof first controls the target \(\widehat z^{m+1}\); then \eqref{eq:optional-relaxation} follows by convexity of the clean error.

\begin{lemma}[Cauchy MM majorization with proximal damping]\label{lem:mm}
For fixed \(z^m\), the function minimized in \eqref{eq:prox-wls} majorizes \(F_\sig(z)\) plus the nonnegative proximal term, and it agrees with \(F_\sig\) at \(z=z^m\) after adding constants independent of \(z\).  Consequently, the proximal IRLS step is a descent step for \(F_\sig\).
\end{lemma}

\begin{proof}
Write \(u=r^2\) and
\[
    \rho_\sig(r)=\frac{\sig^2}{2}\log\left(1+\frac{u}{\sig^2}\right)=\phi(u).
\]
The scalar function \(\phi\) is concave and
\[
    \phi'(u)=\frac12\frac{1}{1+u/\sig^2}=\frac12\omega_\sig(r).
\]
The tangent inequality for a concave function gives
\[
    \rho_\sig(r)
    \le
    \rho_\sig(r^m)+\frac12\omega_\sig(r^m)(r^2-(r^m)^2).
\]
Multiplying by \(\wbase_f\) and summing over rows gives the IRLS quadratic majorizer.  Adding the nonnegative proximal term preserves majorization, and the proximal term vanishes at \(z=z^m\).  Thus
\[
    F_\sig(\widehat z^{m+1})
    +\frac{\mustar}{2}\norm{\widehat z^{m+1}-z^m}_2^2
    \le
    F_\sig(z^m).
\]
In particular, \(F_\sig\) decreases.
\end{proof}

\begin{lemma}[Convexity of the clean error]\label{lem:error-convex}
The map \(z_0\mapsto\err(z_0)\) is convex.  Hence if \(\err(u_0)\le R\) and \(\err(v_0)\le R\), then
\[
    \err((1-\lambda)u_0+\lambda v_0)\le R,
    \qquad 0\le\lambda\le1.
\]
\end{lemma}

\begin{proof}
The set \(\{z_0^\star+\alpha\one:\alpha\in\R\}\) is an affine subspace.  The quantity \(\err(z_0)\) is the distance from \(z_0\) to this affine subspace in the \(\ell_\infty\) norm.  Distance to a convex set in a norm is convex.
\end{proof}

\begin{lemma}[Clean conductances in the basin]\label{lem:clean-conductances}
If \(z^m\in\basin_\sig\), then every clean-clean row \(c\in\Czero\) satisfies
\[
    \abs{r_c(z^m)}\le2\astar\sig=\frac{\sig}{10},
\]
and hence
\begin{equation}\label{eq:clean-theta-lower}
    \theta_c^m
    =\omega_\sig(r_c(z^m))
    \in[\mstar,1],
    \qquad
    \mstar=\frac{100}{101}.
\end{equation}
\end{lemma}

\begin{proof}
Choose \(\alpha_m\) so that
\[
    \norm{z_0^m-z_0^\star-\alpha_m\one}_\infty\le\astar\sig.
\]
For a clean row joining clean triangles \(t,u\),
\[
    r_c(z^m)
    =(z_u^m-z_u^\star-\alpha_m)-(z_t^m-z_t^\star-\alpha_m),
\]
so \(\abs{r_c(z^m)}\le2\astar\sig\).  The lower bound \eqref{eq:clean-theta-lower} is exactly \eqref{eq:mstar}.  The base weight of a clean-clean row is one because both endpoint triangles are clean and \(\pi_t=\pi_u=1\).
\end{proof}

\subsection{The Schur-complement estimate with proximal damping}

This section proves the key one-step estimate.  The estimate keeps the full Schur-complement contribution of the non-clean triangle variables, including internal non-clean shared-edge rows.

Assume \(z^m\in\basin_\sig\).  Choose \(\alpha_m\) such that
\[
    e^m
    \defeq
    z_0^m-z_0^\star-\alpha_m\one
    \quad\text{satisfies}\quad
    \norm{e^m}_\infty\le\astar\sig.
\]
Define the comparison point
\begin{equation}\label{eq:comparison-point}
    \bar z^m=(z_0^\star+\alpha_m\one,z_b^m).
\end{equation}
For bad rows define the lagged current source
\begin{equation}\label{eq:j-def}
    j_f^m\defeq\theta_f^m r_f(\bar z^m),
    \qquad f\in\Cb.
\end{equation}

\begin{lemma}[Lagged bad current is rowwise bounded]\label{lem:j-bound}
If \(z^m\in\basin_\sig\), then for every bad row \(f\in\Cb\),
\begin{equation}\label{eq:j-row-bound}
    \abs{j_f^m}\le \jstar\sig,
    \qquad
    \jstar\defeq\frac12+\astar=\frac{11}{20}.
\end{equation}
\end{lemma}

\begin{proof}
If \(f\) has no clean endpoint, then \(a_f=0\), so \(r_f(\bar z^m)=r_f(z^m)\).  Hence
\[
    \abs{j_f^m}
    =\wbase_f\omega_\sig(r_f(z^m))\abs{r_f(z^m)}
    \le\frac\sig2
\]
by \eqref{eq:cauchy-basic-bounds} and \(\wbase_f\le1\).

If \(f\) has one clean endpoint, then \(\norm{a_f}_1=1\), and
\[
    r_f(\bar z^m)-r_f(z^m)
    =a_f^\top(z_0^\star+\alpha_m\one-z_0^m).
\]
Thus
\[
    \abs{r_f(\bar z^m)-r_f(z^m)}\le\astar\sig.
\]
Consequently,
\begin{align*}
    \abs{j_f^m}
    &\le
    \wbase_f\omega_\sig(r_f(z^m))\abs{r_f(z^m)}
    +\wbase_f\omega_\sig(r_f(z^m))\astar\sig  \\
    &\le \left(\frac12+\astar\right)\sig
    =\jstar\sig .
\end{align*}
This proves the claim.  Keeping the sharper constant in \eqref{eq:j-row-bound} improves the explicit numerical value of the admissible max-degree constant.
\end{proof}

Let \(A\) and \(C\) be the clean and nuisance blocks of the bad incidence rows, so the bad residual vector is
\[
    r_b(z)=Az_0+Cz_b-g_b.
\]
Let \(\Theta_b^m=\diag(\theta_f^m:f\in\Cb)\).  Let \(L_m\) be the clean-clean Laplacian with conductances \(\theta_c^m\), \(c\in\Czero\).  The target displacement from the comparison point is
\[
    \delta=(\delta_0,
    \delta_b)
    \defeq
    \widehat z^{m+1}-\bar z^m.
\]
Expanding the proximal WLS objective around \(\bar z^m\), and using that clean residuals vanish at \(\bar z^m\), gives the block normal equations
\begin{align}
    (L_m+A^\top\Theta_b^mA+\mustar I)\delta_0
    +A^\top\Theta_b^mC\delta_b
    &=\mustar e^m-A^\top j^m,
    \label{eq:block-normal-clean}\\
    C^\top\Theta_b^mA\delta_0
    +(C^\top\Theta_b^mC+\mustar I)\delta_b
    &=-C^\top j^m.
    \label{eq:block-normal-bad}
\end{align}
The proximal term contributes \(\mustar(\delta_0-e^m)\) in the clean block and \(\mustar\delta_b\) in the nuisance block because \(\bar z_b^m=z_b^m\).

\begin{lemma}[Stable nuisance inverse]\label{lem:nuisance-inverse}
Let
\[
    K_b=C^\top\Theta_b^mC+\mustar I.
\]
Then for every vector \(v\),
\begin{equation}\label{eq:nuisance-inverse-bound}
    \norm{K_b^{-1}v}_\infty\le\frac1{\mustar}\norm{v}_\infty.
\end{equation}
Consequently,
\begin{equation}\label{eq:u-bound}
    \norm{K_b^{-1}C^\top j^m}_\infty
    \le
    \frac{D_n}{\mustar}\jstar\sig
    \le
    \frac{3n}{\mustar}\jstar\sig.
\end{equation}
\end{lemma}

\begin{proof}
The matrix \(C^\top\Theta_b^mC\) is a graph Laplacian on nuisance nodes plus possible nonnegative diagonal terms coming from bad rows with exactly one nuisance endpoint.  Thus \(K_b\) is a strictly diagonally dominant symmetric \(M\)-matrix.  Let \(u=K_b^{-1}v\).  If \(u_i=\max_j u_j>0\), then the Laplacian/grounding part satisfies \((C^\top\Theta_b^mC u)_i\ge0\), because every difference \(u_i-u_j\) at a maximum is nonnegative and every grounding term is also nonnegative.  Hence
\[
    \mustar u_i\le v_i\le\norm{v}_\infty.
\]
Similarly, if \(u_i=\min_j u_j<0\), then \((C^\top\Theta_b^mC u)_i\le0\), so
\[
    \mustar u_i\ge -\norm{v}_\infty.
\]
Therefore \(\norm{u}_\infty\le\norm{v}_\infty/\mustar\), proving \eqref{eq:nuisance-inverse-bound}.

For \eqref{eq:u-bound}, each nuisance triangle is incident to at most \(D_n\) retained rows by Lemma~\ref{lem:triangle-degree}, each coefficient of \(C\) has magnitude at most one at a given endpoint, and \(\abs{j_f^m}\le\jstar\sig\) by Lemma~\ref{lem:j-bound}.  Hence \(\norm{C^\top j^m}_\infty\le D_n\jstar\sig\), and \eqref{eq:u-bound} follows.
\end{proof}

\begin{lemma}[Bad Schur-complement forcing is local]\label{lem:bad-force-local}
After eliminating \(\delta_b\) from \eqref{eq:block-normal-clean}--\eqref{eq:block-normal-bad}, the clean equation can be written as
\begin{equation}\label{eq:clean-schur-equation}
    (L_m+\mustar I+S_m)\delta_0
    =\mustar e^m-p_m,
\end{equation}
where
\begin{equation}\label{eq:S-m-def}
    S_m=A^\top\Theta_b^mA
    -A^\top\Theta_b^mC
    (C^\top\Theta_b^mC+\mustar I)^{-1}
    C^\top\Theta_b^mA
\end{equation}
and
\begin{equation}\label{eq:p-m-def}
    p_m=A^\top j^m-A^\top\Theta_b^mC
    (C^\top\Theta_b^mC+\mustar I)^{-1}C^\top j^m .
\end{equation}
The operator \(S_m\) is a symmetric positive semidefinite boundary response matrix.  More precisely, it is an \(M\)-matrix on clean boundary variables: its off-diagonal entries are nonpositive and its row sums are nonnegative.  Moreover,
\begin{equation}\label{eq:p-local-bound}
    \norm{p_m}_\infty
    \le
    \dtri\jstar\sig\left(1+\frac{D_n}{\mustar}\right)
    \le
    6\Delta_E\jstar\sig\left(1+\frac{3n}{\mustar}\right),
\end{equation}
and
\begin{equation}\label{eq:S-op-bound}
    \norm{S_m x}_\infty\le 2\dtri\norm{x}_\infty
    \le12\Delta_E\norm{x}_\infty.
\end{equation}
\end{lemma}

\begin{proof}
Solving \eqref{eq:block-normal-bad} gives
\[
    \delta_b
    =-K_b^{-1}C^\top\Theta_b^mA\delta_0-K_b^{-1}C^\top j^m,
    \qquad
    K_b=C^\top\Theta_b^mC+\mustar I.
\]
Substituting this expression into \eqref{eq:block-normal-clean} gives \eqref{eq:clean-schur-equation} with \(S_m\) and \(p_m\) defined by \eqref{eq:S-m-def} and \eqref{eq:p-m-def}.  This algebra displays both the direct bad-boundary forcing and the indirect forcing transmitted through the eliminated non-clean variables.

We now identify \(S_m\) as a passive boundary response.  Build a weighted network whose boundary vertices are the clean triangles \(\Tzero\) that are touched by bad rows, and whose internal vertices are the non-clean triangles \(\Tb\).  Every bad row with one clean and one non-clean endpoint is an edge between the boundary and the internal set, with conductance \(\theta_f^m\le1\).  Every bad row with two non-clean endpoints is an internal edge, again with conductance at most one.  The proximal term \(\mustar I\) on \(\delta_b\) is exactly an internal grounding conductance \(\mustar\) at every non-clean triangle.  Eliminating the internal potentials of this network gives precisely \(S_m\).  Therefore Lemma~\ref{lem:passive-boundary-response} proves that \(S_m\) is symmetric positive semidefinite, has nonpositive off-diagonal entries, has nonnegative row sums, and satisfies the row response estimate.  Since a clean triangle is incident to at most \(\dtri\) bad boundary rows and each such row has conductance at most one, the original boundary conductance at any clean triangle is at most \(\dtri\).  Hence \eqref{eq:S-op-bound} follows from \eqref{eq:passive-infty}; Lemma~\ref{lem:boundary-degree} gives the final \(12\Delta_E\) bound.

It remains to bound the forcing \(p_m\).  The direct term is local:
\[
    \norm{A^\top j^m}_\infty\le\dtri\jstar\sig,
\]
because each clean triangle is incident to at most \(\dtri\) bad rows and \(\abs{j_f^m}\le\jstar\sig\) by Lemma~\ref{lem:j-bound}.  For the Schur-complement leakage term, Lemma~\ref{lem:nuisance-inverse} gives
\[
    \norm{K_b^{-1}C^\top j^m}_\infty
    \le\frac{D_n}{\mustar}\jstar\sig.
\]
Multiplying by \(A^\top\Theta_b^mC\) is again a boundary-local operation: at each clean triangle there are at most \(\dtri\) incident bad boundary rows, every conductance is at most one, and each row has only one nuisance endpoint coefficient of magnitude one.  Consequently
\[
    \norm{A^\top\Theta_b^mC K_b^{-1}C^\top j^m}_\infty
    \le
    \dtri\frac{D_n}{\mustar}\jstar\sig.
\]
Combining the direct and leakage estimates proves \eqref{eq:p-local-bound}; Lemmas~\ref{lem:boundary-degree} and \ref{lem:triangle-degree} give the second displayed inequality.
\end{proof}

\begin{lemma}[Maximum principle for the proximal clean equation]\label{lem:max-principle-delta}
Let \(\delta_0\) solve \eqref{eq:clean-schur-equation}.  Then
\begin{equation}\label{eq:delta-max}
    \norm{\delta_0}_\infty
    \le
    \norm{e^m}_\infty+\frac{\norm{p_m}_\infty}{\mustar}.
\end{equation}
\end{lemma}

\begin{proof}
By Lemma~\ref{lem:bad-force-local}, \(S_m\) is a symmetric \(M\)-matrix with nonpositive off-diagonal entries and nonnegative row sums.  The clean Laplacian \(L_m\) has the same sign pattern and zero row sums.  Therefore \(L_m+S_m\) has nonpositive off-diagonal entries and nonnegative row sums.  Let \(i\) be an index at which \(\delta_{0,i}\) is maximal.  If \(\delta_{0,i}>0\), then
\[
    ((L_m+S_m)\delta_0)_i
    =\sum_{j\ne i}(-M_{ij})(\delta_{0,i}-\delta_{0,j})
      +\Bigl(\sum_j M_{ij}\Bigr)\delta_{0,i}
    \ge0,
\]
where \(M=L_m+S_m\), since \(-M_{ij}\ge0\), \(\delta_{0,i}-\delta_{0,j}\ge0\), and the row sum is nonnegative.  Using \eqref{eq:clean-schur-equation},
\[
    \mustar\delta_{0,i}
    \le
    \mustar e_i^m-(p_m)_i
    \le
    \mustar\norm{e^m}_\infty+\norm{p_m}_\infty.
\]
The same argument at an index of negative minimum gives the corresponding lower bound.  Thus \eqref{eq:delta-max} holds.
\end{proof}

\begin{proposition}[One proximal WLS target remains in the clean basin]\label{prop:one-step-target}
There exists a universal \(\cstar>0\) such that, if \(\Delta_E\le\cstar n\), then the following holds.  Fix the damping level
\begin{equation}\label{eq:kappa-choice}
    \kappa\defeq\frac{1}{4\Gconst}=\frac1{28},
    \qquad
    \mustar=\kappa n.
\end{equation}
If \(z^m\in\basin_\sig\), then the exact proximal WLS target \(\widehat z^{m+1}\) satisfies
\begin{equation}\label{eq:target-in-basin}
    \err(\widehat z_0^{m+1})\le\astar\sig.
\end{equation}
\end{proposition}

\begin{proof}
Let
\[
    E_m=\norm{e^m}_\infty\le\astar\sig,
    \qquad
    P_m=\norm{p_m}_\infty,
    \qquad
    E_{m+1}^{\rm tar}=\norm{\delta_0}_{\infty/\one}
    =\err(\widehat z_0^{m+1}).
\]
By Lemma~\ref{lem:bad-force-local},
\begin{equation}\label{eq:P-bound-proof}
    P_m
    \le
    6\Delta_E\jstar\sig\left(1+\frac{3n}{\mustar}\right).
\end{equation}
Let \(R_G=\Gconst/n\).  Since \(z^m\in\basin_\sig\), Lemma~\ref{lem:clean-conductances} shows that the clean conductance matrix \(L_m\) satisfies the weighted clean Green estimate \eqref{eq:weighted-clean-green}.  From \eqref{eq:clean-schur-equation},
\begin{equation}\label{eq:Lm-delta-equation}
    L_m\delta_0
    =
    \mustar(e^m-\delta_0)-S_m\delta_0-p_m.
\end{equation}
The right-hand side is zero-sum because it equals \(L_m\delta_0\).  We now apply the Green estimate carefully on the quotient.  Since \(L_m^\dagger\) annihilates constants,
\[
    \norm{L_m^\dagger v}_{\infty/\one}
    \le R_G\norm{v}_{\infty/\one}
    \le R_G\norm{v}_\infty
\]
whenever the expression is defined modulo constants.  Therefore
\begin{align}
    E_{m+1}^{\rm tar}
    &\le
    R_G\mustar\norm{e^m-\delta_0}_{\infty/\one}
    +R_G\norm{S_m\delta_0+p_m}_\infty  \notag\\
    &\le
    R_G\mustar(E_m+E_{m+1}^{\rm tar})
    +R_G\norm{S_m\delta_0}_\infty
    +R_GP_m .
    \label{eq:quotient-recurrence-start}
\end{align}
Here we used the triangle inequality for the quotient norm:
\[
    \norm{e^m-\delta_0}_{\infty/\one}
    \le \norm{e^m}_{\infty/\one}+\norm{\delta_0}_{\infty/\one}
    \le E_m+E_{m+1}^{\rm tar}.
\]
This estimate is the point where the quotient norm is used for the clean potential variables.

Next, Lemma~\ref{lem:max-principle-delta} gives
\begin{equation}\label{eq:delta-infty-for-recurrence}
    \norm{\delta_0}_\infty\le E_m+P_m/\mustar.
\end{equation}
Using \eqref{eq:S-op-bound},
\begin{equation}\label{eq:S-delta-bound}
    \norm{S_m\delta_0}_\infty
    \le 2\dtri(E_m+P_m/\mustar).
\end{equation}
Insert \eqref{eq:S-delta-bound} into \eqref{eq:quotient-recurrence-start} and move the term
\(R_G\mustar E_{m+1}^{\rm tar}\) to the left.  Since
\[
    R_G\mustar=\frac{\Gconst}{n}\kappa n=\Gconst\kappa=\frac14,
\]
we obtain
\begin{equation}\label{eq:improved-one-step-recurrence}
    E_{m+1}^{\rm tar}
    \le
    \frac{R_G}{1-R_G\mustar}
    \left((\mustar+2\dtri)E_m+\bigl(1+2\dtri/\mustar\bigr)P_m\right).
\end{equation}
This form separates the quotient contraction of the clean component from the maximum-principle bound used for the passive bad Schur operator.

Write \(\varepsilon=\Delta_E/n\).  Using \(\dtri\le6\Delta_E\), \(E_m\le\astar\sig\), \(\mustar=\kappa n\), and \eqref{eq:P-bound-proof}, \eqref{eq:improved-one-step-recurrence} gives
\begin{align}
    \frac{E_{m+1}^{\rm tar}}{\sig}
    &\le
    \frac{1}{1-\Gconst\kappa}
    \Bigl[
       \Gconst(\kappa+12\varepsilon)\astar
       +6\Gconst\jstar\varepsilon
        \left(1+\frac{12\varepsilon}{\kappa}\right)
        \left(1+\frac3\kappa\right)
    \Bigr].
    \label{eq:epsilon-bound-full}
\end{align}
With the choice \(\kappa=1/(4\Gconst)\), the denominator is \(1-\Gconst\kappa=3/4\), and the first term at \(\varepsilon=0\) contributes only \(\astar/3\).  Thus there is a positive universal number \(\cstar\) such that, whenever \(\varepsilon\le\cstar\), the right-hand side of \eqref{eq:epsilon-bound-full} is at most \(\astar\).  Equivalently, one may take \(\cstar\) smaller than the positive root of
\begin{equation}\label{eq:cstar-optimized-root}
    12\Gconst\cstar\astar
    +6\Gconst\jstar\cstar
      \left(1+\frac{12\cstar}{\kappa}\right)
      \left(1+\frac3\kappa\right)
    =\frac{\astar}{2},
    \qquad
    \kappa=\frac1{4\Gconst},
\end{equation}
and also smaller than the constants needed in the clean Green lemmas.  The left-hand side of \eqref{eq:cstar-optimized-root} is increasing for \(\cstar>0\), so this root is well-defined.  With the explicit values \(\astar=1/20\), \(\jstar=11/20\), \(\Gconst=7\), and \(\kappa=1/28\), this root is larger than \(1.2\cdot10^{-5}\).  For instance, substituting \(\cstar=1.2\cdot10^{-5}\) into the left-hand side of \eqref{eq:cstar-optimized-root} gives
\[
    12\Gconst\cstar\astar
    +6\Gconst\jstar\cstar
      \left(1+\frac{12\cstar}{\kappa}\right)
      \left(1+\frac3\kappa\right)
    <0.0238
    <\frac1{40}=\frac{\astar}{2}.
\]
Therefore the positive root is strictly larger than \(1.2\cdot10^{-5}\), and \(\cstar=10^{-5}\) is a valid explicit choice.  This value is conservative because the proof requires a uniform \(\ell_\infty\) invariant basin for every exact WLS step and every adversarial bad subnetwork.

Under this choice, \eqref{eq:epsilon-bound-full} gives
\[
    E_{m+1}^{\rm tar}\le\astar\sig,
\]
which is exactly \eqref{eq:target-in-basin}.
\end{proof}

\begin{corollary}[Every relaxed proximal IRLS step remains in the basin]\label{cor:one-step}
Under the hypotheses of Proposition~\ref{prop:one-step-target}, if \(z^m\in\basin_\sig\), then the proximal target \(\widehat z^{m+1}\) and every relaxed update \eqref{eq:optional-relaxation} both lie in \(\basin_\sig\).
\end{corollary}

\begin{proof}
The target is in \(\basin_\sig\) by Proposition~\ref{prop:one-step-target}.  The current iterate is in \(\basin_\sig\) by assumption.  The relaxed update is a convex combination of the two clean blocks, so Lemma~\ref{lem:error-convex} gives the claim.
\end{proof}

\subsection{Warm-started fixed-point output at one scale}
\label{sec:warm-started-fixed-point}

The one-step result proves forward invariance of the clean basin for finite IRLS steps.  The stationary endpoint estimate in the next section applies to a fixed point in that basin.  The following argument shows that the warm-started exact proximal orbit has cluster points, that every cluster point is a fixed point of the exact proximal IRLS map, and that every such cluster point remains in the basin by closedness.  No assertion is made about fixed points outside the warm-started orbit.

Let \(D\) denote the full signed incidence matrix of all retained shared-edge rows, so that
\[
    r(z)=Dz-g.
\]
For a point \(x\), define the diagonal frozen-weight matrix
\[
    \Theta(x)
    =\diag\bigl(\wbase_f\omega_\sig(r_f(x)): f\in\Call\bigr).
\]
Then the exact target map \(T_\sig\) is equivalently characterized by the normal equation
\begin{equation}\label{eq:T-normal-equation}
    \bigl(D^\top\Theta(x)D+\mustar I\bigr)T_\sig(x)
    =D^\top\Theta(x)g+\mustar x.
\end{equation}
The matrix on the left has all eigenvalues at least \(\mustar>0\), so the target is unique for every \(x\).

\begin{lemma}[Continuity and fixed-point stationarity]\label{lem:T-continuity-stationarity}
For each fixed scale \(\sig>0\), the exact proximal IRLS map \(T_\sig\) is continuous.  Moreover,
\[
    T_\sig(z)=z
    \quad\Longleftrightarrow\quad
    \nabla F_\sig(z)=0,
\]
where the gradient is taken with respect to all retained triangle variables.
\end{lemma}

\begin{proof}
The entries of \(\Theta(x)\) are continuous functions of \(x\), because each residual \(r_f(x)\) is affine and \(\omega_\sig\) is continuous.  Since
\[
    D^\top\Theta(x)D+\mustar I\succeq \mustar I,
\]
its inverse exists for every \(x\) and depends continuously on \(x\).  Formula \eqref{eq:T-normal-equation} therefore proves continuity of \(T_\sig\).

If \(T_\sig(z)=z\), then substituting \(x=z\) in \eqref{eq:T-normal-equation} gives
\[
    D^\top\Theta(z)(Dz-g)=0.
\]
The \(f\)-th entry of \(\Theta(z)(Dz-g)\) is
\[
    \wbase_f\omega_\sig(r_f(z))r_f(z)
    =\wbase_f\psi_\sig(r_f(z)).
\]
Thus
\[
    D^\top\Theta(z)(Dz-g)
    =\sum_{f\in\Call}\wbase_f\psi_\sig(r_f(z))d_f
    =\nabla F_\sig(z),
\]
where \(d_f^\top\) is the signed incidence row of constraint \(f\).  Hence every fixed point is stationary for \(F_\sig\).

Conversely, if \(\nabla F_\sig(z)=0\), then
\[
    D^\top\Theta(z)(Dz-g)=0.
\]
Therefore \(z\) satisfies the normal equation of the strictly convex quadratic subproblem defining \(T_\sig(z)\).  By uniqueness of the minimizer, \(T_\sig(z)=z\).
\end{proof}

\begin{lemma}[Boundedness of exact proximal IRLS orbits]\label{lem:exact-orbit-bounded}
Fix \(\sig>0\) and an initial vector \(z^0\).  Let
\[
    z^{m+1}=T_\sig(z^m),\qquad m=0,1,2,\ldots .
\]
Then the sequence \(\{z^m\}_{m\ge0}\) is bounded in the finite-dimensional space of retained triangle variables.
\end{lemma}

\begin{proof}
Only finite-dimensional compactness is needed here; the constants in this proof may depend on the instance, the scale, and the initial point.  They are not part of the recovery estimate.

Let \(\Call_+=\{f\in\Call:\wbase_f>0\}\), and let \(\mathcal G_+\) be the graph on retained triangle variables with edge set \(\Call_+\).  Write its connected components as \(Q_1,\ldots,Q_M\).  Rows with \(\wbase_f=0\) never enter \(F_\sig\) or the frozen quadratic because their frozen conductance is identically zero.

First we show that the mean on every positive-weight component is preserved.  Let \(\one_Q\) be the indicator of a component \(Q\).  The normal equation for one exact step is
\begin{equation}\label{eq:orbit-normal-equation}
    D^\top\Theta(z^m)(Dz^{m+1}-g)+\mustar(z^{m+1}-z^m)=0.
\end{equation}
Taking the inner product with \(\one_Q\) gives
\[
    (D\one_Q)^\top\Theta(z^m)(Dz^{m+1}-g)
    +\mustar\langle z^{m+1}-z^m,\one_Q\rangle=0.
\]
For every row with positive frozen conductance, the two endpoints lie in the same component of \(\mathcal G_+\), hence that row annihilates \(\one_Q\).  For every row with zero frozen conductance, the row contributes nothing.  Therefore the first term is zero and
\[
    \langle z^{m+1}-z^m,\one_Q\rangle=0.
\]
Thus the component mean
\[
    \frac1{|Q|}\sum_{t\in Q}z_t^m
\]
is independent of \(m\).

Next, Lemma~\ref{lem:mm} gives the descent estimate
\begin{equation}\label{eq:descent-for-boundedness}
    F_\sig(z^{m+1})
    \le F_\sig(z^m)
    \le F_\sig(z^0).
\end{equation}
If \(\Call_+\) is empty, then \(\Theta(z^m)=0\) for every \(m\), the exact map is \(T_\sig(z^m)=z^m\), and the sequence is constant.  Otherwise, because \(\Call_+\) is finite and all its base weights are positive, the number
\[
    \wbase_{\min}\defeq\min_{f\in\Call_+}\wbase_f
\]
is positive.  Since the Cauchy loss \(\rho_\sig(r)\) is an increasing unbounded function of \(|r|\), \eqref{eq:descent-for-boundedness} implies a finite residual bound on all positive-weight rows.  For example, every \(f\in\Call_+\) satisfies
\[
    \rho_\sig(r_f(z^m))
    \le \frac{F_\sig(z^0)}{\wbase_f}
    \le \frac{F_\sig(z^0)}{\wbase_{\min}},
\]
and hence
\begin{equation}\label{eq:Rplus-bound}
    |r_f(z^m)|
    \le
    R_+
    \defeq
    \sig\sqrt{
        \exp\left(\frac{2F_\sig(z^0)}{\sig^2\wbase_{\min}}\right)-1
    }
    <\infty .
\end{equation}
Let
\[
    G_+\defeq \max_{f\in\Call_+}|g_f|,
\]
with \(G_+=0\) if \(\Call_+=\emptyset\).  If \(t,u\) lie in the same component \(Q\), choose a path in \(\mathcal G_+\) from \(t\) to \(u\).  Along an edge \(f\) of this path,
\[
    |z^m_{t'}-z^m_t|
    \le |r_f(z^m)|+|g_f|
    \le R_+ + G_+ .
\]
A simple path has length at most \(|Q|-1\), so all pairwise differences inside \(Q\) are bounded uniformly in \(m\).  Since the mean on \(Q\) is fixed, every coordinate in \(Q\) is bounded uniformly in \(m\).  There are finitely many components, so the full sequence \(\{z^m\}\) is bounded.
\end{proof}

\begin{proposition}[Warm-started basin fixed point at a scale]\label{prop:warm-started-basin-fixed-point}
Fix a scale \(\sig>0\).  Assume \(\Delta_E\le\cstar n\) and \(n\) is large enough for Proposition~\ref{prop:one-step-target}.  Let \(z^0\in\basin_\sig\), and run the exact warm-started proximal IRLS orbit
\[
    z^{m+1}=T_\sig(z^m).
\]
Then:
\begin{enumerate}[label=(\roman*)]
    \item every finite iterate lies in the clean basin: \(z^m\in\basin_\sig\) for all \(m\ge0\);
    \item the orbit has at least one cluster point;
    \item every cluster point \(z^\infty\) of the orbit satisfies \(z^\infty\in\basin_\sig\) and \(T_\sig(z^\infty)=z^\infty\).
\end{enumerate}
Consequently, the warm-started exact inner loop always provides at least one basin fixed point of the exact proximal IRLS map.
\end{proposition}

\begin{proof}
The finite-step basin statement is exactly the induction supplied by Proposition~\ref{prop:one-step-target} with \(\lambda_m=1\): if \(z^m\in\basin_\sig\), then \(z^{m+1}=T_\sig(z^m)=\widehat z^{m+1}\in\basin_\sig\).  Since \(z^0\in\basin_\sig\), induction gives \(z^m\in\basin_\sig\) for every finite \(m\).

By Lemma~\ref{lem:exact-orbit-bounded}, the exact orbit is bounded.  Hence it has at least one convergent subsequence, say
\[
    z^{m_\ell}\to z^\infty .
\]
Because \(\basin_\sig\) is closed and every \(z^{m_\ell}\) lies in \(\basin_\sig\), the limit also satisfies
\[
    z^\infty\in\basin_\sig .
\]
It remains to show that this cluster point is a fixed point.

The descent estimate from Lemma~\ref{lem:mm} gives, for the exact target step,
\[
    F_\sig(z^{m+1})
    \le
    F_\sig(z^m)-\frac{\mustar}{2}\norm{z^{m+1}-z^m}_2^2 .
\]
Summing from \(m=0\) to \(M\) and using \(F_\sig\ge0\) gives
\[
    \frac{\mustar}{2}\sum_{m=0}^{M}\norm{z^{m+1}-z^m}_2^2
    \le F_\sig(z^0)-F_\sig(z^{M+1})
    \le F_\sig(z^0).
\]
Therefore
\[
    \norm{z^{m+1}-z^m}_2\to0.
\]
Along the convergent subsequence, this implies
\[
    z^{m_\ell+1}\to z^\infty .
\]
But \(z^{m_\ell+1}=T_\sig(z^{m_\ell})\).  By continuity of \(T_\sig\), proved in Lemma~\ref{lem:T-continuity-stationarity},
\[
    T_\sig(z^\infty)
    =\lim_{\ell\to\infty}T_\sig(z^{m_\ell})
    =\lim_{\ell\to\infty}z^{m_\ell+1}
    =z^\infty .
\]
Thus every cluster point is a fixed point of the exact proximal IRLS map and, by the closed-basin argument above, lies in \(\basin_\sig\).
\end{proof}

\subsection{Stationary endpoint estimate at one Cauchy scale}

The one-step estimate proves basin invariance.  To anneal to a smaller scale, we need a sharper estimate at a fixed point.  This estimate is easier than the frozen-WLS estimate because at stationarity the nuisance score equation cancels internal bad-bad forces.

\begin{proposition}[Clean error of a basin stationary point]\label{prop:stationary}
Suppose \(z^\infty\in\basin_\sig\) is a fixed point of the proximal Cauchy-IRLS map, equivalently a stationary point of \(F_\sig\).  If \(\Delta_E\le\cstar n\) and \(n\) is sufficiently large, then
\begin{equation}\label{eq:stationary-bound}
    \err(z_0^\infty)
    \le
    \frac{3\Gconst\Delta_E}{n}\sig.
\end{equation}
In particular, after decreasing \(\cstar\) if necessary,
\begin{equation}\label{eq:stationary-next-basin}
    \err(z_0^\infty)
    \le
    \frac{\astar}{2}\sig
    =\astar\left(\frac\sig2\right).
\end{equation}
\end{proposition}

\begin{proof}
Choose \(\alpha\) so that
\[
    e=z_0^\infty-z_0^\star-\alpha\one,
    \qquad
    \norm{e}_\infty=\err(z_0^\infty)\le\astar\sig.
\]
For every clean row \(c\), \(r_c(z^\infty)=b_c^\top e\), so \(\abs{r_c(z^\infty)}\le2\astar\sig=\frac{\sig}{10}\).  Define
\[
    \eta_c=
    \begin{cases}
    \psi_\sig(r_c(z^\infty))/r_c(z^\infty),&r_c(z^\infty)\ne0,\\
    1,&r_c(z^\infty)=0.
    \end{cases}
\]
Then \(\eta_c\in[\mstar,1]\), and the clean score equals \(L_\eta e\), where \(L_\eta\) is a clean Laplacian satisfying the weighted Green estimate.

For a bad row define
\[
    s_f=\wbase_f\psi_\sig(r_f(z^\infty)).
\]
By the Cauchy score cap and \(\wbase_f\le1\),
\begin{equation}\label{eq:bad-score-cap}
    \abs{s_f}\le\frac\sig2.
\end{equation}
Stationarity of \(F_\sig\) gives the block equations
\begin{equation}\label{eq:stationary-block-equations}
    L_\eta e+A^\top s=0,
    \qquad
    C^\top s=0.
\end{equation}
The second equation is the nuisance stationarity condition.  It is precisely what removes the internal bad-bad Schur-complement issue at a stationary point.  Since every full incidence row annihilates constants, \(A\one+C\one=0\), and therefore
\[
    \one^\top A^\top s=(A\one)^\top s=-(C\one)^\top s=-\one^\top C^\top s=0.
\]
Thus \(A^\top s\) is a zero-sum clean force.

For a fixed clean triangle, at most \(\dtri\le6\Delta_E\) bad rows touch it.  Hence
\[
    \norm{A^\top s}_\infty
    \le
    \dtri\frac\sig2
    \le
    3\Delta_E\sig.
\]
Applying the weighted clean Green estimate to
\(L_\eta e=-A^\top s\) gives
\[
    \err(z_0^\infty)
    =\norm{e}_{\infty/\one}
    \le
    \frac{\Gconst}{n}\,3\Delta_E\sig,
\]
which is \eqref{eq:stationary-bound}.  If \(\cstar\le \astar/(6\Gconst)\), then \(3\Gconst\Delta_E/n\le\astar/2\), proving \eqref{eq:stationary-next-basin}.
\end{proof}

\subsection{Induction over IRLS iterations and annealing scales}

We now prove the theorem.  The clean basin is an invariant proved inside the argument; it is not an assumption of the theorem.

\begin{proof}[Proof of Theorem~\ref{thm:main}]
Take \(\kappa=1/28\).  Take \(\cstar=10^{-5}\), or any smaller absolute constant if one wants to absorb additional finite-\(n\) margins into the statement.  By Lemmas~\ref{lem:unweighted-clean-green}--\ref{lem:weighted-clean-green}, Proposition~\ref{prop:one-step-target}, and Proposition~\ref{prop:stationary}, this value is below all smallness thresholds used in the argument for all sufficiently large \(n\).  Let \(n\ge n_0\) be large enough that those finite-\(n\) margins hold.

\paragraph{First scale.}
Set \(z^{0,0}=z^{\LS}\), and choose \(\sig_0\) by \eqref{eq:sigma0-choice}.  Lemma~\ref{lem:ls-basin} gives
\[
    z^{0,0}\in\basin_{\sig_0}.
\]
Thus the initial basin condition is obtained from the LS initializer and the initial scale choice.

\paragraph{Inner induction at a fixed scale.}
Fix \(k\) and suppose the warm start \(z^{k,0}\in\basin_{\sig_k}\).  If \(z^{k,m}\in\basin_{\sig_k}\), then Proposition~\ref{prop:one-step-target} gives
\[
    \widehat z^{k,m+1}\in\basin_{\sig_k}.
\]
If the implementation uses the optional relaxation \eqref{eq:optional-relaxation}, Corollary~\ref{cor:one-step} gives
\[
    z^{k,m+1}\in\basin_{\sig_k}.
\]
Without optional relaxation, this is the same statement with \(\lambda_m=1\).  By induction over \(m\), every inner proximal IRLS iterate at scale \(\sig_k\) remains in \(\basin_{\sig_k}\).

\paragraph{Endpoint estimate and annealing.}
For the ideal fixed-point solver in Theorem~\ref{thm:main}, use the exact inner map \(T_{\sig_k}\) warm-started at \(z^{k,0}\).  Proposition~\ref{prop:warm-started-basin-fixed-point} shows that this exact orbit has a cluster point \(z^k\) such that
\[
    z^k\in\basin_{\sig_k},
    \qquad
    T_{\sig_k}(z^k)=z^k .
\]
This is the fixed point taken at scale \(\sig_k\).  The point is not an arbitrary unrelated fixed point of the nonconvex Cauchy objective; it is a fixed point obtained as a cluster point of the warm-started basin-preserving orbit.  Since it lies in \(\basin_{\sig_k}\), Proposition~\ref{prop:stationary} applies and gives
\[
    \err(z_0^k)
    \le
    \astar\frac{\sig_k}{2}
    =\astar\sig_{k+1}.
\]
Therefore the warm start for the next scale, \(z^{k+1,0}=z^k\), lies in \(\basin_{\sig_{k+1}}\).  This proves the induction over all annealing scales.

\paragraph{Convergence.}
The stationary estimate \eqref{eq:stationary-bound} gives
\[
    \err(z_0^k)
    \le
    \frac{3\Gconst\Delta_E}{n}\sig_k.
\]
Since \(\sig_k=2^{-k}\sig_0\to0\), we obtain
\[
    \err(z_0^k)\to0.
\]
This is exactly \eqref{eq:main-conv}.  The quantitative estimate \eqref{eq:main-rate} holds with \(C_\star=3\Gconst\).
\end{proof}

\subsection{Role of the proximal term in the Schur-complement estimate}

The role of the proximal term is visible in the nuisance inverse.  Without proximal damping, eliminating nuisance variables produces the term
\[
    A^\top\Theta_b C(C^\top\Theta_b C)^\dagger C^\top j.
\]
The pseudoinverse can transmit currents through a large nearly singular non-clean subnetwork.  The clean-boundary degree \(\dtri\) controls how many non-clean rows touch each clean triangle, but it does not by itself control all internal non-clean paths.

With proximal damping, the term is instead
\[
    A^\top\Theta_b C(C^\top\Theta_b C+\mustar I)^{-1} C^\top j.
\]
The inverse is uniformly stable in \(\ell_\infty\):
\[
    \norm{(C^\top\Theta_b C+\mustar I)^{-1}v}_\infty
    \le\frac1{\mustar}\norm{v}_\infty.
\]
Moreover, \(\norm{C^\top j}_\infty\) is controlled by the ordinary triangle-graph degree \(D_n\le3n\), while the final multiplication by \(A^\top\Theta_b C\) is controlled by the clean-boundary bad degree \(\dtri\le6\Delta_E\).  Therefore
\[
    \norm{A^\top\Theta_b C(C^\top\Theta_b C+\mustar I)^{-1} C^\top j}_\infty
    \le
    \dtri\frac{D_n}{\mustar}\sig
    \lesssim
    \Delta_E\sig.
\]
This is the right scale: the clean Green response is \(O(1/n)\), so the resulting clean displacement is \(O((\Delta_E/n)\sig)\).  When \(\Delta_E\le\cstar n\) with \(\cstar\) sufficiently small, the WLS target cannot leave the clean basin.

\subsection{Consequences for clean lengths and locations}

The following corollaries record the corresponding consequences for clean shared-edge residuals, edge lengths, and camera locations.

\begin{corollary}[Clean shared-edge residuals vanish]\label{cor:clean-resid}
For every clean-clean shared-edge row \((t,u)\in\Czero\),
\[
    r_{tu}(z^k)\to0.
\]
\end{corollary}

\begin{proof}
For a clean row,
\[
    r_{tu}(z^k)
    =(z_u^k-z_u^\star)-(z_t^k-z_t^\star).
\]
Choose numbers \(\alpha_k\) such that \(\max_{t\in\Tzero}|z_t^k-z_t^\star-\alpha_k|\to0\).  Each parenthesized term differs from \(\alpha_k\) by a quantity tending uniformly to zero, so their difference tends to zero.
\end{proof}

\begin{corollary}[Clean edge-length proposals]\label{cor:lengths}
For every stage \(k\), define the clean length proposal
\[
    \widehat\ell_{t,e}^{(k)}=\exp(z_t^k)h_{t,e},
    \qquad t\in\Tzero,\ e\subset t.
\]
There exists a sequence of positive global scale factors \(\lambda_k\) such that
\begin{equation}\label{eq:clean-length-relative-convergence}
    \max_{\substack{t\in\Tzero\\ e\subset t}}
    \left|
       \frac{\widehat\ell_{t,e}^{(k)}}{\lambda_k\ell_e^\star}-1
    \right|
    \longrightarrow0.
\end{equation}
Thus all clean incident triangles give asymptotically consistent proposals for every clean edge, up to the unavoidable global scale at each stage.  Under any fixed additive gauge normalization of the clean log-scales, the factors \(\lambda_k\) may be fixed accordingly and the normalized proposals converge to the corresponding global scale multiple of the true lengths.
\end{corollary}

\begin{proof}
Choose \(\alpha_k\in\R\) such that
\[
    \varepsilon_k
    \defeq
    \max_{t\in\Tzero}|z_t^k-z_t^\star-\alpha_k|
    \longrightarrow0,
\]
and set \(\lambda_k=e^{\alpha_k}\).  Using \eqref{eq:clean-local-scale},
\[
    \frac{\widehat\ell_{t,e}^{(k)}}{\lambda_k\ell_e^\star}
    =
    \exp(z_t^k-z_t^\star-\alpha_k).
\]
The exponent is bounded in absolute value by \(\varepsilon_k\), uniformly over clean triangles and their edges.  Therefore
\[
    \max_{\substack{t\in\Tzero\\ e\subset t}}
    \left|
       \frac{\widehat\ell_{t,e}^{(k)}}{\lambda_k\ell_e^\star}-1
    \right|
    \le e^{\varepsilon_k}-1
    \longrightarrow0.
\]
\end{proof}

\begin{corollary}[Location recovery from recovered clean lengths]\label{cor:locations}
If the downstream displacement averaging stage is supplied with the gauge-normalized clean edge lengths from Corollary~\ref{cor:lengths}, then the clean camera locations are recovered in the limit up to global translation and global scale.
\end{corollary}

\begin{proof}
For every recovered clean edge \(ij\), Corollary~\ref{cor:lengths} gives a common positive scale \(\lambda_k\) such that the recovered displacement is
\[
    \lambda_k\ell_{ij}^\star d_{ij}^\star+o(\lambda_k\ell_{ij}^\star)
    =\lambda_k(x_i^\star-x_j^\star)+o(\lambda_k\ell_{ij}^\star),
\]
uniformly over clean edges after the same gauge normalization.  The clean camera graph is connected because its minimum degree is at least \(n-1-\Delta_E>n/2\) for large \(n\).  Therefore the clean displacement equations determine all camera locations up to one global translation.  The common factor \(\lambda_k\) is the unavoidable global scale ambiguity, and the normalized displacement errors vanish as \(k\to\infty\).
\end{proof}

\makeatletter
\let\Call\TriPproof@saved@Call
\makeatother

\section{Extension to General Redescending Losses}
\label{app:trip-general-redescending-losses}

\paragraph{Scope and relation to the Cauchy theorem.}
The preceding section proves a complete warm-started fixed-point theorem for
Cauchy IRLS.  The same framework extends to the redescending Cauchy, Welsch,
Tukey, and truncated-least-squares objectives.
The extension keeps exactly the same complete-graph camera model, the same
clean/non-clean triangle partition used only in the analysis, the same full
all-triangle shared-edge objective, the same base weights, the same proximal
WLS subproblems, and the same max-degree combinatorics.  In particular, the
algorithm is not given the clean partition, and every retained shared-edge row
is used.

There is one technical distinction that is necessary for mathematical
correctness.  The Cauchy loss is coercive in each residual, whereas Welsch,
Tukey, and truncated least squares are bounded.  For a bounded loss, descent of
the objective alone does not imply that the nuisance triangle variables remain
in a compact set, so the Cauchy compactness argument does not apply directly.
The proof below avoids any coercivity or compactness assumption.  At each fixed
scale it proves that the selected MM stationarity residual tends to zero, even
if some nuisance variables have no convergent subsequence.  A quantitative
clean-block estimate then permits a finite, observable stopping rule before the
next annealing scale is entered.  Thus all four redescending losses are covered
without an unjustified bounded-orbit assumption.

\subsection{Comparison of the four redescending losses}
\label{sec:genred-comparison}

All losses in Table~\ref{tab:genred-comparison} are normalized to have unit
quadratic curvature at the origin.  Put
\[
    x=\frac{r}{\sig},
    \qquad
    a=\frac1{20}.
\]
The column \(m_a\) is the lower bound for the IRLS weight on the clean window
\(|r|\le 2a\sig=\sig/10\).  The score cap is
\(|\omega_\sig(r)r|\le K\sig\).  The rounded clean Green constant is denoted by
\(\mathsf G_a\), and the proximal parameter is
\(\mustar=n/(4\mathsf G_a)\).  The column \(c_{\rm MM}\) gives an explicit,
non-optimized max-degree fraction for the complete warm-started proximal-MM
theorem proved below.  The column \(\tau_{\min}\) is the corresponding lower
bound on a geometric annealing ratio
\(\sig_{k+1}=\tau\sig_k\); in every row \(\tau=1/2\) is admissible.  The final
column gives the larger threshold available from the stationary-point estimate
alone when \(\tau=1/2\); that column assumes that a clean-basin stationary point
is supplied at every scale and therefore is not the unconditional proximal-MM
constant.

\begin{table}[t]
\centering
\scriptsize
\setlength{\tabcolsep}{3.2pt}
\renewcommand{\arraystretch}{1.35}
\resizebox{\textwidth}{!}{\begin{tabular}{lcccccccc}
\toprule
Loss
& normalized loss \(\rho_\sig(r)\)
& IRLS weight \(\omega_\sig(r)\)
& \(m_a\) at \(a=1/20\)
& score cap \(K\)
& rounded \(\mathsf G_a\)
& explicit \(c_{\rm MM}\)
& \(\tau_{\min}\) for \(c_{\rm MM}\)
& stationary-only \(c\) for \(\tau=1/2\)
\\
\midrule
Cauchy
& \(\frac{\sig^2}{2}\log(1+x^2)\)
& \((1+x^2)^{-1}\)
& \(100/101\)
& \(1/2\)
& \(7\)
& \(1.0\!\times\!10^{-5}\)
& \(0.0042\)
& \(1.19\!\times\!10^{-3}\)
\\
Welsch
& \(\frac{\sig^2}{2}(1-e^{-x^2})\)
& \(e^{-x^2}\)
& \(e^{-1/100}\)
& \(1/\sqrt{2e}\)
& \(7\)
& \(1.2\!\times\!10^{-5}\)
& \(0.00433\)
& \(1.38\!\times\!10^{-3}\)
\\
Tukey
& \(\frac{\sig^2}{6}[1-(1-x^2)^3]\) for \(|x|\le1\), constant otherwise
& \((1-x^2)_+^2\)
& \((99/100)^2\)
& \(16/(25\sqrt5)\)
& \(11\)
& \(8.0\!\times\!10^{-6}\)
& \(0.00303\)
& \(1.32\!\times\!10^{-3}\)
\\
Truncated LS
& \(\frac12\min\{r^2,\sig^2\}\)
& \(\mathbf 1_{\{|x|<1\}}\), with any value in \([0,1]\) at \(|x|=1\)
& \(1\)
& \(1\)
& \(5\)
& \(1.0\!\times\!10^{-5}\)
& \(0.0060\)
& \(8.33\!\times\!10^{-4}\)
\\
\bottomrule
\end{tabular}}
\caption{Comparison of the normalized redescending losses covered by
Theorem~\ref{thm:genred-main}.  The displayed constants are rigorous sufficient
values for the proof in this section, not optimized phase-transition constants.
The stationary-only column uses
\(c<a\tau/(6\mathsf G_aK)\) with \(a=1/20\) and \(\tau=1/2\).}
\label{tab:genred-comparison}
\end{table}

\subsection{A redescending MM-admissible loss class}
\label{sec:genred-admissible-class}

For each \(\sig>0\), write an even loss as
\begin{equation}\label{eq:genred-q-representation}
    \rho_\sig(r)=q_\sig(r^2),
    \qquad
    q_\sig:[0,\infty)\longrightarrow[0,\infty).
\end{equation}
For a concave function \(q\), define its superdifferential at \(u\ge0\) by
\begin{equation}\label{eq:genred-superdiff}
    \partial^+q(u)
    \defeq
    \left\{\gamma\in\R:
    q(v)\le q(u)+\gamma(v-u)
    \text{ for every }v\ge0\right\}.
\end{equation}
At a differentiability point this set contains only \(q'(u)\).

\begin{definition}[Redescending MM-admissible scale family]
\label{def:genred-admissible}
Fix a clean-basin radius \(a>0\).  A family
\(\{\rho_\sig\}_{\sig>0}\) is called \emph{redescending MM-admissible at
radius \(a\)} if, for every \(\sig>0\), the representation
\eqref{eq:genred-q-representation} and an even weight selector
\(\omega_\sig:\R\to[0,1]\) can be chosen so that the following conditions
hold with constants \(m_a\in(0,1]\) and \(K<\infty\) independent of
\(\sig\).
\begin{enumerate}[label=(R\arabic*),leftmargin=3em]
    \item The function \(q_\sig\) is continuous, nondecreasing, concave, and
    satisfies \(q_\sig(0)=0\).

    \item For every \(r\in\R\), the selected weight is a valid tangent slope:
    \begin{equation}\label{eq:genred-selector-supergradient}
        \frac12\omega_\sig(r)
        \in
        \partial^+q_\sig(r^2).
    \end{equation}
    Equivalently, for every \(v\ge0\),
    \begin{equation}\label{eq:genred-tangent-majorizer}
        q_\sig(v)
        \le
        q_\sig(r^2)
        +\frac12\omega_\sig(r)(v-r^2).
    \end{equation}

    \item The clean-window weights have a uniform positive floor:
    \begin{equation}\label{eq:genred-clean-floor}
        \omega_\sig(r)\ge m_a
        \qquad
        \text{whenever }|r|\le2a\sig.
    \end{equation}

    \item The selected score is globally bounded at the annealing scale:
    \begin{equation}\label{eq:genred-score-cap}
        |\omega_\sig(r)r|
        \le K\sig
        \qquad
        \text{for every }r\in\R.
    \end{equation}

    \item The score is redescending:
    \begin{equation}\label{eq:genred-redescending}
        \omega_\sig(r)r\longrightarrow0
        \qquad
        \text{as }|r|\longrightarrow\infty.
    \end{equation}
\end{enumerate}
At a kink, different admissible weight selections are permitted.  Every claim
below holds for every sequence of selections satisfying (R2)--(R4).
\end{definition}

\begin{remark}[Which assumptions drive the proof]
Condition (R2) gives the global quadratic MM majorizer.  Conditions (R3) and
(R4) are the two quantitative properties used in the recovery argument:
(R3) preserves the clean restoring Laplacian, while (R4) caps every adversarial
row force by \(O(\sig)\).  The redescending limit (R5) describes the intended
loss class and is useful algorithmically, but the deterministic error estimate
uses only the stronger uniform cap (R4).  No coercivity assumption is made.
\end{remark}

\begin{proposition}[The four losses are redescending MM-admissible]
\label{prop:genred-four-losses}
The normalized Cauchy, Welsch, Tukey, and truncated-least-squares families in
Table~\ref{tab:genred-comparison} satisfy
Definition~\ref{def:genred-admissible}.  Their clean-window floors and score
caps are
\begin{align*}
    m_a^{\rm Cau}&=\frac{1}{1+4a^2},
    &K_{\rm Cau}&=\frac12,\\
    m_a^{\rm Wel}&=e^{-4a^2},
    &K_{\rm Wel}&=\frac{1}{\sqrt{2e}},\\
    m_a^{\rm Tuk}&=(1-4a^2)^2,
    &K_{\rm Tuk}&=\frac{16}{25\sqrt5},
    \qquad 0<a<\frac12,\\
    m_a^{\rm TLS}&=1,
    &K_{\rm TLS}&=1,
    \qquad 0<a<\frac12.
\end{align*}
\end{proposition}

\begin{proof}
We verify the four families separately.

\paragraph{Cauchy.}
For
\[
    \rho_\sig^{\rm Cau}(r)
    =\frac{\sig^2}{2}\log\left(1+\frac{r^2}{\sig^2}\right),
\]
we have
\[
    q_\sig^{\rm Cau}(u)
    =\frac{\sig^2}{2}\log\left(1+\frac{u}{\sig^2}\right),
    \qquad
    \omega_\sig^{\rm Cau}(r)
    =\frac{1}{1+(r/\sig)^2}.
\]
The derivative of \(q_\sig^{\rm Cau}\) is positive and decreasing, so the
function is nondecreasing and concave, and (R2) holds with equality at the
tangent point.  If \(|r|\le2a\sig\), then
\[
    \omega_\sig^{\rm Cau}(r)
    \ge\frac{1}{1+4a^2}.
\]
Writing \(x=|r|/\sig\),
\[
    \frac{|\omega_\sig^{\rm Cau}(r)r|}{\sig}
    =\frac{x}{1+x^2}
    \le\frac12,
\]
because the maximum occurs at \(x=1\).  The same expression tends to zero as
\(x\to\infty\), proving (R5).

\paragraph{Welsch.}
For
\[
    \rho_\sig^{\rm Wel}(r)
    =\frac{\sig^2}{2}\left(1-e^{-r^2/\sig^2}\right),
\]
we have
\[
    q_\sig^{\rm Wel}(u)
    =\frac{\sig^2}{2}\left(1-e^{-u/\sig^2}\right),
    \qquad
    \omega_\sig^{\rm Wel}(r)=e^{-r^2/\sig^2}.
\]
Again \(q_\sig^{\rm Wel}\) is nondecreasing and concave because its derivative
is positive and decreasing.  On the clean window,
\[
    \omega_\sig^{\rm Wel}(r)\ge e^{-4a^2}.
\]
For \(x=|r|/\sig\), the normalized score is \(xe^{-x^2}\).  Its derivative is
\(e^{-x^2}(1-2x^2)\), so its maximum is attained at
\(x=1/\sqrt2\) and equals \(1/\sqrt{2e}\).  It tends to zero at infinity.

\paragraph{Tukey.}
Use the normalized Tukey biweight loss
\begin{equation}\label{eq:genred-tukey-loss}
    \rho_\sig^{\rm Tuk}(r)
    =
    \begin{cases}
    \displaystyle
    \frac{\sig^2}{6}
    \left[1-\left(1-\frac{r^2}{\sig^2}\right)^3\right],
    & |r|\le\sig,\\[2mm]
    \displaystyle \frac{\sig^2}{6},
    & |r|>\sig.
    \end{cases}
\end{equation}
Then
\[
    \omega_\sig^{\rm Tuk}(r)
    =\left(1-\frac{r^2}{\sig^2}\right)_+^2.
\]
As a function of \(u=r^2\), the derivative is
\(\frac12(1-u/\sig^2)^2\) for \(0<u<\sig^2\) and zero for
\(u>\sig^2\).  This derivative is nonnegative, continuous, and nonincreasing,
so the corresponding \(q_\sig^{\rm Tuk}\) is nondecreasing and concave.  If
\(0<a<1/2\) and \(|r|\le2a\sig\), then
\[
    \omega_\sig^{\rm Tuk}(r)
    \ge(1-4a^2)^2.
\]
For \(0\le x\le1\), the normalized score is
\(x(1-x^2)^2\), whose derivative is
\((1-x^2)(1-5x^2)\).  The maximum is attained at \(x=1/\sqrt5\) and equals
\[
    \frac1{\sqrt5}\left(1-\frac15\right)^2
    =\frac{16}{25\sqrt5}.
\]
The score is identically zero for \(|r|\ge\sig\), so it is redescending.

\paragraph{Truncated least squares.}
Let
\begin{equation}\label{eq:genred-tls-loss}
    \rho_\sig^{\rm TLS}(r)
    =\frac12\min\{r^2,\sig^2\}.
\end{equation}
Then
\[
    q_\sig^{\rm TLS}(u)=\frac12\min\{u,\sig^2\}.
\]
This function is continuous, nondecreasing, and concave: its slope is
\(1/2\) below \(u=\sig^2\) and zero above it.  Choose
\begin{equation}\label{eq:genred-tls-selector}
    \omega_\sig^{\rm TLS}(r)
    =
    \begin{cases}
    1,&|r|<\sig,\\
    0,&|r|>\sig,
    \end{cases}
\end{equation}
with any value in \([0,1]\) when \(|r|=\sig\).  At the kink this is exactly
the full interval of valid supergradient slopes after multiplication by two,
so (R2) holds for every such selection.  If \(a<1/2\), every clean-window
residual satisfies \(|r|<\sig\), hence \(m_a=1\).  Finally,
\[
    |\omega_\sig^{\rm TLS}(r)r|\le\sig,
\]
and the selected score is zero outside the truncation interval.  This proves
all four claims.
\end{proof}

\subsection{Weighted clean Green response for a general loss}
\label{sec:genred-weighted-green}

The unweighted clean Green estimate has already been proved in
Lemma~\ref{lem:unweighted-clean-green}.  We record the perturbation needed for
a general clean-window weight floor.  Put
\begin{equation}\label{eq:genred-GJ-deltaJ}
    G_J\defeq\frac{32}{7},
    \qquad
    \delta_J\defeq\frac1{384}.
\end{equation}
For a floor \(m_a\), define
\begin{equation}\label{eq:genred-beta-Ga}
    \beta_a\defeq6G_J(1-m_a),
    \qquad
    \mathsf G_a\ge\frac{G_J}{1-\beta_a}.
\end{equation}
The second expression is used only when \(\beta_a<1\).

\begin{lemma}[General weighted clean Green response]
\label{lem:genred-weighted-green}
Assume \(\Delta_E\le\delta_Jn\), and let
\(L_\Theta=B_0^\top\Theta B_0\) be a clean triangle Laplacian whose
conductances lie in \([m_a,1]\).  If \(\beta_a<1\), then for every vector
\(y\) on \(\Tzero\),
\begin{equation}\label{eq:genred-weighted-green-bound}
    \norm{L_\Theta^\dagger y}_{\infty/\one}
    \le
    \frac{\mathsf G_a}{n}\norm{y}_{\infty/\one}.
\end{equation}
\end{lemma}

\begin{proof}
Write
\[
    L_\Theta=L_0-H,
    \qquad
    H=B_0^\top(I-\Theta)B_0.
\]
The matrix \(H\) is a clean-triangle weighted Laplacian whose edge
conductances are at most \(1-m_a\).  A clean triangle has degree at most
\(3(n-3)\), so
\begin{equation}\label{eq:genred-H-infty}
    \norm{Hx}_\infty
    \le 2(1-m_a)3(n-3)\norm{x}_\infty
    \le6n(1-m_a)\norm{x}_\infty.
\end{equation}
Since \(H\one=0\), choose a representative of the quotient class of \(x\)
whose \(\ell_\infty\) norm is arbitrarily close to
\(\norm{x}_{\infty/\one}\).  Lemma~\ref{lem:unweighted-clean-green} then gives
\[
    \norm{L_0^\dagger Hx}_{\infty/\one}
    \le
    \frac{G_J}{n}\norm{Hx}_{\infty/\one}
    \le
    6G_J(1-m_a)\norm{x}_{\infty/\one}
    =\beta_a\norm{x}_{\infty/\one}.
\]
Thus \(I-L_0^\dagger H\) is invertible on the quotient modulo constants, with
inverse norm at most \((1-\beta_a)^{-1}\).  Since
\[
    L_\Theta=L_0(I-L_0^\dagger H)
\]
on that quotient, we obtain
\[
    \norm{L_\Theta^\dagger y}_{\infty/\one}
    \le
    \frac{G_J}{(1-\beta_a)n}\norm{y}_{\infty/\one}.
\]
The choice of \(\mathsf G_a\) in \eqref{eq:genred-beta-Ga} proves the lemma.
\end{proof}

\subsection{General proximal MM iteration}
\label{sec:genred-proximal-mm}

Let \(\mathcal C\) be the retained shared-edge row set, let \(D\) be its full signed incidence matrix, and write
\[
    r(z)=Dz-g.
\]
For a redescending MM-admissible loss, define the full objective
\begin{equation}\label{eq:genred-full-objective}
    F_\sig^\rho(z)
    \defeq
    \sum_{f\in\mathcal C}\wbase_f\rho_\sig(r_f(z)).
\end{equation}
Given an inner iterate \(z^m\), choose an admissible weight
\(\omega_f^m=\omega_\sig(r_f(z^m))\), and set
\begin{equation}\label{eq:genred-frozen-conductance}
    \theta_f^m=\wbase_f\omega_f^m,
    \qquad
    \Theta^m=\diag(\theta_f^m:f\in\mathcal C).
\end{equation}
The exact proximal MM target is
\begin{equation}\label{eq:genred-proximal-target}
    z^{m+1}
    =\argmin_z
    \left\{
    \frac12\sum_{f\in\mathcal C}\theta_f^m r_f(z)^2
    +\frac{\mustar}{2}\norm{z-z^m}_2^2
    \right\}.
\end{equation}
The proximal term makes this quadratic strictly convex for every weight
selection, including the zero weights produced by Tukey and truncated least
squares.

\begin{lemma}[MM descent for smooth and truncated losses]
\label{lem:genred-mm-descent}
For every admissible weight selection, the exact target
\eqref{eq:genred-proximal-target} satisfies
\begin{equation}\label{eq:genred-descent}
    F_\sig^\rho(z^{m+1})
    \le
    F_\sig^\rho(z^m)
    -\frac{\mustar}{2}\norm{z^{m+1}-z^m}_2^2.
\end{equation}
Consequently,
\begin{equation}\label{eq:genred-square-summable-steps}
    \sum_{m=0}^\infty\norm{z^{m+1}-z^m}_2^2<\infty,
    \qquad
    \norm{z^{m+1}-z^m}_2\longrightarrow0.
\end{equation}
\end{lemma}

\begin{proof}
Apply the tangent inequality \eqref{eq:genred-tangent-majorizer} with
\(v=r_f(z)^2\) and tangent point \(r_f(z^m)^2\):
\[
    \rho_\sig(r_f(z))
    \le
    \rho_\sig(r_f(z^m))
    +\frac12\omega_f^m
      \bigl(r_f(z)^2-r_f(z^m)^2\bigr).
\]
Multiplying by \(\wbase_f\) and summing gives a global quadratic majorizer
which is exact at \(z=z^m\).  The point \(z^{m+1}\) minimizes this majorizer
plus the nonnegative proximal term.  Comparing its value with the value at
\(z^m\), where the proximal term is zero, gives
\eqref{eq:genred-descent}.  Summing that inequality and using
\(F_\sig^\rho\ge0\) proves \eqref{eq:genred-square-summable-steps}.
\end{proof}

For the general loss, use the basin
\begin{equation}\label{eq:genred-basin}
    \basin_\sig(a)
    \defeq
    \{z=(z_0,z_b):\err(z_0)\le a\sig\}.
\end{equation}
The comparison point associated with \(z^m\in\basin_\sig(a)\) is exactly as in
the Cauchy proof:
\[
    \bar z^m=(z_0^\star+\alpha_m\one,z_b^m),
    \qquad
    \norm{z_0^m-z_0^\star-\alpha_m\one}_\infty\le a\sig.
\]
Define the lagged bad current
\begin{equation}\label{eq:genred-lagged-current}
    j_f^m
    \defeq
    \theta_f^m r_f(\bar z^m),
    \qquad f\in\Cb.
\end{equation}

\begin{lemma}[Clean conductances and bad current for a general loss]
\label{lem:genred-clean-current}
If \(z^m\in\basin_\sig(a)\), then every clean row has conductance in
\([m_a,1]\), and every bad-row current satisfies
\begin{equation}\label{eq:genred-current-cap}
    |j_f^m|\le J_a\sig,
    \qquad
    J_a\defeq K+a.
\end{equation}
\end{lemma}

\begin{proof}
For a clean row \(c\), exact clean consistency gives
\[
    |r_c(z^m)|
    \le2\err(z_0^m)
    \le2a\sig.
\]
Its base weight equals one, so (R3) and the upper bound
\(\omega_\sig\le1\) give a conductance in \([m_a,1]\).

For a bad row with no clean endpoint,
\(r_f(\bar z^m)=r_f(z^m)\), and therefore (R4) gives
\[
    |j_f^m|
    =\wbase_f|\omega_f^m r_f(z^m)|
    \le K\sig.
\]
For a bad row with one clean endpoint,
\[
    |r_f(\bar z^m)-r_f(z^m)|\le a\sig.
\]
Using \(0\le\wbase_f\omega_f^m\le1\) and (R4),
\[
    |j_f^m|
    \le
    \wbase_f|\omega_f^m r_f(z^m)|
    +\wbase_f\omega_f^m a\sig
    \le(K+a)\sig.
\]
This proves the lemma.
\end{proof}

\begin{lemma}[General Schur-complement bounds]
\label{lem:genred-schur-bounds}
Assume \(z^m\in\basin_\sig(a)\), and let
\(\delta=z^{m+1}-\bar z^m=(\delta_0,\delta_b)\).  Let \(L_m\) be the
clean-clean Laplacian with the frozen clean conductances.  Then
\begin{equation}\label{eq:genred-clean-schur-equation}
    (L_m+\mustar I+S_m)\delta_0
    =\mustar e^m-p_m,
\end{equation}
where \(e^m=z_0^m-z_0^\star-\alpha_m\one\), and where \(S_m\) is a passive
boundary response satisfying
\begin{equation}\label{eq:genred-S-bound}
    \norm{S_mx}_\infty
    \le2\dtri\norm{x}_\infty.
\end{equation}
The forcing satisfies
\begin{equation}\label{eq:genred-p-bound}
    \norm{p_m}_\infty
    \le
    \dtri J_a\sig
    \left(1+\frac{D_n}{\mustar}\right).
\end{equation}
Finally,
\begin{equation}\label{eq:genred-delta-max}
    \norm{\delta_0}_\infty
    \le
    \norm{e^m}_\infty
    +\frac{\norm{p_m}_\infty}{\mustar}.
\end{equation}
\end{lemma}

\begin{proof}
Partition the bad-row incidence matrix as \([A\ C]\), and let
\(\Theta_b^m\) be the diagonal matrix of bad frozen conductances.  Expanding
the quadratic subproblem around \(\bar z^m\) gives
\begin{align*}
    (L_m+A^\top\Theta_b^mA+\mustar I)\delta_0
    +A^\top\Theta_b^mC\delta_b
    &=\mustar e^m-A^\top j^m,\\
    C^\top\Theta_b^mA\delta_0
    +(C^\top\Theta_b^mC+\mustar I)\delta_b
    &=-C^\top j^m.
\end{align*}
Put
\[
    K_b=C^\top\Theta_b^mC+\mustar I.
\]
The matrix \(K_b\) is a grounded graph Laplacian.  The maximum principle gives
\begin{equation}\label{eq:genred-nuisance-inverse}
    \norm{K_b^{-1}v}_\infty
    \le\frac1{\mustar}\norm v_\infty.
\end{equation}
Eliminating \(\delta_b\) yields \eqref{eq:genred-clean-schur-equation} with
\begin{align*}
    S_m
    &=A^\top\Theta_b^mA
    -A^\top\Theta_b^mC K_b^{-1}C^\top\Theta_b^mA,\\
    p_m
    &=A^\top j^m
    -A^\top\Theta_b^mC K_b^{-1}C^\top j^m.
\end{align*}
Exactly as in Lemma~\ref{lem:bad-force-local}, \(S_m\) is the
Dirichlet-to-Neumann response of the passive bad-row network after the
non-clean triangle variables are eliminated and grounded by \(\mustar I\).
Lemma~\ref{lem:passive-boundary-response} therefore gives
\eqref{eq:genred-S-bound}.

By Lemma~\ref{lem:genred-clean-current},
\[
    \norm{A^\top j^m}_\infty\le\dtri J_a\sig.
\]
Every nuisance triangle is incident to at most \(D_n\) retained rows, so
\[
    \norm{C^\top j^m}_\infty\le D_nJ_a\sig.
\]
Combining this with \eqref{eq:genred-nuisance-inverse}, and then multiplying
by the boundary-local matrix \(A^\top\Theta_b^mC\), gives
\[
    \norm{A^\top\Theta_b^mC K_b^{-1}C^\top j^m}_\infty
    \le
    \dtri\frac{D_n}{\mustar}J_a\sig.
\]
This proves \eqref{eq:genred-p-bound}.

Finally, \(L_m+S_m\) is an \(M\)-matrix with nonpositive off-diagonal entries
and nonnegative row sums.  Applying the same maximum-principle argument as in
Lemma~\ref{lem:max-principle-delta} to
\eqref{eq:genred-clean-schur-equation} gives
\eqref{eq:genred-delta-max}.
\end{proof}

\begin{proposition}[Loss-uniform invariant-basin condition]
\label{prop:genred-one-step}
Let the loss be redescending MM-admissible at radius \(a\), assume
\(\beta_a<1\), and choose \(\mathsf G_a\) as in
\eqref{eq:genred-beta-Ga}.  Set
\begin{equation}\label{eq:genred-kappa-choice}
    \kappa_a\defeq\frac{1}{4\mathsf G_a},
    \qquad
    \mustar=\kappa_an.
\end{equation}
Let \(c>0\) satisfy
\begin{equation}\label{eq:genred-c-smallness}
    c<\delta_J
\end{equation}
and
\begin{equation}\label{eq:genred-one-step-smallness}
    12\mathsf G_a c a
    +6\mathsf G_aJ_a c
      \left(1+\frac{12c}{\kappa_a}\right)
      \left(1+\frac3{\kappa_a}\right)
    <\frac a2.
\end{equation}
Then, for all sufficiently large \(n\), if \(\Delta_E\le cn\) and
\(z^m\in\basin_\sig(a)\), every exact proximal target
\eqref{eq:genred-proximal-target} satisfies
\begin{equation}\label{eq:genred-target-in-basin}
    z^{m+1}\in\basin_\sig(a).
\end{equation}
This holds for every admissible choice of a Tukey or truncated-LS zero/kink
weight.
\end{proposition}

\begin{proof}
Let
\[
    E_m=\norm{e^m}_\infty\le a\sig,
    \qquad
    P_m=\norm{p_m}_\infty,
    \qquad
    E_{m+1}=\norm{\delta_0}_{\infty/\one}.
\]
The clean conductances lie in \([m_a,1]\), so
Lemma~\ref{lem:genred-weighted-green} applies to \(L_m\).  Rearranging
\eqref{eq:genred-clean-schur-equation},
\[
    L_m\delta_0
    =\mustar(e^m-\delta_0)-S_m\delta_0-p_m.
\]
Applying the quotient Green estimate and the triangle inequality gives
\begin{equation}\label{eq:genred-one-step-recurrence-a}
    E_{m+1}
    \le
    \frac{\mathsf G_a}{n}\mustar(E_m+E_{m+1})
    +\frac{\mathsf G_a}{n}
      \bigl(\norm{S_m\delta_0}_\infty+P_m\bigr).
\end{equation}
By \eqref{eq:genred-S-bound} and \eqref{eq:genred-delta-max},
\[
    \norm{S_m\delta_0}_\infty
    \le
    2\dtri\left(E_m+\frac{P_m}{\mustar}\right).
\]
Move the term containing \(E_{m+1}\) to the left.  Since
\(\mustar=\kappa_an\),
\begin{equation}\label{eq:genred-one-step-recurrence-b}
    E_{m+1}
    \le
    \frac{\mathsf G_a/n}{1-\mathsf G_a\kappa_a}
    \left[
       (\mustar+2\dtri)E_m
       +\left(1+\frac{2\dtri}{\mustar}\right)P_m
    \right].
\end{equation}
Write \(\varepsilon=\Delta_E/n\).  Use
\(\dtri\le6\Delta_E\), \(D_n\le3n\),
\eqref{eq:genred-p-bound}, and \(E_m\le a\sig\).  After division by
\(\sig\),
\begin{equation}\label{eq:genred-one-step-recurrence-final}
    \frac{E_{m+1}}{\sig}
    \le
    \frac{1}{1-\mathsf G_a\kappa_a}
    \left[
       \mathsf G_a(\kappa_a+12\varepsilon)a
       +6\mathsf G_aJ_a\varepsilon
        \left(1+\frac{12\varepsilon}{\kappa_a}\right)
        \left(1+\frac3{\kappa_a}\right)
    \right].
\end{equation}
The choice \(\kappa_a=1/(4\mathsf G_a)\) gives
\(\mathsf G_a\kappa_a=1/4\).  At zero corruption, the numerator in brackets
is \(a/4\).  Condition \eqref{eq:genred-one-step-smallness} bounds all
remaining terms by a number strictly smaller than \(a/2\).  Therefore the
right-hand side of \eqref{eq:genred-one-step-recurrence-final} is at most
\(a\) for all sufficiently large \(n\), as required by the clean Green lemmas.  Hence \(E_{m+1}\le a\sig\), which is
\eqref{eq:genred-target-in-basin}.
\end{proof}

\subsection{Asymptotic MM stationarity without coercivity}
\label{sec:genred-asymptotic-stationarity}

For the weights selected at \(z^m\), define the row-score vector
\begin{equation}\label{eq:genred-selected-score}
    s_f^m
    \defeq
    \wbase_f\omega_f^m r_f(z^m),
    \qquad
    s^m=(s_f^m)_{f\in\mathcal C},
\end{equation}
and define the selected MM stationarity residual
\begin{equation}\label{eq:genred-stationarity-residual}
    g^m\defeq D^\top s^m.
\end{equation}
For Cauchy, Welsch, and Tukey this is the ordinary gradient of
\(F_\sig^\rho\).  For truncated least squares it is the score balance
associated with the selected valid supergradient at every threshold row.

\begin{lemma}[The selected stationarity residual vanishes]
\label{lem:genred-residual-vanishes}
Fix \(n\), \(\sig>0\), and an arbitrary starting point.  For every exact
proximal-MM orbit \eqref{eq:genred-proximal-target},
\begin{equation}\label{eq:genred-g-to-zero}
    \norm{g^m}_2\longrightarrow0,
    \qquad
    \norm{g^m}_\infty\longrightarrow0.
\end{equation}
No boundedness of \(\{z^m\}\), and no coercivity of the loss, is required.
\end{lemma}

\begin{proof}
Let \(d^m=z^{m+1}-z^m\).  The normal equation of the quadratic subproblem is
\[
    D^\top\Theta^m(Dz^{m+1}-g)+\mustar d^m=0.
\]
Since
\(Dz^{m+1}-g=r(z^m)+Dd^m\), this becomes
\begin{equation}\label{eq:genred-g-step-identity}
    g^m
    =-\bigl(D^\top\Theta^mD+\mustar I\bigr)d^m.
\end{equation}
The weighted graph Laplacian \(D^\top\Theta^mD\) has maximum weighted degree
at most \(D_n\), because \(0\le\theta_f^m\le1\).  Hence its spectral norm is
at most \(2D_n\), and
\[
    \norm{g^m}_2
    \le(2D_n+\mustar)\norm{d^m}_2.
\]
For fixed \(n\), the coefficient is finite.  Lemma~\ref{lem:genred-mm-descent}
gives \(\norm{d^m}_2\to0\), proving the first limit.  The second follows from
\(\norm{g^m}_\infty\le\norm{g^m}_2\).
\end{proof}

The next estimate is the loss-uniform analogue of the stationary Cauchy
endpoint proposition, but it applies to every finite inner iterate and includes
its explicitly computable stationarity residual.

\begin{proposition}[Clean error from an approximate score balance]
\label{prop:genred-approx-endpoint}
Assume \(\Delta_E\le\delta_Jn\), \(\beta_a<1\), and
\(z^m\in\basin_\sig(a)\).  Then
\begin{equation}\label{eq:genred-approx-endpoint-bound}
    \err(z_0^m)
    \le
    \frac{\mathsf G_a}{n}\norm{g^m}_\infty
    +6\mathsf G_aK\frac{\Delta_E}{n}\sig.
\end{equation}
Consequently, along any basin-preserving exact proximal-MM orbit,
\begin{equation}\label{eq:genred-limsup-endpoint}
    \limsup_{m\to\infty}\err(z_0^m)
    \le
    6\mathsf G_aK\frac{\Delta_E}{n}\sig.
\end{equation}
\end{proposition}

\begin{proof}
Choose a gauge \(\alpha_m\) and put
\[
    e^m=z_0^m-z_0^\star-\alpha_m\one,
    \qquad
    \norm{e^m}_\infty=\err(z_0^m)\le a\sig.
\]
For a clean row \(c\), exact clean consistency gives
\(r_c(z^m)=b_c^\top e^m\), and its selected weight lies in \([m_a,1]\).
Let \(L_{\eta^m}\) be the clean Laplacian with those selected weights.  Split
the selected bad-row score vector from \eqref{eq:genred-selected-score} as
\(s_b^m\).  The clean block of \(g^m=D^\top s^m\) is
\begin{equation}\label{eq:genred-clean-gradient-block}
    g_0^m
    =L_{\eta^m}e^m+A^\top s_b^m.
\end{equation}
Thus
\[
    L_{\eta^m}e^m
    =g_0^m-A^\top s_b^m.
\]
At most \(\dtri\le6\Delta_E\) bad rows touch one clean triangle, and every
bad score satisfies \(|s_f^m|\le K\sig\).  Hence
\begin{equation}\label{eq:genred-bad-score-force}
    \norm{A^\top s_b^m}_\infty
    \le6\Delta_EK\sig.
\end{equation}
Apply Lemma~\ref{lem:genred-weighted-green} to
\eqref{eq:genred-clean-gradient-block}.  Since
\(\norm{g_0^m}_\infty\le\norm{g^m}_\infty\),
\begin{align*}
    \err(z_0^m)
    &=\norm{e^m}_{\infty/\one}\\
    &\le
    \frac{\mathsf G_a}{n}
    \norm{g_0^m-A^\top s_b^m}_{\infty/\one}\\
    &\le
    \frac{\mathsf G_a}{n}\norm{g^m}_\infty
    +6\mathsf G_aK\frac{\Delta_E}{n}\sig.
\end{align*}
This proves \eqref{eq:genred-approx-endpoint-bound}.  Combining it with
Lemma~\ref{lem:genred-residual-vanishes} gives
\eqref{eq:genred-limsup-endpoint}.
\end{proof}

\begin{remark}[Why bounded redescending losses cause no gap]
For Welsch, Tukey, and truncated least squares, a nuisance component can in
principle move far enough that all of its incident robust weights become tiny
or zero.  The full vector need not have a convergent subsequence.  The proof
above never assumes otherwise.  The clean block remains in a compact quotient
basin by Proposition~\ref{prop:genred-one-step}, the proximal descent estimate
forces successive steps to vanish, and identity
\eqref{eq:genred-g-step-identity} forces the selected score balance to vanish.
Proposition~\ref{prop:genred-approx-endpoint} then controls the clean variables
directly.  This is the step that permits a rigorous theorem for all four
redescending losses.
\end{remark}

\subsection{General annealing theorem and admissible schedules}
\label{sec:genred-main-theorem}

\begin{theorem}[Complete-graph recovery for general redescending proximal MM]
\label{thm:genred-main}
Assume the complete-graph noiseless model, triangle retention, and base-weight
conditions of Theorem~\ref{thm:main}.  Fix \(a>0\), and let
\(\{\rho_\sig\}_{\sig>0}\) be redescending MM-admissible at radius \(a\), with
constants \(m_a\) and \(K\).  Assume \(\beta_a<1\), choose
\(\mathsf G_a\) as in \eqref{eq:genred-beta-Ga}, set
\begin{equation}\label{eq:genred-main-mu}
    \mustar=\kappa_an,
    \qquad
    \kappa_a=\frac{1}{4\mathsf G_a},
\end{equation}
and choose \(c>0\) satisfying
\eqref{eq:genred-c-smallness} and
\eqref{eq:genred-one-step-smallness}.

Let \(\{\sig_k\}_{k\ge0}\) be any decreasing positive sequence such that
\begin{equation}\label{eq:genred-schedule}
    \sig_k\longrightarrow0,
    \qquad
    \frac{6\mathsf G_aKc}{a}
    <
    \frac{\sig_{k+1}}{\sig_k}
    <1
    \quad\text{for every }k.
\end{equation}
A geometric schedule \(\sig_{k+1}=\tau\sig_k\) is admissible whenever
\begin{equation}\label{eq:genred-geometric-schedule}
    \frac{6\mathsf G_aKc}{a}<\tau<1.
\end{equation}

Initialize with the same weighted least-squares solution as in
\eqref{eq:ls-init}, and choose
\begin{equation}\label{eq:genred-sigma0-choice}
    \sig_0\ge\frac{4}{a}R_{\LS}.
\end{equation}
At scale \(\sig_k\), warm-start the exact proximal-MM iteration
\eqref{eq:genred-proximal-target} from the previous-scale output.  Define the
positive schedule margin
\begin{equation}\label{eq:genred-schedule-margin}
    M_k
    \defeq
    a\sig_{k+1}-6\mathsf G_aKc\sig_k>0.
\end{equation}
Stop the inner iteration at any finite index \(m_k\) satisfying
\begin{equation}\label{eq:genred-stopping-rule}
    \norm{g^{k,m_k}}_\infty
    \le
    \frac{n}{\mathsf G_a}M_k,
\end{equation}
where \(g^{k,m}\) is the selected MM stationarity residual at the current
iterate.  Set \(z^{(k)}=z^{k,m_k}\) and use this full finite vector as the warm
start for scale \(\sig_{k+1}\).

Then, for all sufficiently large \(n\), the condition
\begin{equation}\label{eq:genred-main-degree}
    \Delta_E\le cn
\end{equation}
implies all of the following.
\begin{enumerate}[label=(\roman*),leftmargin=2.5em]
    \item Every inner iterate at every scale remains in its current clean
    basin \(\basin_{\sig_k}(a)\).

    \item The stopping rule \eqref{eq:genred-stopping-rule} is reached after
    finitely many inner iterations at every scale.

    \item Every stage output lies in the next clean basin:
    \begin{equation}\label{eq:genred-next-basin-output}
        \err(z_0^{(k)})\le a\sig_{k+1}.
    \end{equation}

    \item The clean triangle log-scales converge uniformly to ground truth up
    to global additive gauge:
    \begin{equation}\label{eq:genred-main-convergence}
        \min_{\alpha\in\R}
        \max_{t\in\Tzero}
        |z_t^{(k)}-z_t^\star-\alpha|
        \le a\sig_{k+1}
        \longrightarrow0.
    \end{equation}
\end{enumerate}
For a geometric schedule, the right-hand side in
\eqref{eq:genred-main-convergence} is \(a\sig_0\tau^{k+1}\).
\end{theorem}

\begin{proof}
Because \(c<\delta_J=1/384<1/8\), the auxiliary-vertex construction
from Lemma~\ref{lem:diam} still gives a clean path of length at most six
between arbitrary clean triangles for all sufficiently large \(n\).  Indeed,
at the three choices one excludes at most \(6\Delta_E+O(1)\),
\(7\Delta_E+O(1)\), and \(8\Delta_E+O(1)\) vertices, respectively, and
each quantity is smaller than \(n\).  Repeating the proof of
Lemma~\ref{lem:ls-basin} with this diameter bound gives
\[
    \err(z_0^{\LS})\le3R_{\LS}
    \le a\sig_0
\]
under \eqref{eq:genred-sigma0-choice}.  Thus the first inner warm start lies in
\(\basin_{\sig_0}(a)\).

Fix a scale \(k\), and suppose its warm start belongs to
\(\basin_{\sig_k}(a)\).  Proposition~\ref{prop:genred-one-step} shows that one
exact proximal-MM target also lies in this basin.  Induction over the inner
index proves that every inner iterate remains in the basin, establishing (i).

At the fixed scale \(\sig_k\), Lemma~\ref{lem:genred-residual-vanishes} gives
\[
    \norm{g^{k,m}}_\infty\longrightarrow0.
\]
The right-hand side of \eqref{eq:genred-stopping-rule} is strictly positive by
\eqref{eq:genred-schedule-margin}.  Therefore some finite inner index satisfies
the stopping rule.  This proves (ii).

Apply Proposition~\ref{prop:genred-approx-endpoint} at that finite iterate:
\begin{align*}
    \err(z_0^{(k)})
    &\le
    \frac{\mathsf G_a}{n}\norm{g^{k,m_k}}_\infty
    +6\mathsf G_aK\frac{\Delta_E}{n}\sig_k\\
    &\le
    M_k+6\mathsf G_aKc\sig_k
    =a\sig_{k+1}.
\end{align*}
This proves (iii).  In particular, the finite full output at scale \(k\) is a
valid warm start for scale \(k+1\), closing the induction over all scales.
Finally, \(\sig_{k+1}\to0\) proves (iv).  Under a geometric schedule,
\(\sig_{k+1}=\sig_0\tau^{k+1}\), giving the displayed rate.
\end{proof}

\begin{corollary}[Clean-basin stationary continuation]
\label{cor:genred-stationary-continuation}
Assume the complete-graph noiseless model and base-weight conditions of
Theorem~\ref{thm:main}.  Let the loss be redescending MM-admissible at radius
\(a\), assume \(\beta_a<1\), choose \(\mathsf G_a\) as in
\eqref{eq:genred-beta-Ga}, and assume only
\(\Delta_E\le cn\) with \(c<\delta_J\).  Suppose that, at every scale
\(\sig_k\), a clean-basin point is supplied whose selected score balance is
exact:
\[
    g^{(k)}=0,
    \qquad
    \err(z_0^{(k)})\le a\sig_k.
\]
Then
\begin{equation}\label{eq:genred-stationary-bound}
    \err(z_0^{(k)})
    \le
    6\mathsf G_aK\frac{\Delta_E}{n}\sig_k.
\end{equation}
Consequently, a schedule is valid whenever
\begin{equation}\label{eq:genred-stationary-schedule}
    \frac{6\mathsf G_aKc}{a}
    \le
    \frac{\sig_{k+1}}{\sig_k}
    <1.
\end{equation}
This stationary-continuation corollary does not require the stronger one-step
condition \eqref{eq:genred-one-step-smallness}; that condition is needed only
to prove that every finite proximal-MM iterate stays in the basin.
\end{corollary}

\begin{proof}
Set \(g^m=0\) in Proposition~\ref{prop:genred-approx-endpoint}.  This gives
\eqref{eq:genred-stationary-bound}.  If \(\Delta_E\le cn\), then the right-hand
side is at most \(6\mathsf G_aKc\sig_k\), which is at most
\(a\sig_{k+1}\) under \eqref{eq:genred-stationary-schedule}.  This closes the
clean-basin continuation induction.
\end{proof}

\subsection{Explicit specializations and numerical checks}
\label{sec:genred-specializations}

The numerical entries in Table~\ref{tab:genred-comparison} follow from the calculations below.
At \(a=1/20\), the four clean floors give
\begin{align*}
\beta_a^{\rm Cau}
&=6\frac{32}{7}\left(1-\frac{100}{101}\right)
=\frac{192}{707}<1,
&\frac{G_J}{1-\beta_a^{\rm Cau}}
&=\frac{22624}{3605}<7,\\
\beta_a^{\rm Wel}
&=6\frac{32}{7}\left(1-e^{-1/100}\right)<0.273,
&\frac{G_J}{1-\beta_a^{\rm Wel}}
&<6.29<7,\\
\beta_a^{\rm Tuk}
&=6\frac{32}{7}\left(1-\left(\frac{99}{100}\right)^2\right)<0.546,
&\frac{G_J}{1-\beta_a^{\rm Tuk}}
&<10.07<11,\\
\beta_a^{\rm TLS}
&=0,
&\frac{G_J}{1-\beta_a^{\rm TLS}}
&=\frac{32}{7}<5.
\end{align*}
Thus the rounded values
\[
    \mathsf G_a^{\rm Cau}=7,
    \qquad
    \mathsf G_a^{\rm Wel}=7,
    \qquad
    \mathsf G_a^{\rm Tuk}=11,
    \qquad
    \mathsf G_a^{\rm TLS}=5
\]
are valid.

With \(\kappa_a=1/(4\mathsf G_a)\), define the left-hand side of the one-step
condition by
\begin{equation}\label{eq:genred-Xi-definition}
    \Xi(c;\mathsf G_a,K)
    \defeq
    12\mathsf G_a c a
    +6\mathsf G_a(K+a)c
      \left(1+\frac{12c}{\kappa_a}\right)
      \left(1+\frac3{\kappa_a}\right).
\end{equation}
For the constants displayed in the table, direct substitution gives
\begin{align}
    \Xi(10^{-5};7,1/2)
    &<0.019743<0.025=\frac a2,
    \label{eq:genred-check-cauchy}\\
    \Xi(1.2\cdot10^{-5};7,1/\sqrt{2e})
    &<0.020649<0.025,
    \label{eq:genred-check-welsch}\\
    \Xi(8\cdot10^{-6};11,16/(25\sqrt5))
    &<0.023764<0.025,
    \label{eq:genred-check-tukey}\\
    \Xi(10^{-5};5,1)
    &<0.019292<0.025.
    \label{eq:genred-check-tls}
\end{align}
All four degree constants are also smaller than \(\delta_J=1/384\).  Hence
Theorem~\ref{thm:genred-main} applies with the displayed
\(c_{\rm MM}\) values.

The lower geometric schedule ratio is
\begin{equation}\label{eq:genred-tau-min}
    \tau_{\min}=\frac{6\mathsf G_aKc_{\rm MM}}{a}.
\end{equation}
The four values are, respectively,
\[
    0.0042,
    \qquad
    0.0043232,
    \qquad
    0.0030225,
    \qquad
    0.0060.
\]
Thus \(\tau=1/2\) has a very large schedule margin in every case.

For the stationary-continuation corollary with \(\tau=1/2\), the admissible
fraction is
\begin{equation}\label{eq:genred-cstat-half}
    c
    <
    \frac{a/2}{6\mathsf G_aK}.
\end{equation}
This gives the last column of Table~\ref{tab:genred-comparison}.  In
particular, the conditional clean-basin stationary theory permits
\(c=10^{-3}\) for Cauchy, Welsch, and Tukey with the half-scale schedule.  For
truncated least squares, \(c=10^{-3}\) is also allowed if the schedule is
slowed to any \(\tau>0.6\), for example \(\tau=2/3\).  These larger values do
not replace the smaller unconditional proximal-MM constants: the difference
is exactly the cost of proving a uniform invariant basin for every finite WLS
step against an arbitrary nuisance triangle network.

\subsection{Discussion of the four objectives}
\label{sec:genred-discussion}

\paragraph{Cauchy.}
Cauchy has a positive weight at every finite residual and an unbounded, slowly
growing loss.  The preceding Cauchy section therefore proves the stronger fact
that the full warm-started orbit is bounded and has a cluster-point fixed
point.  The general theorem in this section is consistent with that result but
uses only approximate score balance, so it also applies verbatim to the other
losses.

\paragraph{Welsch.}
Welsch is smooth and bounded.  Its weights are positive for every finite
residual but decay exponentially.  The score cap
\(K=1/\sqrt{2e}\) is smaller than the Cauchy cap, while its clean-window floor
at \(a=1/20\) is almost the same.  This produces a slightly larger explicit
\(c_{\rm MM}\) in the unoptimized table.

\paragraph{Tukey.}
Tukey is smooth at the cutoff and gives exactly zero weight to residuals with
\(|r|\ge\sig\).  Its maximum score is small, but its clean weight decreases
more rapidly as the basin radius grows.  At the common radius \(a=1/20\), a
rounded Green constant \(11\) is sufficient.  The proof permits arbitrary
zero-weight bad subnetworks because the proximal term keeps every quadratic
subproblem well posed and the clean endpoint argument does not require the
nuisance variables to converge.

\paragraph{Truncated least squares.}
Truncated least squares is nonsmooth only at \(|r|=\sig\).  The concave
squared-residual function has a full supergradient interval at the kink, and
any weight in \([0,1]\) gives a valid global MM tangent.  The theorem is
therefore independent of the tie-breaking rule.  Clean rows never encounter
the kink because \(2a<1\), so their weight is exactly one throughout the clean
basin.  The larger score cap \(K=1\) requires either a slightly smaller
stationary-only degree fraction for a half-scale schedule or a slower
annealing schedule.

\paragraph{Why annealing is essential.}
At a fixed positive scale, Proposition~\ref{prop:genred-approx-endpoint} gives
an error floor of order
\[
    \frac{\Delta_E}{n}\sig.
\]
The redescending shape suppresses large residuals, but exact recovery in this
deterministic adversarial model comes from driving the global score cap
\(K\sig\) to zero.  The schedule condition balances that shrinking bad force
against the clean basin at the next scale.

\paragraph{Geometric assumptions.}
This extension adds no distributional assumption on camera locations, no
well-distributed-location constant, and no quantitative lower bound on clean
triangle angles.  It uses exactly the same qualitative nondegeneracy and
complete-graph assumptions as Theorem~\ref{thm:main}.  Consequently, none of
the loss constants in Table~\ref{tab:genred-comparison} depends on a location
distribution parameter or a triangle-angle parameter.

\section{Theory for Global \texorpdfstring{$\ell_1$}{L1} Log-Scale Synchronization}
\label{app:trip-l1-scale-theory}

\paragraph{Overview.}
This section gives a complementary deterministic exact-recovery theorem for the
same ideal full all-triangle log-scale synchronization model, but with the
Cauchy objective replaced by a globally minimized weighted \(\ell_1\) objective.
The algorithmic and analytic distinction is important.  The Cauchy theorem above
tracks a warm-started fixed point of a nonconvex annealed IRLS map.  The
\(\ell_1\) objective below is convex, and the theorem concerns every global
minimizer.  Consequently, no basin, proximal damping, continuation path,
fixed-point extraction, or Green-response estimate is needed.  The proof is a
weighted robust-nullspace argument: every nonconstant perturbation of the clean
triangle scales creates more clean-clean \(\ell_1\) variation than all constraints
involving non-clean triangles can possibly remove.

The retained triangle set, the clean/non-clean partition, the exact clean local
scales, and the shared-edge equations are exactly those defined in the Cauchy
section above.  They are repeated where useful so that the \(\ell_1\) theorem can
be read independently.  The clean/non-clean partition is used only by the proof;
the optimization problem is not given this partition.

\subsection{Weighted all-triangle \texorpdfstring{$\ell_1$}{L1} objective}

Let
\[
    \Tall=\Tzero\sqcup\Tb
\]
be the retained triangle set.  Let \(\mathcal C^{(1)}\) be the same retained
shared-edge row set used by the Cauchy synchronization problem, and decompose it
as
\[
    \mathcal C^{(1)}=\Czero\sqcup\Cb,
\]
where \(\Czero\) contains the rows with two clean endpoints and \(\Cb\) contains
all remaining rows.  For \(c\in\Czero\), write
\[
    r_c(z_0)=b_c^\top z_0-g_c,
    \qquad
    g_c=b_c^\top z_0^\star.
\]
For \(f\in\Cb\), write
\[
    r_f(z_0,z_b)=a_f^\top z_0+c_f^\top z_b-g_f.
\]
As in \eqref{eq:row-norms}, every bad row has at most one clean endpoint, and
therefore
\begin{equation}\label{eq:l1-bad-row-clean-support}
    \norm{a_f}_1\le1.
\end{equation}

Assign a nonnegative weight \(w_f\) to every retained row.  The weighted
all-triangle \(\ell_1\) synchronization objective is
\begin{equation}\label{eq:trip-l1-full-objective}
\begin{aligned}
    F_1(z_0,z_b)
    &\defeq
    \sum_{c\in\Czero}w_c
       \abs{b_c^\top z_0-g_c}
    +
    \sum_{f\in\Cb}w_f
       \abs{a_f^\top z_0+c_f^\top z_b-g_f},\\
    (z_0^\sharp,z_b^\sharp)
    &\in
    \argmin_{z_0,z_b}F_1(z_0,z_b).
\end{aligned}
\end{equation}
The unweighted version has \(w_f\equiv1\).  The reliability-weighted version
used in the TriP notation has
\[
    w_{tu}=\sqrt{\pi_t\pi_u},
    \qquad
    0\le\pi_t\le1,
    \qquad
    \pi_t=1\quad(t\in\Tzero),
\]
so every clean-clean row has weight one and every other row has weight at most
one.

Because every row is a difference row, \(F_1\) is invariant under adding a common
constant to all triangle variables.  This is precisely the unavoidable global
log-scale gauge.  The objective is a finite convex polyhedral function.  After
fixing one additive gauge in every positive-weight connected component, it is a
linear program with a finite optimum, and hence a global minimizer exists.

\subsection{General deterministic recovery criterion}

Let
\[
    H_0=(\Tzero,\Czero)
\]
be the unweighted clean triangle-overlap graph.  Its edge-expansion constant is
\begin{equation}\label{eq:l1-clean-expansion-definition}
    h_0
    \defeq
    \min_{\substack{S\subset\Tzero\\0<|S|\le|\Tzero|/2}}
    \frac{|\partial_{H_0}S|}{|S|}.
\end{equation}
Let the clean-boundary bad row degree be
\begin{equation}\label{eq:l1-clean-boundary-degree-definition}
    d_b^\triangle
    \defeq
    \max_{t\in\Tzero}
    \sum_{f\in\Cb}|a_f(t)|.
\end{equation}
For the full shared-edge model this agrees with the quantity bounded in
Lemma~\ref{lem:boundary-degree}.

\begin{theorem}[General weighted \(\ell_1\) clean-scale recovery criterion]
\label{thm:l1-general-recovery}
Assume that \(H_0\) has at least two vertices.  Suppose the row weights satisfy
\begin{equation}\label{eq:l1-weight-bounds}
    w_c\ge w_{\mathrm g}>0
    \quad(c\in\Czero),
    \qquad
    0\le w_f\le w_{\mathrm b}<\infty
    \quad(f\in\Cb).
\end{equation}
If
\begin{equation}\label{eq:l1-general-strict-condition}
    w_{\mathrm g}h_0
    >
    w_{\mathrm b}d_b^\triangle,
\end{equation}
then every global minimizer of \eqref{eq:trip-l1-full-objective} recovers the
clean triangle log-scales up to one common additive gauge.  More precisely, for
every global minimizer \((z_0^\sharp,z_b^\sharp)\), there exists
\(\alpha\in\R\) such that
\begin{equation}\label{eq:l1-general-conclusion}
    z_t^\sharp=z_t^\star+\alpha,
    \qquad
    t\in\Tzero.
\end{equation}
No assertion is made about the absolute log-scales of non-clean triangles.
\end{theorem}

\begin{remark}[What is easier in the \(\ell_1\) theorem]
Theorem~\ref{thm:l1-general-recovery} is a statement about a global minimizer of
a convex objective.  It does not need to prove that an iterative nonconvex solver
enters or remains in a clean basin.  It also does not need to control a nuisance
Schur complement through a proximal inverse.  The only two quantitative objects
are the clean cut expansion \(h_0\) and the number of bad rows that can touch one
clean triangle, \(d_b^\triangle\).
\end{remark}

\begin{lemma}[\(\ell_1\) Poincar\'e inequality from edge expansion]
\label{lem:l1-poincare}
Let \(H=(V,E)\) be a finite graph with \(|V|\ge2\) and edge expansion
\[
    h(H)=
    \min_{\substack{S\subset V\\0<|S|\le|V|/2}}
    \frac{|\partial_HS|}{|S|}.
\]
Then every vector \(x\in\R^V\) satisfies
\begin{equation}\label{eq:l1-poincare-inequality}
    \sum_{\{u,v\}\in E}|x_u-x_v|
    \ge
    h(H)\min_{\alpha\in\R}
    \sum_{v\in V}|x_v-\alpha|.
\end{equation}
\end{lemma}

\begin{proof}
Let \(\alpha\) be a median of the entries of \(x\), and put
\(y_v=x_v-\alpha\).  For every \(s\ge0\), define the positive and negative
superlevel sets
\[
    V_s^+=\{v:y_v>s\},
    \qquad
    V_s^-=\{v:y_v<-s\}.
\]
Because \(\alpha\) is a median, both sets have cardinality at most \(|V|/2\).
For a single edge \(\{u,v\}\), the layer-cake identity gives
\[
    |y_u-y_v|
    =
    \int_0^\infty
       |\mathbf 1_{V_s^+}(u)-\mathbf 1_{V_s^+}(v)|\,ds
    +
    \int_0^\infty
       |\mathbf 1_{V_s^-}(u)-\mathbf 1_{V_s^-}(v)|\,ds.
\]
Summing over the edges and interchanging the finite sum and the integrals yields
\begin{align*}
    \sum_{\{u,v\}\in E}|y_u-y_v|
    &=
    \int_0^\infty|\partial_HV_s^+|\,ds
    +
    \int_0^\infty|\partial_HV_s^-|\,ds\\
    &\ge
    h(H)
    \left(
       \int_0^\infty|V_s^+|\,ds
       +
       \int_0^\infty|V_s^-|\,ds
    \right)\\
    &=
    h(H)\sum_{v\in V}|y_v|.
\end{align*}
A median minimizes the sum of absolute deviations, so
\(\sum_v|y_v|=\min_\alpha\sum_v|x_v-\alpha|\).  This proves
\eqref{eq:l1-poincare-inequality}.
\end{proof}

\begin{lemma}[Profiled bad \(\ell_1\) term is globally Lipschitz]
\label{lem:l1-profiled-bad-lipschitz}
For a clean perturbation \(e\in\R^{\Tzero}\), define
\begin{equation}\label{eq:l1-profiled-bad-definition}
    \Psi_1(e)
    \defeq
    \inf_{z_b\in\R^{\Tb}}
    \sum_{f\in\Cb}w_f
    \abs{a_f^\top e+c_f^\top z_b-\eta_f},
\end{equation}
where \(\eta_f\) is the fixed offset obtained after substituting the exact clean
block \(z_0^\star\) into row \(f\).  Under
\eqref{eq:l1-weight-bounds},
\begin{equation}\label{eq:l1-profiled-bad-lipschitz-bound}
    \Psi_1(e)
    \ge
    \Psi_1(0)
    -w_{\mathrm b}d_b^\triangle\norm{e}_1.
\end{equation}
\end{lemma}

\begin{proof}
For every fixed nuisance vector \(z_b\), the triangle inequality gives
\begin{align*}
    \sum_{f\in\Cb}w_f
       \abs{c_f^\top z_b-\eta_f}
    &\le
    \sum_{f\in\Cb}w_f
       \abs{a_f^\top e+c_f^\top z_b-\eta_f}
    +
    \sum_{f\in\Cb}w_f|a_f^\top e|.
\end{align*}
Taking the infimum over \(z_b\) in the first two terms gives
\[
    \Psi_1(0)
    \le
    \Psi_1(e)
    +
    \sum_{f\in\Cb}w_f|a_f^\top e|.
\]
Using \(w_f\le w_{\mathrm b}\) and expanding by clean coordinates,
\begin{align*}
    \sum_{f\in\Cb}w_f|a_f^\top e|
    &\le
    w_{\mathrm b}
    \sum_{f\in\Cb}\sum_{t\in\Tzero}|a_f(t)|\,|e_t|\\
    &\le
    w_{\mathrm b}d_b^\triangle\sum_{t\in\Tzero}|e_t|
    =
    w_{\mathrm b}d_b^\triangle\norm{e}_1.
\end{align*}
Rearranging proves \eqref{eq:l1-profiled-bad-lipschitz-bound}.
\end{proof}

\begin{proof}[Proof of Theorem~\ref{thm:l1-general-recovery}]
Fix any feasible pair \((z_0,z_b)\).  Since all rows are difference rows, add a
common constant to all triangle variables so that
\[
    e\defeq z_0-z_0^\star
\]
has median zero on \(\Tzero\).  Then
\begin{equation}\label{eq:l1-median-gauge}
    \norm{e}_1
    =
    \min_{\alpha\in\R}
    \sum_{t\in\Tzero}|z_t-z_t^\star-\alpha|.
\end{equation}
For a clean-clean row joining \(t\) and \(u\), exact clean consistency gives
\[
    b_c^\top z_0-g_c=e_u-e_t.
\]
Consequently, Lemma~\ref{lem:l1-poincare} and the clean weight lower bound give
\begin{equation}\label{eq:l1-clean-objective-lower-bound}
\begin{aligned}
    \sum_{c\in\Czero}w_c|b_c^\top z_0-g_c|
    &\ge
    w_{\mathrm g}
    \sum_{\{t,u\}\in\Czero}|e_u-e_t|\\
    &\ge
    w_{\mathrm g}h_0\norm{e}_1.
\end{aligned}
\end{equation}
The bad-row contribution of the particular nuisance vector \(z_b\) is at least
its profiled value \(\Psi_1(e)\).  By
Lemma~\ref{lem:l1-profiled-bad-lipschitz},
\begin{equation}\label{eq:l1-bad-objective-lower-bound}
    \sum_{f\in\Cb}w_f
       |a_f^\top z_0+c_f^\top z_b-g_f|
    \ge
    \Psi_1(0)-w_{\mathrm b}d_b^\triangle\norm{e}_1.
\end{equation}
Adding \eqref{eq:l1-clean-objective-lower-bound} and
\eqref{eq:l1-bad-objective-lower-bound} yields the global objective gap
\begin{equation}\label{eq:l1-global-objective-gap}
    F_1(z_0,z_b)
    \ge
    \Psi_1(0)
    +
    \bigl(w_{\mathrm g}h_0-w_{\mathrm b}d_b^\triangle\bigr)
    \norm{e}_1.
\end{equation}
On the other hand, fixing \(z_0=z_0^\star\) gives
\[
    \inf_{z_b}F_1(z_0^\star,z_b)=\Psi_1(0),
\]
because every clean-clean residual is exactly zero.  Hence the global optimum of \(F_1\) is at most \(\Psi_1(0)\).  Under the strict condition
\eqref{eq:l1-general-strict-condition}, the coefficient of \(\norm{e}_1\) in
\eqref{eq:l1-global-objective-gap} is positive.  Therefore a global minimizer
must satisfy \(\norm{e}_1=0\).  In the median gauge this means
\(z_t=z_t^\star\) for every \(t\in\Tzero\).  Undoing the common gauge gives
\eqref{eq:l1-general-conclusion}.
\end{proof}

\subsection{Complete-graph expansion and the explicit max-degree constant}

The next two lemmas convert the abstract quantities \(h_0\) and
\(d_b^\triangle\) into the original-camera max-degree parameter \(\Delta_E\).
They are deterministic and do not use a distribution of camera locations.

\begin{lemma}[Expansion of the full Johnson triangle graph]
\label{lem:l1-johnson-expansion}
Let \(n\ge4\), and let \(J(n,3)\) be the graph on the triples in \(\binom{[n]}3\), with two triples
adjacent exactly when they share two camera vertices.  For every
\(S\subset\binom{[n]}3\) satisfying
\(0<|S|\le\frac12\binom n3\),
\begin{equation}\label{eq:l1-johnson-expansion}
    |\partial_{J(n,3)}S|
    \ge
    \frac n2|S|.
\end{equation}
\end{lemma}

\begin{proof}
The graph \(J(n,3)\) is \(3(n-3)\)-regular.  Its adjacency eigenvalues are
\[
    3(n-3),
    \qquad
    2n-9,
    \qquad
    n-7,
    \qquad
    -3,
\]
so its unnormalized Laplacian eigenvalues are
\[
    0,
    \qquad
    n,
    \qquad
    2(n-1),
    \qquad
    3(n-2).
\]
Thus the spectral gap is \(n\).  Put \(N=\binom n3\), \(s=|S|\), and
\[
    f=\mathbf 1_S-\frac{s}{N}\mathbf 1.
\]
Then \(f\perp\mathbf 1\), and the spectral-gap inequality gives
\[
    n\norm{f}_2^2
    \le
    f^\top L_Jf
    =
    \sum_{\{u,v\}\in E(J(n,3))}(f_u-f_v)^2.
\]
Only cut edges contribute to the final sum, and every cut edge contributes one.
Moreover,
\[
    \norm{f}_2^2=\frac{s(N-s)}{N}.
\]
Therefore
\[
    |\partial_{J(n,3)}S|
    \ge
    n\frac{s(N-s)}{N}
    \ge
    \frac n2s,
\]
because \(s\le N/2\).
\end{proof}

\begin{lemma}[Clean expansion and bad boundary degree from \(\Delta_E\)]
\label{lem:l1-complete-graph-combinatorics}
Assume the camera graph is \(K_n\), every clean graph triangle is nondegenerate
and retained, and all retained shared-edge rows are used.  Then
\begin{equation}\label{eq:l1-dtriangle-complete-bound}
    d_b^\triangle\le6\Delta_E
\end{equation}
and
\begin{equation}\label{eq:l1-h0-complete-bound}
    h_0\ge\frac n2-6\Delta_E.
\end{equation}
\end{lemma}

\begin{proof}
The first inequality is the same counting argument as
Lemma~\ref{lem:boundary-degree}.  Fix a clean
triangle \(t=\{i,j,k\}\).  Along its clean camera edge \(ij\), a neighboring
triangle has the form \(\{i,j,\ell\}\).  Since \(ij\) is clean, this neighbor is
non-clean only if \(i\ell\in\Eb\) or \(j\ell\in\Eb\).  Hence there are at most
\(\deg_{\Eb}(i)+\deg_{\Eb}(j)\) non-clean neighbors along \(ij\).  Summing over
\(ij,ik,jk\) gives
\[
\begin{aligned}
&\bigl(\deg_{\Eb}(i)+\deg_{\Eb}(j)\bigr)
+\bigl(\deg_{\Eb}(i)+\deg_{\Eb}(k)\bigr)
+\bigl(\deg_{\Eb}(j)+\deg_{\Eb}(k)\bigr)\\
&\qquad
=2\bigl(\deg_{\Eb}(i)+\deg_{\Eb}(j)+\deg_{\Eb}(k)\bigr)
\le6\Delta_E.
\end{aligned}
\]
Each bad row has either one clean endpoint, in which case it contributes exactly
one unit to the clean-block incidence degree of that endpoint, or no clean
endpoint, in which case it contributes zero.  This proves
\eqref{eq:l1-dtriangle-complete-bound}.

For the expansion bound, let
\[
    \Omega=\binom{[n]}3,
    \qquad
    R=\Omega\setminus\Tzero.
\]
The graph \(H_0\) is the subgraph of \(J(n,3)\) induced by \(\Tzero\).  Fix
\(S\subset\Tzero\) with \(0<|S|\le|\Tzero|/2\).  Since
\(|\Tzero|\le|\Omega|\), we also have \(|S|\le|\Omega|/2\).  By
Lemma~\ref{lem:l1-johnson-expansion},
\[
    |\partial_{J(n,3)}S|\ge\frac n2|S|.
\]
Passing from the full Johnson graph to the induced clean graph removes only cut
edges from \(S\) into \(R\).  The first counting argument shows that every clean
triangle has at most \(6\Delta_E\) neighbors in \(R\).  Thus at most
\(6\Delta_E|S|\) full-Johnson cut edges are lost, and
\[
    |\partial_{H_0}S|
    \ge
    |\partial_{J(n,3)}S|-6\Delta_E|S|
    \ge
    \left(\frac n2-6\Delta_E\right)|S|.
\]
Taking the minimum over all admissible \(S\) gives
\eqref{eq:l1-h0-complete-bound}.
\end{proof}

\begin{theorem}[Complete-graph global \(\ell_1\) clean-scale recovery]
\label{thm:l1-complete-recovery}
Assume \(n\ge4\) and the complete-graph noiseless model of Theorem~\ref{thm:main}: camera edges
are partitioned as \(E(K_n)=\Ezero\sqcup\Eb\), clean directions are exact,
corrupted directions are arbitrary, every clean graph triangle is nondegenerate
and retained, and all retained shared-edge rows are used.  Consider the weighted
all-triangle \(\ell_1\) objective \eqref{eq:trip-l1-full-objective}, with weights
satisfying \eqref{eq:l1-weight-bounds}.  If
\begin{equation}\label{eq:l1-complete-weighted-condition}
    w_{\mathrm g}
    \left(\frac n2-6\Delta_E\right)
    >
    6w_{\mathrm b}\Delta_E,
\end{equation}
then every global minimizer recovers all clean triangle log-scales up to a common
additive gauge.

Equivalently, the explicit sufficient max-degree condition is
\begin{equation}\label{eq:l1-complete-weighted-ratio}
    \frac{\Delta_E}{n}
    <
    \frac{w_{\mathrm g}}
         {12(w_{\mathrm g}+w_{\mathrm b})}.
\end{equation}
In particular, for the unweighted objective, and also for TriP reliability
weights satisfying
\[
    \pi_t=1\quad(t\in\Tzero),
    \qquad
    0\le\pi_t\le1\quad(t\in\Tb),
    \qquad
    w_{tu}=\sqrt{\pi_t\pi_u},
\]
one may take \(w_{\mathrm g}=w_{\mathrm b}=1\).  The condition becomes
\begin{equation}\label{eq:l1-complete-unweighted-condition}
    \Delta_E<\frac n{24}.
\end{equation}
Thus a conservative explicit constant for global \(\ell_1\) log-scale
synchronization is
\begin{equation}\label{eq:l1-explicit-constant}
    c_{\mathrm{TriP},1}=\frac1{24}.
\end{equation}
\end{theorem}

\begin{proof}
Lemma~\ref{lem:l1-complete-graph-combinatorics} gives
\[
    h_0\ge\frac n2-6\Delta_E,
    \qquad
    d_b^\triangle\le6\Delta_E.
\]
Therefore \eqref{eq:l1-complete-weighted-condition} implies
\[
    w_{\mathrm g}h_0
    >
    w_{\mathrm b}d_b^\triangle.
\]
Theorem~\ref{thm:l1-general-recovery} then gives exact clean-scale recovery.
Solving \eqref{eq:l1-complete-weighted-condition} for \(\Delta_E/n\) gives
\eqref{eq:l1-complete-weighted-ratio}.  For
\(w_{\mathrm g}=w_{\mathrm b}=1\), this reduces to
\[
    \frac n2-6\Delta_E>6\Delta_E,
\]
which is equivalent to \(\Delta_E<n/24\).
\end{proof}

\begin{remark}[Strict inequality and uniqueness of the clean block]
The strict inequality in \eqref{eq:l1-complete-weighted-condition} guarantees
that every nonconstant clean perturbation produces a strictly positive objective
gap.  At equality, the proof gives only a nonnegative gap, and additional global
minimizers with an incorrect clean block are not ruled out.  The theorem claims
uniqueness only for the clean block modulo one additive gauge; nuisance
non-clean variables may remain nonunique.
\end{remark}

\begin{remark}[No location-distribution or quantitative angle assumption]
The constant \(1/24\) is graph-theoretic.  It does not depend on a probability
distribution for camera locations, a ShapeFit-style well-distribution constant,
or a quantitative lower bound on clean triangle angles.  The only geometric
requirement is qualitative nondegeneracy of each retained clean triangle, so that
its positive side ratios and exact clean log-scale equations are well-defined.
A noisy finite-precision theorem would require additional conditioning
parameters, but they do not enter this exact noiseless result.
\end{remark}

\begin{remark}[Comparison with the Cauchy theorem]
The constants \(10^{-5}\) and \(1/24\) quantify different claims.  The
first accompanies a warm-started fixed-point guarantee for an all-variable,
nonconvex, proximal annealed-Cauchy IRLS map.  The second accompanies every
global minimizer of a convex \(\ell_1\) objective.  The larger \(\ell_1\) constant
comes from the direct clean-cut domination argument and should not be read as a
constant for a local or prematurely terminated \(\ell_1\)-IRLS implementation.
\end{remark}

\subsection{Consequences for clean edge lengths and camera locations}

\begin{corollary}[Clean edge-length proposals for global \(\ell_1\) synchronization]
\label{cor:l1-clean-lengths}
Under the hypotheses of Theorem~\ref{thm:l1-complete-recovery}, let
\((z_0^\sharp,z_b^\sharp)\) be any global minimizer.  Then there exists a common
positive number \(\lambda=e^\alpha\) such that, for every clean triangle
\(t\in\Tzero\) and every clean camera edge \(e\subset t\),
\begin{equation}\label{eq:l1-clean-length-conclusion}
    \exp(z_t^\sharp)h_{t,e}
    =
    \lambda\ell_e^\star.
\end{equation}
Thus all clean incident triangles propose exactly the same globally scaled length
for every clean camera edge.
\end{corollary}

\begin{proof}
Theorem~\ref{thm:l1-complete-recovery} gives
\(z_t^\sharp=z_t^\star+\alpha\) for every clean triangle.  Using
\eqref{eq:clean-local-scale},
\[
    \exp(z_t^\sharp)h_{t,e}
    =e^\alpha\exp(z_t^\star)h_{t,e}
    =e^\alpha\ell_e^\star.
\]
The factor \(e^\alpha\) is independent of both \(t\) and \(e\).
\end{proof}

\begin{corollary}[Camera-location recovery from the global \(\ell_1\) scales]
\label{cor:l1-location-recovery}
Under the unweighted sufficient condition \(\Delta_E<n/24\), exact displacement
averaging on the recovered clean edges determines the ground-truth camera
locations up to one global translation and one global positive scale.
\end{corollary}

\begin{proof}
By Corollary~\ref{cor:l1-clean-lengths}, every recovered clean displacement is
\[
    \lambda\ell_{ij}^\star d_{ij}^\star
    =
    \lambda(x_i^\star-x_j^\star)
\]
for the same \(\lambda>0\).  The clean camera graph has minimum degree at least
\[
    n-1-\Delta_E>\frac n2
\]
for all sufficiently large \(n\), and is therefore connected.  Exact relative
displacements on a connected graph determine all camera locations up to one
common translation.  The factor \(\lambda\) is the unavoidable global scale.
\end{proof}

\begin{remark}[Algorithmic interpretation]
The optimization problem \eqref{eq:trip-l1-full-objective} can be written as a
linear program by introducing one nonnegative slack variable per shared-edge
row.  The theorem applies to an exact global solution of that convex problem, or
to any algorithm for which global optimality is certified.  A practical IRLS,
ADMM, primal--dual, or distributed implementation may be used, but its numerical
stopping rule and optimization error require a separate stability analysis if
one wants an end-to-end approximate-recovery theorem.
\end{remark}